%% file: ex_article.tex
\documentclass[onefignum,onetabnum]{siamonline220329}
\input{ex_shared}

\ifpdf
\hypersetup{
  pdftitle={Manifold learning and optimization using tangent space proxies},
  pdfauthor={D. Doe, P. T. Frank, and J. E. Smith}
}
\fi

\begin{document}

\maketitle

% REQUIRED
\begin{abstract}
    In machine learning, Riemannian manifolds offer a useful abstraction for approximating commonly encountered, non-Euclidean empirical data distributions and optimization state spaces. While Euclidean machine learning algorithms have been adapted to Riemannian manifolds, these adaptations rely on computationally intensive differential-geometric primitives, such as exponential maps and parallel transports. Here, we present a framework for efficiently approximating differential-geometric primitives on arbitrary manifolds via construction of an \emph{atlas graph} representation, which leverages the canonical characterization of a manifold as a finite collection, or \emph{atlas}, of overlapping \emph{coordinate charts}. We first show the utility of this framework in a setting where the manifold is expressed in closed form, specifically, a runtime advantage, compared with state-of-the-art approaches, for first-order optimization over the Grassmann manifold. Moreover, using point cloud data for which a complex manifold structure was previously established, i.e., high-contrast image patches, we show that an atlas graph with the correct geometry can be directly learned from the point cloud. Finally, we demonstrate that learning an atlas graph enables downstream key machine learning tasks. In particular, we implement a Riemannian generalization of support vector machines that uses the learned atlas graph to approximate complex differential-geometric primitives, including Riemannian logarithms and vector transports. These settings suggest the potential of this framework for even more complex settings, where ambient dimension and noise levels may be much higher. Taken together, our results speed up and render more broadly applicable Riemannian optimization routines at the forefront of modern data science and machine learning.
\end{abstract}

% REQUIRED

% REQUIRED

\section{Introduction}\label{sec:introduction}
Scientists increasingly encounter empirical data that are well characterized by non-Euclidean Riemannian manifolds that lack closed-form algebraic structure. For example, single-cell RNA-sequencing (scRNA-seq) datasets from peripheral blood mononuclear cells, embryonic mouse cells during gastrulation, %between embryonic days E6.5 and E8.5, 
and embryonic mouse brain cells %at E18 
all have nontrivial Riemannian geometry with statistically significant scalar curvature \citep{sritharan}. %,pbmc,gastrulation,neurogenesis}]. 
Similarly, computational tasks in computer vision, %and neurology
such as edge detection, are effectively performed on low-dimensional feature spaces with Riemannian geometry, such as Carlsson's high-contrast natural-image patches, which can be approximated using a parametrization of the Klein bottle \citep{carlsson,wang_huang_wang_2014}. Further, approximating real-world networks as Riemannian manifolds and learning metric and Ricci curvature information induced by end-to-end round trip time between nodes has been shown to elucidate the topology of the private networks of hyperscalers \citep{salamatian_2023}.

At the same time, solution state spaces underlying several problems in machine learning and engineering possess the Riemannian geometry of a manifold with algebraic structure. For example, in online PCA, maintaining a low-dimensional representation of the data depends on efficiently making updates on the Grassmann manifold  $\mathbf{Gr}_{n,k}$ of $k$-dimensional linear subspaces of $\mathbb R^n$ \citep{naik_2018,cardot_2018,hauberg_2016}. Graph partitioning and compression algorithms typically operate on a low-rank approximation of the graph adjacency matrix or Laplacian, which belong to the manifold $\mathbf{S}^+_{n,k}$ of $n\times n$ symmetric, positive-semidefinite (PSD) matrices of rank $k$ \cite[e.g.,][]{arv,cifuentes_2021}. In the realm of quantum computing: compiling programs on quantum hardware for $N$ qubits is tantamount to an optimization problem on a product manifold of subgroups of the special unitary group $SU(2^N)$, and parametrized quantum circuits---expected to enable quantum supremacy in the age of noisy intermediate-scale quantum (NISQ) hardware---similarly belong to this product manifold \citep{maronese_2022,benedetti_2019}.

Alongside increased awareness of Riemannian structure in both empirical data and machine learning tasks comes the realization that \textit{state-of-the-art methods for learning manifold representations from point clouds leave much to be desired}.
Conversion of empirically observed manifolds into low-dimensional representations is mostly relegated to dimensionality reduction techniques, which lose much of the topological and geometric information intrinsic to the underlying manifold, especially in the case of scRNA-seq data \citep{cooley}. Yet this loss may be unnecessary, as several fundamental invariants of Riemannian manifolds can be robustly computed from point cloud data. Homology groups, intrinsic dimension, pairwise geodesic distances, and the spectrum of the manifold's Laplace-Beltrami operator were being effectively computed 
almost two decades ago, as was pairwise Gromov-Hausdorff distance\footnote{Used to create a notion of distance between metric spaces that is invariant under isometry.} between two manifolds observed as point clouds \citep{niyogi_2008,levina_2004,belkin_2006,memoli_2004}. \textit{A faithful manifold representation scheme should be able to reproduce invariants which are robustly computable}.

While Riemannian structure in empirical data and machine learning make efficient Riemannian optimization routines increasingly relevant, the computational challenges are still largely prohibitive.
%As Riemannian structure is encountered more often in both empirical sciences and machine learning, computational tasks in these domains increasingly lean into differential geometry and Riemannian optimization.} 
Many Riemannian optimization routines are intuitive adaptations of Euclidean routines based on first- and second-order methods. However, it is far more complex to compute analogous first- and second-order primitives on Riemannian manifolds. For example, the action of a gradient on a point in Euclidean space is simply vector addition, whereas in a Riemannian manifold, this action requires the computation of geodesics. Geodesics are solutions to a specific second-order ordinary differential equation; while closed-form solutions are known for some manifolds\footnote{Absil \textit{et al.} give closed forms for the Grassmannian and symmetric positive-definite matrices \citep{absil}.}, closed-form solutions are generally unavailable, even for cases as simple as oblate ellipsoids in $\mathbb{R}^3$ \citep[e.g.,][]{ganshin_1969}. The computational difficulty only increases as one considers other differential-geometric primitives integral to Riemannian optimization, such as parallel transports and Riemannian Hessian operators \citep[e.g.,][]{hosseini_and_sra}. \textit{Riemannian optimization lacks a scheme for approximating differential-geometric primitives for general manifolds}, including both those that have closed forms as the locus of some equation (e.g., $\mathbf{Gr}_{n,k}$ and $\mathbf{S}^+_{n,k}$) and those learned from point clouds without definitive algebraic structure. 

There is now a greater incentive than ever before for applying Riemannian optimization to this latter case. For example: scRNA-seq is an application space where it is both the case data can be represented using low-dimensional manifold approximations \citep[e.g.][]{sritharan}, and analysis techniques are easily formalized as vector fields on a Riemannian manifold. RNA velocity, for instance, approximates
%but also a domain as many novel high-throughput experimental technologies provide us with data that are best understood in terms of Riemannian geometry. For example, emerging methods of analysis of scRNA-seq data allow the the estimation of ``RNA velocity,'' 
 %which  uses both unspliced and spliced RNA transcripts to 
% the approximate 
instantaneous rates of change in transcriptomic states of single cells \citep{frank_rec},
 which naturally lives in the tangent space to the manifold of RNA expression. Another example is the application of stochastic models to low-dimensional transcriptomic spaces to study, among other things, cell fate decisions in embryogenesis and innate lymphoid cell specialization \citep{moon,riesenfeld}. In practice, these models are computed in terms of dimensionality-reduced representations in order to learn proximity and directionality; the dimensionality reduction methods most often used have been shown to lose important topological and geometric information for general manifolds~\citep{cooley}.

\subsection{Our contributions}

%\samnote{I removed the section references here because I think it is better for this not to read as an outline. I also reordered the contributions to reflect their order in the paper; otherwise the outline already feels off-kilter.}
We present an ``atlas graph'' representation scheme that is, by construction, effectively universal in the Riemannian geometries it can encode. %for differential-geometric primitives
We demonstrate that the atlas graph scheme can both speed up first-order Riemannian optimization over manifolds with closed-form algebraic structure and enable it in the first place over manifolds learned from empirical point cloud data. For the former case, our experiments indicate a runtime advantage, compared to state-of-the-art manifold representation schemes, for an online subspace learning problem on the Grassmann manifold,
a classical target of Riemannian optimization routines~\cite{absil}. %(Sec.~\ref{sec:gras_results}).
In the case of point clouds, we present an approach for constructing an atlas graph representation %from a point cloud 
via ``approximate coordinate charts'', %(Sec.~\ref{sec:smoothlyembedded}), 
as well as  algorithms for quickly, closely approximating several differential-geometry primitives. 
%For the latter case, we show the atlas graph scheme enables both learning of Riemannian geometries from point clouds, as well as the effective approximation of differential-geometric primitives for first-order methods. W
Using the space of Carlsson's high-contrast, natural image patches, we show that the atlas graph scheme can correctly infer manifold structure from points clouds while preserving persistent homology\footnote{By reconstructing persistent homology, our method enables approximation of Gromov-Hausdorff distance \citep[Section 7.3]{oudot}.}, intrinsic dimensionality, and approximate geodesic path lengths. %(Sec.~\ref{sec:manifold_learning}). 
In fact, the scheme reproduces the spectrum of the Laplace-Beltrami operator\footnote{Sritharan, et al.~\citep{sritharan} show the local quadratic approximation scheme that we use in an atlas graph representation of Carlsson's image patches reproduces the spectrum of the Laplace-Beltrami operator.}, a highly nontrivial reconstruction that still does not require prior knowledge of the data manifold.
Finally, we give the first %ever 
implementation of a Riemannian optimization algorithm -- the Riemannian principal boundary (RPB) algorithm \citep{yao_2020} -- on a learned manifold with nontrivial topology that lacks closed-form differential-geometric primitives. Its application demonstrates the potential of the atlas graph scheme to enable modern methods of optimization on faithful manifold representations of empirical data. %(Sec.~\ref{sec:learned_manifold_optimization}).
%Using the atlas graph representation, this paper makes the following additional contributions: (1) an analysis of its runtime advantage for first-order Riemannian optimization techniques, compared to state-of-the-art manifold representation schemes; (2) a demonstration of accurate recovery of the Riemannian geometry of manifolds underlying observed point clouds; and (3) the implementation and application of Riemannian optimization and machine learning tasks on a manifold without closed form. 

Our results provide a foundation for future innovations on two major fronts: (1) learning manifold geometries from high-dimensional point cloud data with nontrivial sampling noise, and (2) implementing efficient Riemannian optimization routines for general manifolds.

\subsubsection{Outline} To present our results in context, we first describe related work %in Riemannian optimization, numerical differential geometry, manifold learning, and dimensionality reduction 
(Sec.~\ref{sec:related_work}). We give necessary mathematical background and define the atlas graph (Sec.~\ref{sec:atlas_graph}). Our main contributions are then presented in three parts, focused on enabling quasi-Euclidean first-order updates for faster Grassmanian optimization (Sec.~\ref{sec:gras_results}), learning manifold representations of point cloud data that preserve topological and geometric features (Sec.~\ref{sec:manifold_learning}), and performing Riemannian optimization over manifolds learned from point clouds (Sec.~\ref{sec:learned_manifold_optimization}). Finally, we discuss limitations and potential impacts more broadly (Sec.~\ref{sec:discussion}).

%We construct an atlas graph representation of the $(n,k)$-Grassmann manifold that permits quasi-Euclidean first-order updates (Sec.~\ref{sec:grass_example}), and show that it enables faster online estimation of the empirical Fr\'echet mean, compared to state-of-the-art first-order update schemes 
%on $\mathbf{Gr}_{n,k}$ %for a range of values of $n$ and $k$. (Sec.~\ref{sec:grass_experiment}).
%Turning to empirical data, we present an approach for constructing an atlas graph representation from a point cloud via ``approximate coordinate charts'' (Sec.~\ref{sec:smoothlyembedded}), as well as   algorithms for approximating several differential-geometry primitives  (Sec.~\ref{sec:primitives}). We apply these methods to create a representation of the manifold of high-contrast natural image patches that we show preserves homology and pairwise geodesic distances (Sec.~\ref{sec:carlsson_example}) and also supports the differential-geometric primitives necessary to learn a separating boundary between convex and concave image patches (Sec.~\ref{sec:rpb}).

\newcommand{\etal}{\textit{et al}.}
\section{Related work}\label{sec:related_work}
There exists a large body of literature on manifold optimization and manifold learning. We only discuss closely related works.
For a more comprehensive review of manifold optimization and manifold learning, we refer the reader to the excellent monographs by Absil \etal~\citep{absil} and Ghojogh \etal~\cite{ghojoghElementsDimensionalityReduction2023}.

\subsection{Related work on manifold optimization}
For general manifolds, the differential-geometric primitives necessary to deploy optimization algorithms need considerable effort to implement. 
For instance, applying a first-order update at a point in Euclidean space requires simple vector addition, but the analogous first-order update on a Riemannian manifold requires computing the exponential map of a tangent vector at the point. The exponential map is characterized by a second-order ordinary differential equation (ODE), whose closed-form solution, if known, can be expensive to compute \citep{kobayashi_nomizu}. 
Circumnavigating this difficulty is often discussed in terms of retractions, i.e., smooth maps that assign to each point a map from the tangent space to the manifold which agrees with the exponential map in the zero-th and first derivatives. 
Euclidean ``step-and-project'' methods, which make updates on a manifold using Euclidean gradients in an ambient space before mapping the update back into the manifold, are prototypical examples of retractions, but can suffer from slow running time due to numerical condition issues \citep[e.g.,][]{absil}.
Because finding sufficiently fast, accurate retractions is a challenge, they are often discussed in theory  rather than practice \citep{hosseini_and_sra}, with the notable exception of the B\"urer-Monteiro scheme for convex optimization on the  PSD elliptope \citep{cifuentes_2021}. %This popular scheme represents a matrix $A$ in the PSD elliptope by a square root, and makes updates by simple vector addition upon the square root.

Manifolds defined as the locii of certain equations (e.g., spheres, Grassmann manifolds, Stiefel manifolds, and fixed-rank PSD matrices) often have a well defined structure as submanifolds or quotients of $\mathbb{R}^{n \times k}$\citep{absil}. Consequently, their Riemannian primitives, such as geodesics, can be expressed entirely in terms of matrix operations \citep[e.g.,][]{weber,chakraborty_gifme}. For general quotient manifolds, certain constructions, such as vertical and horizontal bundles \citep[e.g.,][]{absil}, or problem-specific reparameterizations \citep[e.g.,][]{boumal_2016,em_alt}, may also speed up computation of first-order updates.
In these scenarios, the current state-of-the-art software is the MATLAB package Manopt 
\citep{boumal_2016,em_alt,luchnikov_2021}, which has a Python implementation, Pymanopt \citep{pymanopt}. This software numerically computes differential-geometric primitives, such as exponential maps, Riemannian logarithms, and parallel transport of tangent vectors. The computation of differential-geometric primitives by Manopt has been used to efficiently implement more complicated Riemannian optimization routines, such as fitting Gaussian mixture models or, for quantum computing, optimal unitary matrices \citep{em_alt,luchnikov_2021}. While Manopt is very effective in specific contexts, it cannot handle arbitrary Riemannian manifolds. Furthermore, in several cases, it does not use representations of minimal dimensionality. For example, while $\mathbf{Gr}_{n,k}$ is of dimension $k(n-k)$, Manopt represents an element $\overline{M}\in\mathbf{Gr}_{n,k}$%\samnote{I cannot parse the next part of this sentence}
using an $n\times k$ matrix $M$ such that $\overline{M}$ is the columnspace $\mathbf{col}M$ of $M$. Beyond increasing space and runtime requirements, this representation requires Manopt routines to account for redundant operations by orthonormalizing columnspaces,
%(i.e., $\mathbf{col}M=\mathbf{col}(MG)$ for all invertible matrices $G\in\mathbb{R}^{k\times k}$).
as the choice of $M$ is not unique.
%\samnote{Need you to explain to me exactly you mean to say here so I can make it more concise.}

% Differential-geometric primitives implemented for quotient manifolds also tend to be slow. \samnote{Does Manopt not also suffer this issue?} \hinote{Yes, Manopt does also suffer this issue; this statement applies to almost all the quotient manifolds Manopt represents. There is one exception to this---PSD matrices---but I address this in the Discussion.} 

\subsection{Related work on manifold learning}
For general manifolds observed empirically as point clouds, 
existing Riemannian optimization techniques have seen limited application, as they are written in terms of differential-geometric primitives whose closed forms are not known. In these cases,
low-dimensional representations are typically learned via dimensionality reduction techniques, which %at least in scRNA-seq data,
may forfeit or distort key topological and geometric information intrinsic to the underlying manifold \citep{cooley}.  Standard downstream machine learning approaches, such as clustering and inference of pseudotime, are also sensitive to the low-dimensional representation and to data fluctuations \citep{patruno_2020}.

However, differential-geometric and topological invariants of Riemannian manifolds may still be computed from point cloud data in a statistically meaningful way. Little \textit{et al.} prove that if a manifold $\mathcal{M}\subset\mathbb{R}^D$ is observed as a finite sample perturbed by sub-Gaussian noise, then their multiscale singular value decomposition (mSVD) algorithm can be used to compute both local tangent plane approximations and the intrinsic dimensionality of $\mathcal{M}$ \citep{little}. Sritharan \textit{et al.} describe a method to compute the scalar curvature from local quadratic approximations paired with the Gauss-Codazzi equations---an approach they also use to estimate the Riemannian metric and curvature tensor\footnote{Unlike scalar curvature, these estimates are not analyzed for accuracy or convergence rates.}---and show that they have greater statistical power over longstanding harmonic methods \citep{kac,sritharan}. Carlsson \textit{et al.} achieved a landmark success in recovering topological features from point cloud data in an analysis of the empirical distribution of $3\times3$ high-contrast, natural image patches \citep{carlsson}. Their study presented an elegant polynomial parameterization of the underlying manifold, whose validity was confirmed through \textit{ad hoc} persistent homology computations. However, these computations were only feasible due to the low ambient and intrinsic dimensionalities (nine and two, respectively) of the data. For empirically observed manifolds with higher ambient and intrinsic dimensionalities, such as in scRNA-seq datasets, such methods cannot be used
%with either statistical significance or computational tractability 
\citep{sritharan,oudot}.

\section{The atlas graph framework}\label{sec:atlas_graph}
Here, we define the atlas graph in terms of topological manifolds (Sec.~\ref{sec:top_mfld}). We then discuss the addition of differentiable, Riemannian, and vector field structures to topological manifolds (Sec.~\ref{sec:differentiable_mfld}--\ref{sec:descent}), and describe in each section how to compute relevant manifold invariants using an underlying atlas graph.
%\subsection{Mathematical preliminaries}\label{sec:preliminaries}
%In this section, we outline required mathematical background for the atlas graph representation scheme. %We briefly review what is minimally necessary to encode a Riemannian manifold, with particular emphasis on differentiable structures and atlases, before defining the atlas graph. 
%In addition, we review a recent method \cite{sritharan} for learning local quadratic representations of manifolds from point clouds that we use in Section~\ref{}. %This will be used later in our construction of an atlas graph representation of Carlsson's space of natural image patches.

%\subsection{Coordinate charts and atlases on Riemannian manifolds}
\subsection{Topological manifolds}\label{sec:top_mfld}
A topological manifold is a topological space $\mathcal{M}$ such that, for all $p\in\mathcal{M}$, there exists an open subset $\mathcal{U}\ni p$ of $\mathcal{M}$ such that $\mathcal{U}$ is homeomorphic (i.e., topologically equivalent) to $\mathbb{R}^n$ for some $n$ \citep{munkres}. We present two formalisms that have been used to represent arbitrary manifolds. The first formalism is more conventional, and we use it sparingly in our computation. The second formalism, which we refer to as the \textit{atlas graph framework}, is used for almost all of the computations in our experiments.

\subsubsection{Conventional representation of manifolds using atlases}
A \textit{coordinate chart} on a manifold $\mathcal{M}$ is a triple of the form $\left(\mathcal{U},\mathcal{V},\varphi\right)$ where 
%\begin{itemize}
%    \setlength\itemsep{0em}
    %\item 
    $\mathcal{U}$ is an open set in $\mathcal{M}$,
    %\item 
    $\mathcal{V}$ is an open set in $\mathbb{R}^d$, and
    %\item $
    $\varphi:\mathcal{U}\to\mathcal{V}$ is a homeomorphism, which we refer to as a \textit{coordinate chart map}.
%\end{itemize}

An \textit{atlas} of $\mathcal{M}$ is a collection of coordinate charts $\left\{\left(\mathcal{U}_\alpha,\mathcal{V}_\alpha,\varphi_\alpha\right)\right\}_{\alpha\in A}$, indexed by a set $A$, such that $\bigcup_{\alpha\in A}\mathcal{U}_\alpha=\mathcal{M}$, i.e., the coordinate charts ``cover'' $\mathcal{M}$. For any two coordinate charts $\mathcal{U}_\alpha,\mathcal{U}_\beta$ that overlap in $\mathcal{M}$ (i.e., $\mathcal{U}_\alpha\cap\mathcal{U}_\beta\neq\varnothing$), we define the \textit{transition map}:
%$\psi_{\alpha\beta}$: $\mathcal{V}_\alpha \rightarrow \mathcal{V}_\beta$ so that
\begin{equation}\label{eqn:transition_map}    \psi_{\alpha\beta}:=\left.\left(\varphi_\beta\circ\varphi_\alpha^{-1}\right)\right\rvert_{\varphi_\alpha\left(\mathcal{U}_\alpha\cap\mathcal{U}_\beta\right)}.
\end{equation}
Because $\psi_{\alpha\beta}$ is the restriction of a function $f:\mathbb{R}^d\to\mathbb{R}^d$, properties like $k$-times differentiability and smoothness are well-defined for $\psi_{\alpha\beta}$. If all transition maps of an atlas $\mathcal{A}$ are $C^k$ functions, we refer to $\mathcal{A}$ as a \textit{$C^k$-atlas}. If there exist $p\in\mathcal{M}$ and $\vec{\xi}\in\mathcal{V}_\alpha$ for some $\alpha$ such that $\vec{\xi}=\varphi_\alpha(p)$, we say that $\vec{\xi}$ is the \textit{representative} of $p$ in $\mathcal{V}_\alpha$.
%As a consequence of the definition of a manifold, every manifold admits an atlas. If the manifold is compact, the existence of a finite atlas is guaranteed.
%\samnote{Do you want to say something about how an atlas can be defined for every manifold?}
% We leverage the atlas graph framework (defined next) for efficient differential-geometric primitives,
%as we demonstrate that it allows for fast, simple, and general computation of differential-geometric primitives. However, 
% as well as the conventional atlas framework to create data structures amenable to an atlas graph approach. For example, we create a data structure from the Ehresmann atlas of the Grassmann manifold \citep{ehresmann_1934} that allows for faster online Fr\'echet mean computation than state-of-the-art methods (Section~\ref{sec:grass_example}).
%used to create a data structure representing the Grassmann manifold $\mathbf{Gr}_{n,k}$ that exploits the homeogeneity of $\mathbf{Gr}_{n,k}$ to add compressed coordinate charts~(Section~\ref{sec:atlas_graph_spec}) to an atlas graph of $\mathbf{Gr}_{n,k}$ in an \textit{ad hoc} fashion~(Section~\ref{sec:grass_example}).

We later leverage the conventional atlas framework to formalize a means of learning approximate coordinate charts from point cloud data and create a data structure from approximate coordinate charts that supports differential-geometric primitives~(Sections~\ref{sec:smoothlyembedded} and \ref{sec:primitives}).

\subsubsection{Representation of manifolds using atlas graphs}\label{sec:atlas_graph_spec}

We describe the atlas graph construction solely in terms of coordinate charts and transition maps, which gives it broad, flexible applicability to all Riemannian manifolds. Unlike the atlas construction, which is formalized in terms of abstract open sets, the atlas graph construction is immediately amenable to fast approximation of first-order primitives on Riemannian manifolds. This is because the atlas graph is encoded solely by open Euclidean sets and local diffeomorphisms between them.

From any smooth manifold $\mathcal{M}$ with finite atlas $\mathcal{A}=\left\{\left(\mathcal{U}_\alpha,\mathcal{V}_\alpha,\varphi_\alpha\right)\right\}_{\alpha\in A}$, we can construct the sets
%\begin{equation}\label{eqn:nodes_and_edges}
    $V=\left\{\mathcal{V_\alpha}\right\}_{\alpha\in A}$ and $E=\left\{\psi_{\alpha\beta}\right\}_{(\alpha,\beta)\in A^*},$
%\end{equation}
where $A^*$ is defined as
\begin{equation*}A^*=\left\{\left.(\alpha,\beta)\in A\times A\right\rvert\alpha\neq\beta\text{ and }\mathcal{U}_\alpha\cap\mathcal{U}_\beta\neq \varnothing \right\}.
\end{equation*}
We refer to each $\mathcal{V}_\alpha$ as a \textit{compressed coordinate chart}. We can think of the compressed coordinate charts in $V$ and transition maps in $E$ as the vertices and edges, respectively, of a graph $G=(V,E)$, which we refer to 
%$$\overline{\mathcal{A}}=\big(\left\{\mathcal{V_\alpha}\right\}_{\alpha\in A},\left\{\psi_{\alpha\beta}\right\}_{(\alpha,\beta)\in A^*}\big)$$
as an \textit{atlas graph} of $\mathcal{M}$. An atlas graph $G$ does not strictly contain enough information to reconstruct $\mathcal{M}$ (as the sets $\mathcal{U}_\alpha$ are irrecoverable), but it is sufficient for reconstructing a manifold $\overline{\mathcal{M}}$ that is diffeomorphic to $\mathcal{M}$ \citep[][Lemma 1.35]{lee_2012}. Similar to the case of atlases, we refer to an atlas graph as a \textit{$C^k$-atlas graph} if all transition maps $\psi_{\alpha\beta}$ are $k$-times differentiable. Unlike the atlas construction, which is formalized in terms of abstract open sets, the atlas graph construction is immediately amenable to fast approximation of first-order primitives on Riemannian manifolds. %This is because the atlas graph is encoded solely by open Euclidean sets and local diffeomorphisms between them.

%\samnote{YOU MUST BETTER CLARIFY EXACTLY WHAT WAS DONE AND WHAT IS NEW IN YOUR TREATMENT. SEE IF MY EDITS IN NEXT PARAGRAPH DO THIS APPROPRIATELY.}
While atlas graphs have arisen previously as intermediate structures in proofs of the sufficiency of an atlas to encode a manifold's topological and differentiable structure \citep[e.g.][Lemma 1.35]{lee_2012}, we use them here in novel ways as the basis for efficient differential-geometric primitives (Secs.~\ref{sec:grass_experiment} and \ref{sec:rpb}). We also adapt the construction to make use of approximate coordinate chart maps $\varphi_\alpha$ to learn manifold representations from point clouds (Sec.~\ref{sec:smoothlyembedded}).  

\subsection{Differentiable manifolds}\label{sec:differentiable_mfld}
A $C^k$-atlas graph for $k>0$ is an encoding of the \textit{tangent bundle}, which assigns a vector space $T_p\mathcal{M}$ to each $p\in\mathcal{M}$ called \textit{the tangent space at }$p$. 
The atlas graph construction explicitly encodes $T_p\mathcal{M}$ by the isomorphic representation 
$T_{\vec{\xi}}\mathcal{V}_\alpha \approxeq T_{\vec{\xi}}\mathbb{R}^d$, the affine space $\mathbb{R}^d$ centered at $\vec{\xi} = \varphi(p)$.
For the sake of clarity, we always refer to elements of compressed coordinate charts $\mathcal{V}_\alpha$ as $\vec{\xi}$, elements of the tangent plane $T_p\mathcal{M}$ as $\vec{v}$, and elements of $T_{\vec{\xi}}\mathcal{V}_\alpha$,
as $\vec{\tau}$. For further clarity, we refer to $\vec{v}\in T_p\mathcal{M}$ as a \textit{tangent vector} and $\vec{\tau}\in T_{\vec{\xi}}\mathcal{V}_\alpha$ as a \textit{representative tangent vector}.

Tangent spaces formalize ideas of covariant differentiation and infinitesimal movement on the manifold \citep{absil,dasilva}. A manifold paired with a tangent bundle is a \textit{differentiable manifold}.
Endowing a differentiable manifold with notions of distance and curvature requires a \textit{Riemannian metric} $\mathfrak{g}$, which assigns to each $p\in\mathcal{M}$ an inner product $\left\langle\cdot,\cdot\right\rangle_{\mathfrak{g}_p}$ on $T_p\mathcal{M}$ that is smooth as a function of $p$. This metric induces a notion of path length on the manifold. A shortest path with respect to the metric, i.e., a \textit{geodesic path}, is specified
%by a starting point $p\in\mathcal{M}$ and ending point $q\in\mathcal{M}$; equivalently, it can be specified
by a starting point $p$ and initial tangent vector $\vec{v}\in T_p\mathcal{M}$, as there exists a unique geodesic path starting at $p$ with initial direction $\vec{v}/\left\lVert\vec{v}\right\rVert_{\mathfrak{g}}$ of length $\left\lVert\vec{v}\right\rVert_{\mathfrak{g}}$. For fixed $p$, the map taking initial tangent vector $\vec{v}$ to the unique endpoint $q$ of this path is referred to as the \textit{exponential map} at $p$, written with the shorthand $\mathbf{Exp}_p\left(\vec{v}\right)=q$. This map is injective for tangent vectors $\vec{v}$ sufficiently close to the origin in $T_p\mathcal{M}$, and therefore locally invertible. The inverse, known as the \emph{Riemannian logarithm} at $p$ and denoted by $\mathbf{Log}_p$, is defined as the map taking $q\in\mathcal{M}$ to the $\mathfrak{g}_p$-smallest tangent vector $\vec{v}$ such that $\mathbf{Exp}_p\left(\vec{v}\right)=q$.

The exponential map is a special case of a \textit{retraction}, which smoothly assigns to each point $p\in\mathcal{M}$ a smooth map $\mathbf{Ret}_p:T_p\mathcal{M}\to\mathcal{M}$ such that $\mathbf{Ret}_p\left(\vec{0}\right)=p$ and the differential of $\mathbf{Ret}_p$ at $\vec{0}$ is the identity. In this way, a retraction can be interpreted as a local approximation of the exponential map.

By construction, retractions are also locally invertible; therefore, for each retraction $\mathbf{Ret}$, there exists a \emph{retraction logarithm} which, for fixed $p\in\mathcal{M}$, takes nearby $q\in\mathcal{M}$ to the $\mathfrak{g}_p$-smallest tangent vector $\vec{v}$ such that $\mathbf{Ret}\left(\vec{v}\right)=q$. We denote this retraction logarithm as $\mathbf{Log}^{\mathbf{Ret}}_p(q)=\vec{v}$, and we refer to $p$ as the \textit{basepoint} of the logarithm. 

\subsection{Representation of the Riemannian metric on atlas graphs}\label{sec:atlas_graph_met}
Let $\left(\mathcal{U},\mathcal{V},\varphi\right)$ be a coordinate chart of Riemannian manifold $\mathcal{M}$. If $\mathfrak{g}$ is the Riemannian metric on $\mathcal{M}$, then $\mathfrak{g}_p$ can be thought of as a linear isomorphism from $T_p\mathcal{M}$ to $T_p^*\mathcal{M}$ (i.e., the vector space of linear functionals on $T_p\mathcal{M}$). For $\vec{\xi}=\varphi(p)$, we can define a distinct Riemannian metric $\overline{\mathfrak{g}}_{\vec{\xi}}:T_{\varphi(p)}\mathcal{V}\to T_{\varphi(p)}^*\mathcal{V}$ in terms of local coordinate charts such that
\begin{equation}\label{eqn:g_bar}
\overline{\mathfrak{g}}_{\vec{\xi}}(\vec{\tau}):\vec{\tau}^\prime\mapsto\mathfrak{g}_p\left(\left[D\varphi^{-1}\right]_{\varphi(p)}\left(\vec{\tau}\right)\right)\left(\left[D\varphi^{-1}\right]_{\varphi(p)}\left(\vec{\tau}^\prime\right)\right),
\end{equation}
where $D\varphi^{-1}$ denotes the differential of $\varphi^{-1}$.
%While $\mathfrak{g}_p$ acts on a tangent vector $\vec{v}\in\mathcal{U}$, $\overline{\mathfrak{g}}_{\varphi(p)}$ acts on the representative tangent vector $\vec{\tau}=\left[D\varphi\right]_p\left(\vec{v}\right)$.

%Because $\overline{\mathfrak{g}}_{\vec{\xi}}$ is an inner product on $\mathbb{R}^d$, it possesses a matrix representation $\overline{G}_{\vec{\xi}}$ such that
%\begin{equation}\label{eqn:atlas_graph_mat}
%    \overline{\mathfrak{g}}_{\vec{\xi}}\left(\vec{\tau}\right)\left(\vec{\tau}^\prime\right)=\vec{\tau}^\top\overline{G}_{\vec{\xi}}\vec{\tau}^\prime.
%\end{equation}

\subsection{Descent along a vector field on an atlas graph}\label{sec:descent}
%\lonote{improve this discussion}
Let $\mathcal{M}$ be a Riemannian manifold of dimension $d$, and let $V$ be a smooth vector field on $\mathcal{M}$. For all $p\in\mathcal{M}$, the vector field $V$ assigns a tangent vector $\vec{v}\in T_p\mathcal{M}$. The soul of first-order methods on manifolds is to interpret $V$ as \textit{acting} on $\mathcal{M}$. Conventionally, this action takes $p$ to the exponential map $\mathbf{Exp}_p\left(\vec{v}\right)$, but some applications instead take $p$ to the retraction $\mathbf{Ret}_p\left(\vec{v}\right)$ for some suitable choice of retraction \citep{hosseini_and_sra}.

Beyond universal representation of Riemannian manifolds, the atlas graph approach brings an intuitive action by smooth vector fields on the represented manifold, which we refer to as a \textit{quasi-Euclidean update}: for $\vec{\xi}\in\mathcal{V}\subset\mathbb{R}^d$, a representative tangent vector $\vec{\tau}\in\mathbb{R}^d$ takes $\vec{\xi}$ to $\vec{\xi}+\vec{\tau}$.
%By homeomorphism, this is the same as the tangent vector $\vec{v}\in T_p\mathcal{M}$ taking $p$ to $\varphi^{-1}\left(\varphi(p)+\left[D\varphi\right]_p\left(\vec{v}\right)\right)$, where $\left[D\varphi\right]_p$ is the differential of $\varphi$ evaluated at $p$.
This action is easy to compute; more importantly, it approximates the exponential map well for small $\vec{\tau}$ by virtue of being a retraction when restrained to a single coordinate chart, as formalized in the following claim, which is proved in Appendix \ref{app:quasi_euclidean}.
\begin{claim}\label{clm:quasi_euclidean_approx}
    Quasi-Euclidean updates approximate the exponential map up to $O\left(\left\lVert\vec{\tau}\right\rVert^2_{\overline{\mathfrak{g}}}\right)$. Further, restriction of quasi-Euclidean updates to a single coordinate chart $\left(\mathcal{U},\mathcal{V},\varphi\right)$ comprises a retraction.
\end{claim}
Beyond the previous claim, the analysis of the trade-off between the accuracy of first-order updates on atlas graphs and the complexity of adding more charts requires singificantly more structure and assumptions that go beyond the scope of the current work. A brief discussion of these challenges is included in Appendix~\ref{app:quasi_instability}.

We leverage quasi-Euclidean updates to achieve faster first-order convergence in online Fr\'echet mean estimation on the Grassmann manifold (Sec.~\ref{sec:grass_experiment}), and to approximate first-order updates in the Riemannian principal boundary algorithm (Sec.~\ref{sec:rpb}).

\subsubsection{Approximation of parallel transport along quasi-Euclidean updates}\label{sec:vector_transport_def}
When a vector field $V$ acts on a differentiable manifold $\mathcal{M}$ to take a point $p$ to $q$, vectors in $T_p\mathcal{M}$ should be transformed into vectors in $T_q\mathcal{M}$. There are many linear isomorphisms between $T_p\mathcal{M}$ and $T_q\mathcal{M}$, and the ``correct'' choice of isomorphism should be uniquely determined by the path $\gamma:[0,1]\to\mathcal{M}$ along which $V$ takes $p$ to $q$. This isomorphism $T_p\mathcal{M}\to T_q\mathcal{M}$ determined by $\gamma$ is called the \textit{parallel transport}. Its fundamental role in first-order methods over manifolds is described in~\citep{hosseini_and_sra}.

On an atlas graph, we approximate parallel transport $\mathcal{P}^{\mathbf{Ret}_p}_{p\to\mathbf{Ret}_p\left(\vec{v}\right)}\vec{w}$ of tangent vector $\vec{w}\in T_p\mathcal{M}$ to $T_{\mathbf{Ret}_p\left(\vec{v}\right)}\mathcal{M}$ along $\mathbf{Ret}_p$ with the identity vector transport $\mathcal{T}_{\vec{\xi}\to\vec{\xi}+\vec{\tau}}:\vec{\sigma}\mapsto\vec{\sigma}$.
%\begin{equation}\label{eqn:app_vec_trans}
%    \mathcal{T}_{p\to\mathbf{Ret}_p\left(\vec{v}\right)}:\vec{w}\mapsto\left[D\varphi^{-1}\right]_{\mathbf{Ret}_p\left(\vec{v}\right)}\left(\left[D\varphi\right]_p\left(\vec{w}\right)\right).
%\end{equation}
%Equivalently, this action takes a representative tangent vector $\vec{\sigma}:=\left[D\varphi\right]_p\left(\vec{w}\right)$ to itself.
This vector transport approximates parallel transport along quasi-Euclidean updates up to order $O\left(\left\lVert\vec{\sigma}\right\rVert_{\overline{\mathfrak{g}}}\left\lVert\vec{\tau}\right\rVert_{\overline{\mathfrak{g}}}^2\right)$
%, where $\vec{\tau}=\left[D\varphi\right]_p\left(\vec{v}\right)$
(Appendix~\ref{app:vec_trans}). It is used to approximate Riemannian logarithms (Secs.~\ref{sec:grass_experiment} and~\ref{sec:rpb}), to approximate geodesic path lengths (Sec.~\ref{sec:carlsson_example}) and implement the Riemannian principal boundary algorithm (Sec.~\ref{sec:rpb}).

\section{Enabling quasi-Euclidean first-order updates for faster Grassmannian optimization}\label{sec:gras_results}
%This section is structured such that each theoretical contribution is immediately followed by associated empirical results. 
Here, we construct an atlas graph representation of the $(n,k)$-Grassmann manifold that permits quasi-Euclidean first-order updates (Sec.~\ref{sec:grass_example}), and show that it enables faster online estimation of the empirical Fr\'echet mean, compared to state-of-the-art first-order update schemes 
%on $\mathbf{Gr}_{n,k}$ %for a range of values of $n$ and $k$. 
(Sec.~\ref{sec:grass_experiment}).

\subsection{An atlas graph representation of the (n,k)-Grassmannian}\label{sec:grass_example}
%We construct an atlas graph of  $\mathbf{Gr}_{n,k}$ that enables faster first-order updates accurate up to first order. 
The atlas graph is constructed from a conventional atlas of $\mathbf{Gr}_{n,k}$ derived from a cell complex presented by Ehresmann~\citep{ehresmann_1934}, which we adapt to our notation for convenience. Our atlas graph has charts from the Ehresmann atlas and also permits \textit{ad hoc} creation of new coordinate charts centered at any point in the manifold. We use \textit{ad hoc} chart creation to maintain proximity to the origin in compressed charts and, hence, accuracy of quasi-Euclidean updates, which results in fast online learning (Sec.~\ref{sec:grass_experiment}).

In addition to the Ehresmann atlas, there exists another canonical representation of the Grassmann manifold as a quotient of the Stiefel matrices~\cite{bendokat_2011}. We find the Ehresmann atlas simpler to present, but also use the Stiefel construction  (e.g., Sec.~\ref{sec:grass_experiment}, Algorithm~\ref{alg:grass_quasi_euclidean} and Fig.~\ref{fig:first_order_subroutines}) to enable direct comparisons with existing Riemannian optimization approaches on the Grassmann manifold that use it, such as Pymanopt \citep{pymanopt} and GiFEE \citep{chakraborty_gifme}. 

\subsubsection{Coordinate charts of the Ehresmann atlas}\label{sec:ehresmann}
We begin with some intuition for the Ehresmann atlas construction. The $(n,k)$-Grassmannian can be understood as the manifold of $n\times n$ orthogonal projection matrices of rank $k$. There are $\binom{n}{k}$ such matrices whose entries are all zero, save for exactly $k$ diagonal entries that are equal to one. These matrices are in one-to-one correspondence with sets of $k$ fixed integers $i_1,\ldots,i_k$ satisfying $1\leq i_1<\ldots<i_k\leq n$, and therefore with the permutations\footnote{For finite $S\subset\mathbb{R}$, the notation $\min_{(m)}S$ denotes the $m$th smallest element of $S$.}
\begin{align}\label{eqn:permutation_indices}
    \pi_{i_1,\ldots,i_k}:\{1,\ldots,n\}&\to\{1,\ldots,n\} \\
    j&\mapsto\left\{\begin{array}{lr}
        i_j, & j\leq k \\
        \min_{(j-k)}\left(\left\{1,\ldots,n\right\}\setminus\left\{i_1,\ldots,i_k\right\}\right), & j>k \nonumber
    \end{array}\right..
\end{align}
% \begin{equation}
%     \begin{array}{rcl}
%         \pi_{i_1,\ldots,i_k}:\{1,\ldots,n\} & \to & \{1,\ldots,n\} \\
%         \ & \mapsto & \left\{\begin{array}\end{array}\right.
%     \end{array}
% \end{equation}
Let $P_{i_1,\ldots,i_k}$ be the $n\times n$ permutation matrix corresponding to $\pi_{i_1,\ldots,i_k}$. The aforementioned $\binom{n}{k}$ matrices take the form
\begin{equation*}
    P_{i_1,\ldots,i_k}\left(\frac{I_k}{\mathbf{0}_{n-k,k}}\right)\left(\frac{I_k}{\mathbf{0}_{n-k,k}}\right)^\top P_{i_1,\ldots,i_k}^\top.
\end{equation*} 
These $\binom{n}{k}$ points are the centerpoints of each coordinate chart in the Ehresmann atlas. Further, for every $P\in\mathbf{Gr}_{n,k}$, there exist a permutation $\pi_{i_1,\ldots,i_k}$ and a matrix $A\in\mathbb{R}^{(n-k)\times k}$ such that $P=P_{i_1,\ldots,i_k}\left(\frac{I_k}{A}\right)\left(P_{i_1,\ldots,i_k}\left(\frac{I_k}{A}\right)\right)^\dag$, where $\dag$ denotes the Moore-Penrose pseudoinverse.

For the coordinate charts in the Ehresmann atlas %directly from this intuition. 
let  $\mathbf{colproj}:\mathbb{R}^{n\times k}\to\mathbb{R}^{n\times n}$ be the map that takes a matrix $A$ to the orthogonal projector onto the column space of $A$. 
The coordinate chart $\left(\mathcal{U}_{i_1,\ldots,i_k},\mathcal{V}_{i_1,\ldots,i_k},\varphi_{i_1,\ldots,i_k}\right)$ is then defined as:
\begin{itemize}
    \setlength\itemsep{0em}
    \item $\varphi_{i_1,\ldots,i_k}^{-1}:A\mapsto\mathbf{colproj}\left(P_{i_1,\ldots,i_k}\left(\frac{I_k}{A}\right)\right)$
    \item $\mathcal{V}_{i_1,\ldots,i_k}=\mathbb{R}^{(n-k)\times k},$ \quad and \quad $\mathcal{U}_{i_1,\ldots,i_k}=\text{im}\left(\varphi_{i_1,\ldots,i_k}^{-1}\right).$
\end{itemize}
For convenience, we denote the coordinate chart corresponding to the identity permutation as $\left(\mathcal{U}_0,\mathcal{V}_0,\varphi_0\right)$. In Claims~\ref{clm:grass_dist_0}~and~\ref{clm:grass_dist_other}~(Appendix \ref{app:misc_grass_comp}), we demonstrate that a point $\vec{\xi}$ in a compressed chart $\mathcal{V}$ of the Ehresmann atlas is closer to the center of another chart than to the center of $\mathcal{V}$ if any coordinate of $\vec{\xi}$ exceeds one (Section~\ref{app:grass_atlas_graph}). 
%For $P\in\mathbf{Gr}_{n,k}$, the chart $\mathcal{V}$  %\samnote{Is this condition also sufficient, or just necessary?}
We use this result to determine when to generate new \emph{ad hoc} charts (Sec.~\ref{sec:grass_experiment}) The creation of \textit{ad hoc} coordinate charts allows for points to always be close to the origin within a coordinate chart, thereby reducing the error of quasi-Euclidean updates (Sec.~\ref{sec:descent}). In the case that $P\in\mathbf{Gr}_{n,k}$ is closer to the center of $\mathcal{V}$ than to the center of any Ehresmann chart, we say that $\mathcal{V}$ is the ``best'' Ehresmann chart for $P$.

\subsubsection{Ad hoc formation of coordinate charts}\label{sec:ad_hoc}
To perform \textit{ad hoc} chart formation, we begin with the case of generating charts centered at projection matrices belonging to $\mathcal{U}_0$. This case generalizes to projection matrices belonging to $\mathcal{U}_{i_1,\ldots,i_k}$ by simple conjugation by $P_{i_1,\ldots,i_k}$. The group $\mathbf{O}(n)$ of $n\times n$ orthogonal matrices acts transitively on $\mathbf{Gr}_{n,k}$ according to the action $Q:P\mapsto QPQ^\top$. Therefore, for all $A\in\mathcal{V}_0$, there exists an orthogonal matrix $Q_A$ such that
\begin{equation}\label{eqn:qa_action}
    \varphi^{-1}_0(A)=Q_A\left(\begin{array}{c|c}
        I_k & \mathbf{0}_{k, n-k} \\
        \hline
        \mathbf{0}_{n-k,k} & \mathbf{0}_{n-k,n-k}
    \end{array}\right)Q_A^\top.
\end{equation}

\begin{claim}\label{clm:Q_A}
    An orthogonal matrix $Q_A$ which satisfies Equation \ref{eqn:qa_action} is given by
    \begin{equation*}
        Q_A=\left(\begin{array}{c|c}
            \sqrt{\left(I_k+A^\top A\right)^{-1}} & -A^\top\sqrt{\left(I_{n-k}+AA^\top\right)^{-1}} \\
            \hline
            A\sqrt{\left(I_k+A^\top A\right)^{-1}} & \sqrt{\left(I_{n-k}+AA^\top\right)^{-1}}
        \end{array}\right).
    \end{equation*}
    \begin{proof}
        See Appendix \ref{app:misc_grass_comp}.
    \end{proof}
\end{claim}

Claim \ref{clm:Q_A} allows us to %always 
change coordinate charts so that quasi-Euclidean updates approximate the exponential map more accurately, as done in Algorithm \ref{alg:grass_quasi_euclidean}. While this specific method %described above 
only works for points in the Ehresmann chart $\left(\mathcal{U}_0,\mathcal{V}_0,\varphi_0\right)$, it generalizes to charts $\left(\mathcal{U}_{i_1,\ldots,i_k},\mathcal{V}_{i_1,\ldots,i_k},\varphi_{i_1,\ldots,i_k}\right)$ by replacing the action in Equation \ref{eqn:qa_action} with the action
\begin{equation}\label{eqn:qa_action_general}
    \varphi^{-1}_{i_1,\ldots,i_k}(A)=Q_AP_{i_1,\ldots,i_k}\left(\begin{array}{c|c}
        I_k & \mathbf{0}_{k,n-k} \\
        \hline
        \mathbf{0}_{n-k,k} & \mathbf{0}_{n-k,n-k}
    \end{array}\right)P_{i_1,\ldots,i_k}^\top Q_A^\top.
\end{equation}
This creates coordinate charts $\left(\mathcal{U}_{A,i_1,\ldots,i_k},\mathcal{V}_{A,i_1,\ldots,i_k},\varphi_{A,i_1,\ldots,i_k}\right)$ for all $A\in\mathbb{R}^{(n-k)\times k}$ and all $P_{i_1,\ldots,i_k}$, with the transition map $\psi_{A,i_1,\ldots,i_k\to A^\prime,j_1,\ldots,j_k}$ taking $B\in\mathcal{V}_{i_1,\ldots,i_k}$ to
\begin{equation}\label{eqn:grass_transition_map}
    \left(\mathbf{0}_{n-k,k}\mid I_{n-k}\right)R\left(\left(I_k\mid\mathbf{0}_{k,n-k}\right)R\right)^{-1},
\end{equation}
where $R=P_{j_1,\ldots,j_k}^\top Q_{A^\prime}^\top Q_AP_{i_1,\ldots,i_k}\left(\frac{I_k}{B}\right)$.
% \begin{equation}\label{eqn:grass_transition_map}
%     \left(\mathbf{0}_{n-k,k}\mid I_{n-k}\right)P_{j_1,\ldots,j_k}^\top Q_{A^\prime}^\top Q_A P_{i_1,\ldots,i_k}\left(\frac{I_k}{B}\right)\left(\left(I_k\mid\mathbf{0}_{k,n-k}\right)P_{j_1,\ldots,j_k}^\top Q_{A^\prime}^\top Q_AP_{i_1,\ldots,i_k}\left(\frac{I_k}{B}\right)\right)^{-1}.
% \end{equation}
%\samnote{The math in the line above is too difficult to parse, and for no good reason. Let us discuss how to make it legible and simpler.}
Within any compressed coordinate chart $\mathcal{V}_{A,i_1,\ldots,i_k}$, transition maps are invoked whenever any element of the coordinates $B$ exceeds one~(Appendix~\ref{app:grass_atlas_graph}), transitioning into a chart centered at $\varphi_{i_1,\ldots,i_k}^{-1}\left(B\right)$.
%This comprises the transition map from the original chart into the new chart. Finally, a few additional computational primitives enable efficient construction of the atlas graph  of  $\mathbf{Gr}_{n,k}$ (Appendix~\ref{sec:frechet_mean_estimation_methods}).

%We formalize quasi-Euclidean first-order updates on the atlas graph as follows: Given coordinate chart $\left(\mathcal{U},\mathcal{V},\varphi\right)$, together with a coordinate $\vec{\xi}\in\mathcal{V}$ and tangent vector $\vec{\tau}\in\mathcal{V}$, the quasi-Euclidean update of $\vec{\xi}$ by $\vec{\tau}$ is simply $\vec{\xi}+\vec{\tau}$.
%To show that quasi-Euclidean updates approximate the exponential map with respect to the Riemannian metric up to first order for coordinates near $\mathbf{0}_{n-k,k}$, it suffices to show that the Riemannian metric at $\mathbf{0}_{n-k,k}$ is a scalar multiple of the identity matrix\footnote{The reciprocal of this scalar multiple can then be applied to the tangent vector $\vec{\tau}$ before acting on $\vec{\xi}$.}. 

%In Claim~\ref{clm:grass_met}, while we do not give this closed form, we prove that the Riemannian metric is approximately equal to a scalar multiple of the identity in small neighborhoods of $\mathbf{0}_{n-k,k}$. This is sufficient by the continuity of the Riemannian metric to demonstrate first-order accuracy near the origin. Creating new charts using conjugation by $Q_A$ given above, this numerical accuracy can always be maintained. In tandem, the atlas graph and updates allow for the computational speedup observed in Section~\ref{sec:grass_experiment}.

\subsection{Fast online subspace learning on the Grassmann manifold}\label{sec:grass_experiment}

We compare the runtime of atlas-graph quasi-Euclidean updates on an atlas graph representation of $\mathbf{Gr}_{n,k}$,
against the runtimes of state-of-the-art first-order update techniques 
for the online subspace-learning task of 
computing the Grassmann inductive Fr\'echet expectation estimator (GiFEE)~\citep{chakraborty_gifme}. 

Given a stream of samples $\mathcal{X}_1,\ldots,\mathcal{X}_i$ from a probability distribution on $\mathbf{Gr}_{n,k}$, the GiFEE estimator is computed by inductively applying the update rule
\begin{align}\label{eqn:log_update}
    \vec{v}_{i+1} & \gets \mathbf{Log}_{M_i}\left(\mathcal{X}_i\right)\\
    M_{i+1} & \gets \textbf{Exp}_{M_i}\left(\frac{1}{i+1}\vec{v}_{i+1}\right),
\end{align}
where $M_1=\mathcal{X}_1$ and the tangent vector $\vec{v}_i\in T_{M_i}\mathbf{Gr}_{n,k}$ is initialized to $\vec{v}_1=\vec{0}$. These update rules are exactly those for online-computing a mean of samples from a real vector space, but adapted to differential-geometric primitives. 
To compute first-order updates on the atlas graph representation, our algorithm (\atlasfrechet , Algorithm~\ref{alg:grass_quasi_euclidean}) 
replaces invocations of the exponential map with invocations of quasi-Euclidean updates on the atlas graph. Due to the inductive nature of the estimator, the space requirements are constant with respect to number of charts used. runtime dependence on the number of chart transitions is discussed in Appendix~\ref{app:grass_atlas_graph}. %While we do not give any analysis as to the conditions under which our algorithm converges, and how fast it converges, 
To ensure that the matrix inversions used in \atlasfrechet\ are well defined, we assume a probabilistic relaxation of an assumption from Theorem 1 of Chakraborty and Vemuri~\citep{chakraborty_gifme}; that is, we assume that with high probability, all points sampled on $\mathbf{Gr}_{n,k}$ are within geodesic radius $\frac{\pi}{4}$. 

% We compare the runtime performance of state-of-the-art first-order update schemes on the Grassmann manifold by implementing the GiFEE algorithm, using various primitives. 
We compare this algorithm against the following Grassmann representation frameworks: (1) the original GiFEE algorithm, using closed forms for logarithms and exponential maps from \citep{chakraborty_gifme}; (2) Pymanopt (MANOPT); and 3)  Pymanopt, using retractions instead of the exponential map for first-order updates (MANOPT-RET). The specifics of how these primitives are implemented for these update schemes are given in Fig.~\ref{fig:first_order_subroutines}. 
As an idealized baseline, we also include an implementation of online PCA (oPCA) based on Oja's rule \citep{liang_2022}, which solves the related task of online subspace learning by very simple Euclidean updates.
\begin{algorithm}[t]
    \caption{\atlasfrechet \\
    Online Fr\'echet mean estimation on $\mathbf{Gr}_{n,k}$ using quasi-Euclidean updates on an atlas graph
    } \label{alg:grass_quasi_euclidean}
    \begin{algorithmic}
        % \Require Fr\'echet mean $\mathcal{X}\in\mathbf{Gr}(n,k),p>1$
        %\Require Stream of samples $X_1,X_2,\ldots\sim\mathbf{GPD}(\mathcal{X},p)$, sampled as Stiefel matrices
        \Require Probability distribution $\mathcal{D}$ on $\mathbf{Gr}_{n,k}$
        \Require Fr\'echet stream of samples $X_1,X_2,\ldots\sim\mathcal{D}$, sampled as Stiefel matrices
        \State $i_1,\ldots,i_k\gets\texttt{ATLAS\_identify\_chart}\left(X_1\right)$ \Comment{Best Ehresmann chart (Sec. \ref{sec:ehresmann}); Alg. \ref{alg:grass_identify}}
        \State $A\gets\texttt{ATLAS\_ingest\_matrix}\left(X_1,i_1,\ldots,i_k\right)$ \Comment{$A=\varphi_{i_1,\ldots,i_k}\left(\mathbf{colproj}\left(X_1\right)\right)$; Alg. \ref{alg:grass_ingest}}
        \State $Q_A\gets I_n$ \Comment{$Q_A$ is used to define the map (\ref{eqn:qa_action})}
        \State $Q_{A,U}\gets \left(I_k\mid\mathbf{0}_{k,n-k}\right)^\top$ \Comment{restriction of $Q_A$ to columns in $\{i_1,\ldots,i_k\}$}
        \State $Q_{A,L}\gets \left(I_{n-k}\mid\mathbf{0}_{n-k,k}\right)^\top$ \Comment{restriction of $Q_A$ to columns not in $\{i_1,\ldots,i_k\}$}
        \State $n\gets 1$
        \While{streaming}
            
            \State $n\gets n+1$
            %\State $\tilde{X}\gets\texttt{Stiefel\_from\_QR}\left(X_n\right)$ \Comment{$\tilde{X}$ and $X_n$ share columnspace; Alg. \ref{alg:QR}}
            %\State $\tilde{A}\gets Q_{A,L}^\top\tilde{X}\left(Q_{A,U}^\top\tilde{X}\right)^{-1}$ \Comment{$\tilde{A}=\varphi_{i_1,\ldots,i_k}\left(Q_A^\top\mathbf{colproj}\left(\tilde{X}\right)Q_A\right)$; Eq. (\ref{eqn:qa_action_general})}
            %\State $A\gets\left(\tilde{A}-A\right)/n$\Comment{Update online Fr\'echet mean estimator}
            \State $\tilde{A}\gets Q_{A,L}^\top X_n\left(Q_{A,U}^\top X_n\right)^{-1}$ \Comment{$\tilde{A}=\varphi_{i_1,\ldots,i_k}\left(Q_A^\top\mathbf{colproj}\left(X_n\right)Q_A\right)$; Eq. (\ref{eqn:qa_action_general})}
            \State $A\gets A+\left(\tilde{A}-A\right)/n$\Comment{Update online Fr\'echet mean estimator}
            \If{any entry $a$ of $A$ violates $\lvert a\rvert<1$}\Comment{change chart if necessary; Claim \ref{clm:grass_dist_0}}
                %\State $A,Q_A,Q_{A,U},Q_{A,L}\gets\texttt{ATLAS\_transition\_map}(A,i_1,\ldots,i_k,Q_A)$ \Comment{Alg. \ref{alg:grass_trans_map}}
                \State $A,Q_A\gets\texttt{ATLAS\_transition\_map}(A,i_1,\ldots,i_k,Q_A)$ \Comment{Alg. \ref{alg:grass_trans_map}}
                \State $Q_{A,U}\gets$ restriction of $Q_A$ to columns in $\{i_1,\ldots,i_k\}$
                \State $Q_{A,L}\gets$ restriction of $Q_A$ to columns not in $\{i_1,\ldots,i_k\}$
            \EndIf
        \EndWhile
        \State \Return $A,i_1,\ldots,i_k,Q_A$
    \end{algorithmic}
\end{algorithm}

Because of the lack of existing benchmark datasets, we introduce the  \emph{geodesic power distribution}
$\mathbf{GPD}(\mathcal \mathcal{X} , p)$ with Fr\'echet mean  $\mathcal X \in\mathbf{Gr}_{n,k}$ and scaling exponent $p > 1$, which inversely controls the entropy of the distribution (formally defined in Appendix~\ref{sec:geodesic_power_distribution}). The $\mathbf{GPD}$ is a natural, efficiently samplable distribution, which guarantees existence and uniqueness of the Fr\'echet mean $\mathcal{X}$.
Chakraborty and Vemuri show that the GiFEE estimator converges in probability to the Fr\'echet mean of any distribution on $\mathbf{Gr}_{n,k}$ under certain limitations on support and Riemannian $L^2$-moment (i.e., a differential-geometric analog of the first moment about the mean) \citep{chakraborty_gifme}. 
A $\mathbf{GPD}$ satisfies the $L^2$-moment constraint for $p>1$; while it does not satisfy the support constraint, the probability density function $\mathbf{GPD}(\mathcal{X},p)$ is close to zero for most points $\mathcal{X}\in\mathbf{Gr}_{n,k}$ for sufficiently high $p$. For the sake of fair comparison, points are sampled from $\mathbf{GPD}\left(\mathcal{X},p\right)$ not as orthogonal projection matrices $P$, but as Stiefel matrices $X$ satisfying $P=XX^\top$. This is done by implementing the procedure of sampling from $\mathbf{GPD}\left(\mathcal{X},p\right)$ (Appendix~\ref{sec:geodesic_power_distribution}) in Pymanopt, which represents elements of $\mathbf{Gr}_{n,k}$ in terms of Stiefel matrices.

%though it does not apply to manifolds beyond the Grassmanian and is not known to compute the Fr\'echet mean, because oPCA consists of a theoretically optimal for the related task of online subspace learning in the case of sub-Gaussian data~\cite{liang_2022}. 
%

%Figs.~\ref{fig:population_frechet_error} and~\ref{fig:empirical_frechet_error} demonstrate that the error in population and empirical Fr\'echet mean are both negligible for smaller values of $k$ and all $n$. Fig.~\ref{fig:table_small_nk} demonstrates runtime advantage for almost all values of $n$ and $k$, and Fig.~\ref{fig:table_high_nk} demonstrates asymptotic runtime advantage for large values of $k$ for fixed $n$.

%\samnote{For Figure 1: Add panel labels \textbf{a} and \textbf{b}; ``Grassman Distance'' should become ``Grassman dist.''; ``seconds'' should become ``sec'' (for both the x axes on left and y axis on right); all variables (n, p, and k) should be italicized. }

\begin{figure}[t]
    \centering
    \begin{overpic}[width=0.54\columnwidth,keepaspectratio]{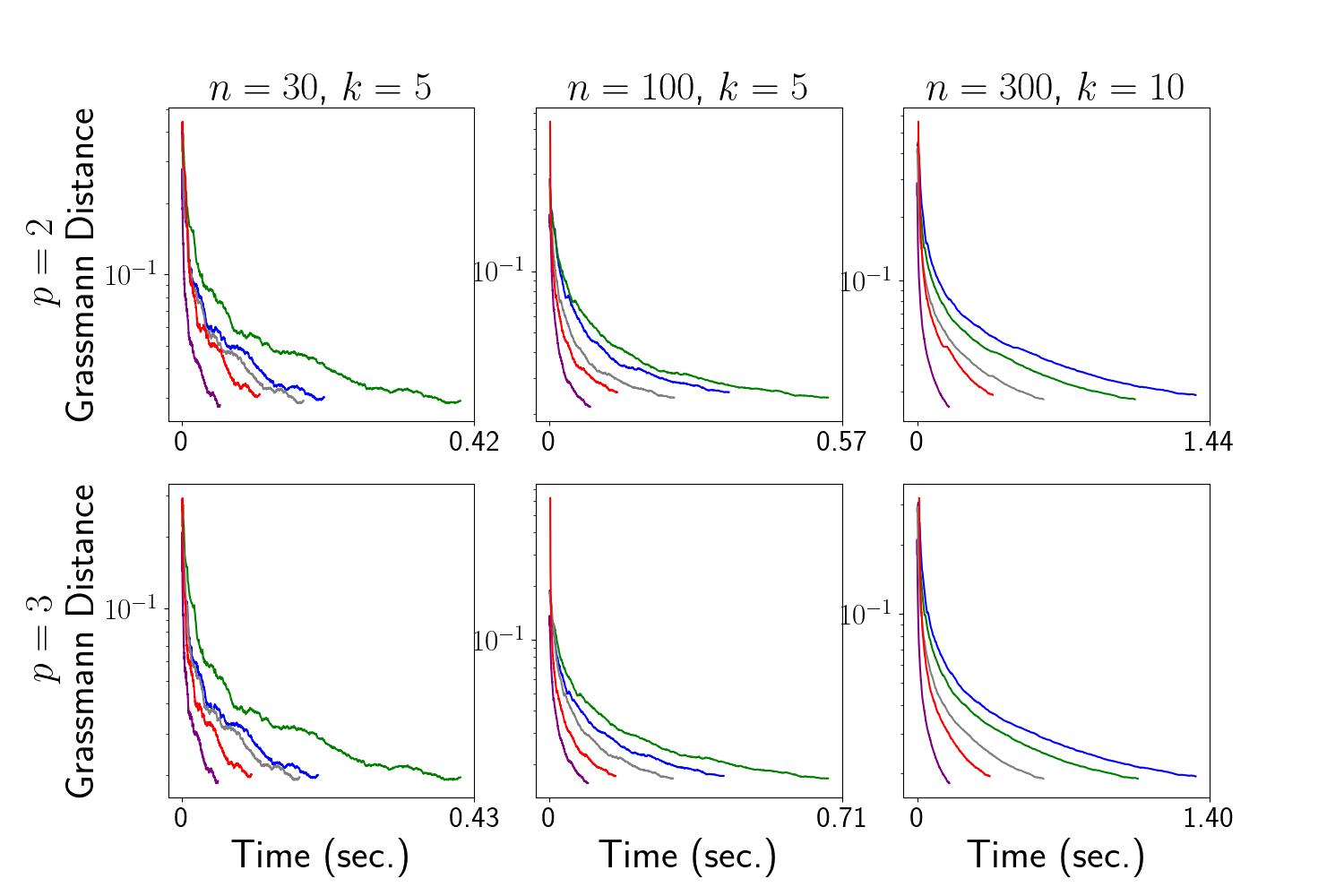}\put(5,70){\textbf{A}}\end{overpic}\begin{overpic}[width=0.54\columnwidth,keepaspectratio]{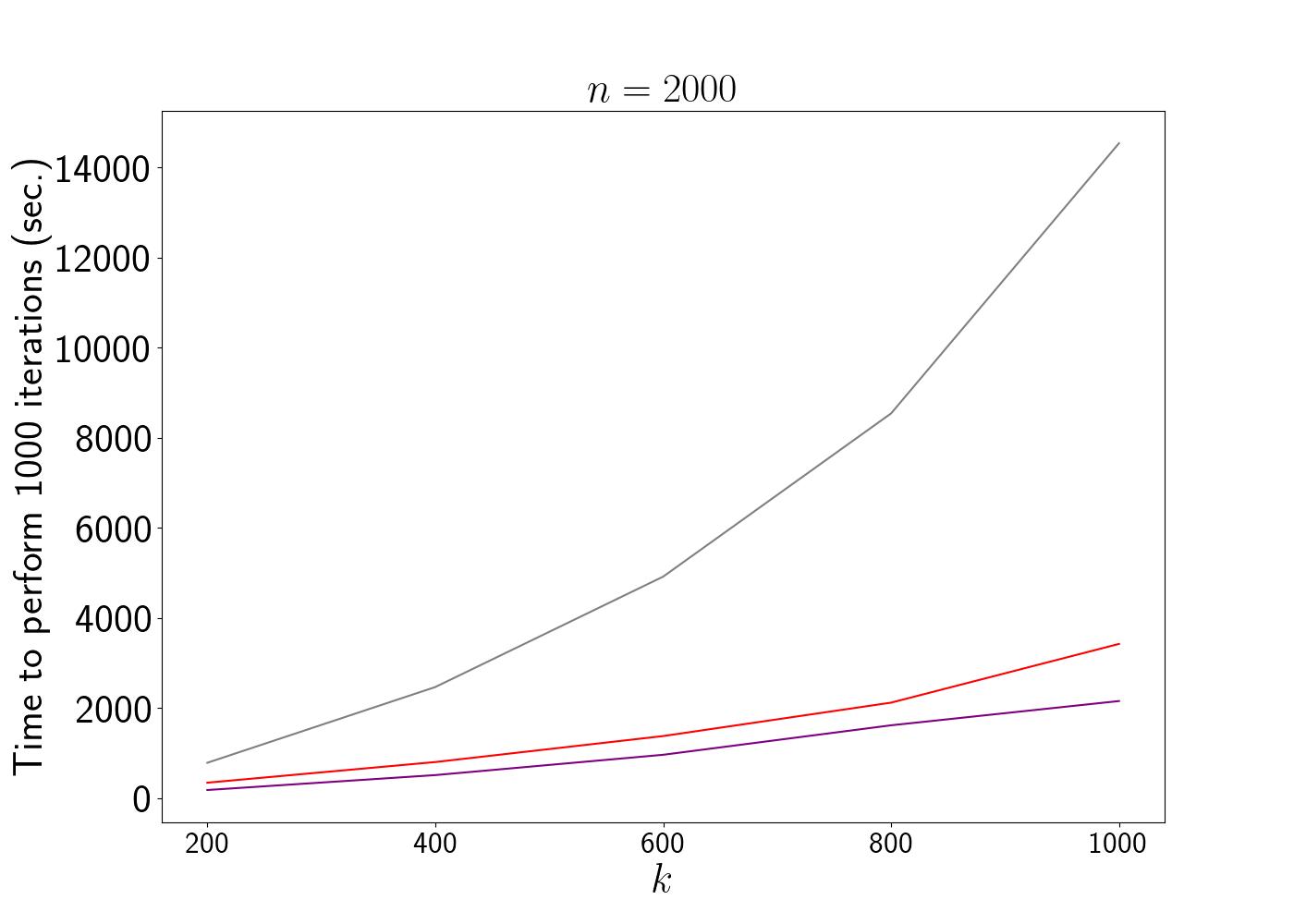}\put(5,70){\textbf{B}}\end{overpic}\vspace{0.8cm}
    \includegraphics[width=\columnwidth, trim={0cm, 1cm, 0cm, 2cm}]{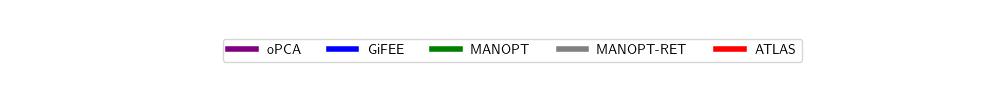}
    \caption{\textbf{Quasi-Euclidean updates in the atlas graph representation of the Grassmannian $\mathbf{Gr}_{n,k}$ converge faster to the population Fr\'echet mean than  other first-order update schemes.} 
    \textbf{A.} Plots show the error in  Grassmann distance (y axis) between the population Fr\'echet mean of $\mathbf{GPD}(\mathcal{X},p)$ and the estimate at each time point (x axis), for each first-order update scheme (color), for varying values of $p$ (rows) and of $n$ and $k$ (columns). \textbf{B.} For samples $\mathbf{X}\sim\mathbf{GPD}(\mathcal{X},p)$ for $n=2,000$ and $p=2$, the plot shows the runtime (y axis) as a function of $k$ (x axis),  for different first-order update schemes (color). MANOPT and GiFEE were excluded due to prohibitively long runtimes.  The analysis focuses on values of $k$ for $k\leq 1000$ because $\mathbf{Gr}_{n,k}$ is homeomorphic to $\mathbf{Gr}_{n,n-k}$. For all panels, 1,000 samples were generated on-line.}
    \label{fig:table_high_nk}
\end{figure}

The results of the experiments demonstrate that our {\atlasfrechet} algorithm has runtime superiority over the other first-order routines in a few ways. The atlas graph framework is faster than all manifold-based frameworks for experiments with $(n,k)$ set to $(30,5)$, $(100,5)$, or $(300,10)$ (Fig.~\ref{fig:table_high_nk}A). Further, for fixed $n$, the runtime of MANOPT-RET appears to grow quadratically with $k$, while the runtime appears to grow merely linearly for {\atlasfrechet} (Fig.~\ref{fig:table_high_nk}B).
%\samnote{Is this supposed to be panel B?}
This implies that even if there exist regimes in which MANOPT-RET is faster, {\atlasfrechet} outspeeds asymptotically.
All methods have similar, high accuracy---measured by geodesic distance between the GiFEE estimator and the population Fr\'echet mean---as a function of the number of iterations. 
Moreover, we note that, even if we are solving a harder, more structured problem in the Fr\'echet mean using manifold-based primitives, the running time of our algorithm tracks closely that of the very efficient oPCA.

\section{Results on learning manifolds from point clouds} \label{sec:manifold_learning}
Turning to point cloud data, we present an approach for constructing an atlas graph representation from a point cloud via ``approximate coordinate charts'' (Sec.~\ref{sec:smoothlyembedded}), followed by algorithms for approximating several differential-geometry primitives (Sec.~\ref{sec:primitives}). We apply these methods to create a representation of the manifold of high-contrast natural image patches that we show preserves homology and pairwise geodesic distances (Sec.~\ref{sec:carlsson_example}).

\subsection{Learning approximate coordinate charts from point clouds}\label{sec:smoothlyembedded}
To generalize the utility of atlas graphs to manifolds lacking known algebraic structure, we present a method for learning approximate coordinate charts from point cloud data (Fig.~\ref{fig:three_subplots}). Consistent with the Manifold Hypothesis, which states that most empirical data distributions observed in nature are well approximated as Riemannian manifolds \citep{fefferman}, we assume that we are given point cloud data sampled from an underlying manifold structure.
%This allows us to learn a manifold structure as an atlas graph from point cloud data lacking any known algebraic structure.
We use previous methods for local linear and quadratic approximations of point cloud data (reviewed in Appendix~\ref{sec:quad_approx_pt_cloud}) to formalize a notion of approximate coordinate charts that can be incorporated into our atlas graph framework. %~\cite{sritharan}. 
%We demonstrate in Section~\ref{sec:carlsson_example} that this method preserves the homology groups and pairwise geodesic distances on Carlsson's manifold of high-contrast patches.

%\subsubsection{Learning local linear and quadratic representations of point cloud data}\label{sec:lin_and_quad}
%The unique tangent plane $T_p\mathcal{M}$ at a point $p\in\mathcal{M}$ is well defined for all $C^1$ differentiable manifolds $\mathcal{M}$. The disjoint union $\bigsqcup_{p\in\mathcal{M}} T_p\mathcal{M}$ is called the \textit{tangent bundle} $T\mathcal{M}$ of $\mathcal{M}$. When $\mathcal{M}$ is a $d$-dimensional smooth manifold\footnote{We use the term ``smooth manifold'' to refer to a $C^\infty$ manifold.} embedded in $\mathbb{R}^n$, it is intuitive to identify $T_p\mathcal{M}$ with a $d$-dimensional hyperplane in $\mathbb{R}^n$ containing $p$. By translation of $\mathbb{R}^n$, we can assume without loss of generality that $p$ is the origin, in which case $T_p\mathcal{M}$ is a $d$-dimensional vector subspace of $\mathbb{R}^n$.

Our approach is based on the premise that, for a set of points $\left\{\vec{x}_1,\ldots,\vec{x}_N\right\}\subset\mathbb{R}^D$ sampled from a manifold $\mathcal{M}\subset\mathbb{R}^D$ and perturbed by sub-Gaussian noise, the tangent plane $T_{\vec{x}}\mathcal{M}$, interpreted as an affine subspace of $\mathbb{R}^D$, can be learned \citep{little}. Since the tangent plane $T_{\vec{x}}\mathcal{M}$ at point $\vec{x}\in\mathcal{U}$ is linearly isomorphic to $\mathbb{R}^d$ (where $d$ is the dimension of $\mathcal{M}$), a coordinate chart $(\mathcal{U},\mathcal{V},\varphi)$ of $\mathcal{M}$ can be learned using an open subset of $T_{\vec{x}}\mathcal{M}$ for $\mathcal{V}$. Then, we find a quadratic polynomial $f:\mathcal{V}\to\mathbb{R}^n$ in the coordinates $\vec{\xi}$ of $\mathcal{V}$ such that, for fixed $\vec{x}\in\mathcal{U}$ and $\vec{\xi}_{\vec{x}}=L^\top\left(\vec{x}-\vec{c}\right)$, the approximation $\vec{x}\approx f\left(\tv_{\vec{x}}\right)$ holds and $f$ has form
\begin{equation}\label{eqn:quad_form_general_body}
    f\left(\tv\right)=\frac{1}{2}MK\left(\tv\otimes\tv\right)+L\tv+\vec{c},
\end{equation}
where $MK$, $L$, and $\vec{c}$ contain the quadratic, linear, and constant terms of the approximation, respectively. Without loss of generality, we assume that $M$ and $L$ are orthogonal Stiefel matrices belonging to $\mathbb{R}^{D\times(D-d)}$ and $\mathbb{R}^{D\times d}$, respectively.

%\samnote{Maybe add a short, intuitive motivation about we need to use approximate coordinate charts, i.e., the $\mathcal{\tilde{U}}$.}
%(for more motivation, see Appendix~\ref{sec:quad_approx_pt_cloud}).
%Using this scheme for local quadratic approximation, we create 
While the compressed chart $\mathcal{V}$ can be learned, an open subset $\mathcal{U}\subset\mathcal{M}$ cannot be. However, for sufficiently small $\mathcal{U}$, an approximation can be learned as the image of a quadratic polynomial. For an \textit{approximate coordinate chart}, instead of creating the coordinate chart $\left(\mathcal{U}\subset\mathcal{M},\mathcal{V},\varphi\right)$, we create the tuple $\left(\tilde{\mathcal{U}}\subset\mathbb{R}^D,\mathcal{V},\tilde{\varphi}\right)$, where:
\begin{enumerate}
    \setlength\itemsep{0em}
    \item $\tilde{\varphi}:\tilde{\mathcal{U}}\to\mathcal{V}$ is a homeomorphism;
    \item $\tilde{\varphi}^{-1}$ is the restriction of a quadratic polynomial in tangential coordinates from $T_p\mathcal{M}$ to its orthocomplement in $\mathbb{R}^D$ (i.e., $\tilde{\varphi}^{-1}=f$ from Equation~\ref{eqn:quad_form_general_body}); and
    \item the uniform measures on $\mathcal{U}$ and $\tilde{\mathcal{U}}$ are close in the sense of Wasserstein geometry\footnote{A quick overview of Wasserstein geometry is given by \citep{ay_2017}.}; this is always true when $\mathcal{U}$ can be approximated as the image of some quadratic function.
\end{enumerate}

%We use $\tv\in\mathcal{V}$ to denote the representative of a point $p\in\mathcal{U}$ in $\mathcal{V}$ (i.e., $\tv=\varphi(p)$), while $\vec{\tau}\in\mathbb{R}^d$ is used to denote a representative tangent vector.

%Under the Euclidean inner product, $T_p\mathcal{M}$ has a unique orthocomplement, which we refer to as $N_p\mathcal{M}$: the \textit{normal plane} of $\mathcal{M}$ at $p$. Just as we can associate with the collection of all tangent planes $T_p\mathcal{M}$ the $2d$-dimensional manifold $T\mathcal{M}$, we can define the normal bundle $N\mathcal{M}:=\bigsqcup_{p\in\mathcal{M}}N_p\mathcal{M}$, a $d(n-d)$-dimensional manifold. Tangent and cotangent bundles are well defined outside the context of smooth manifolds embedded in Euclidean space; however, for many manifolds isometrically embedded into $\mathbb{R}^n$, the definition here is sufficient.

%\subsubsection{Making atlas graphs from local linear and quadratic approximations}\label{sec:applying_lin_quad}

\begin{figure}[t]
    \centering
    \includegraphics[width=0.25\columnwidth, align=c]{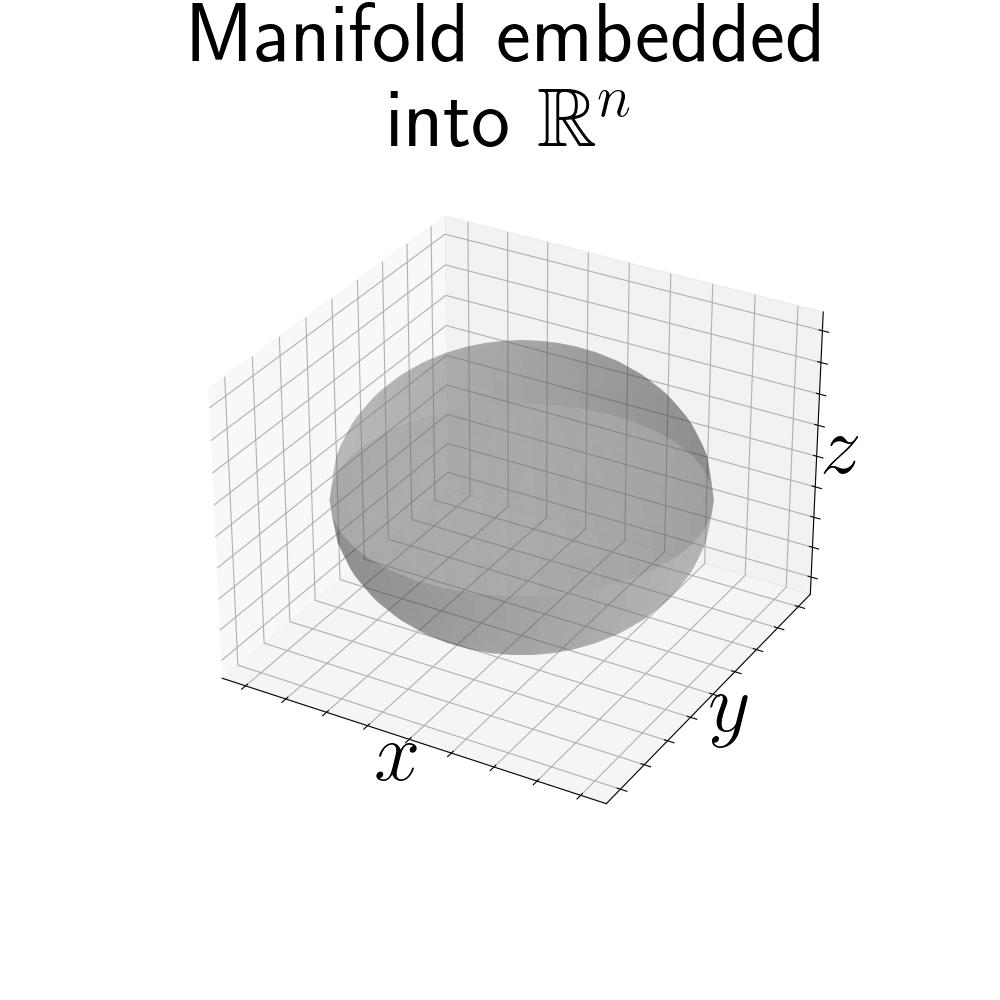}$\xrightarrow[\text{as}]{\text{observed}}$
    \includegraphics[width=0.25\columnwidth, align=c]{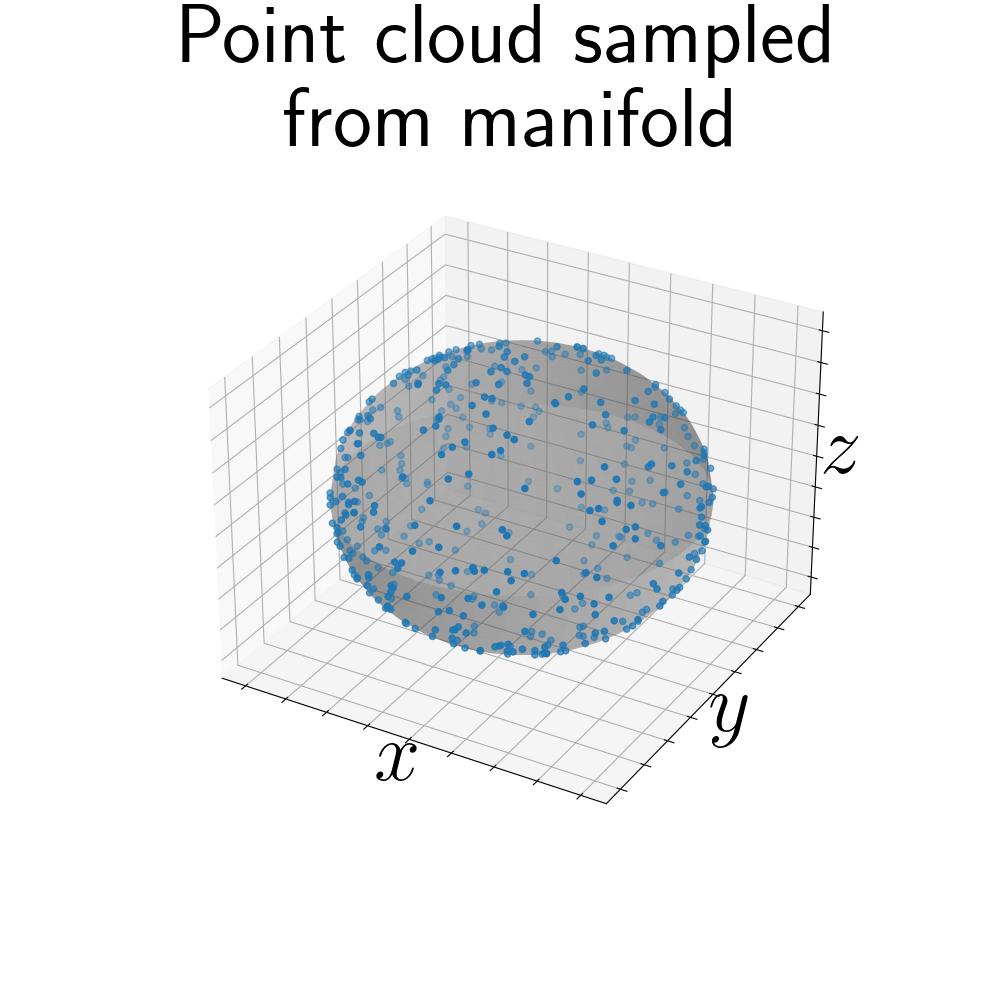}$\xrightarrow[\text{by}]{\text{approximated}}$
    \includegraphics[width=0.25\columnwidth, align=c]{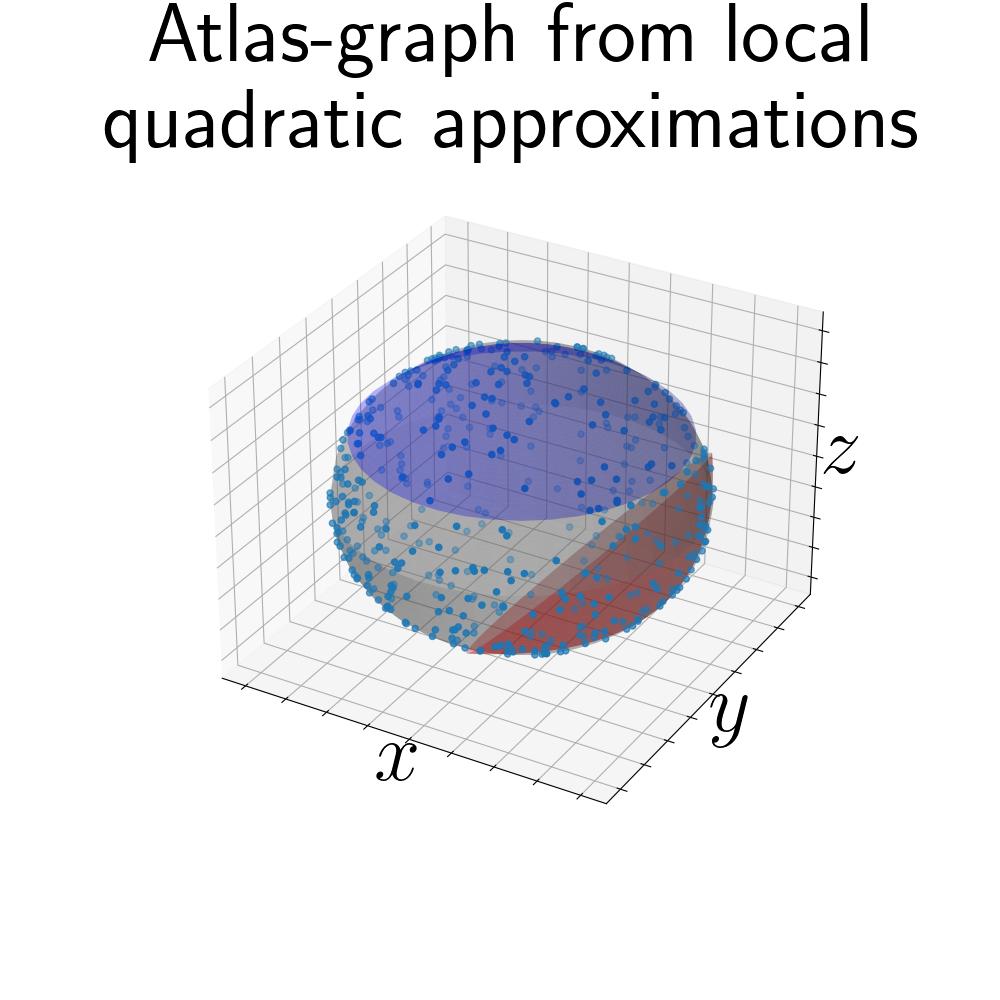}
    \caption{\textbf{Schematic for learning manifold structure from point cloud data using atlas graphs.} The Riemannian geometry of a manifold may not be observed directly but rather revealed via point clouds sampled from the manifold. An atlas graph can be learned from the point cloud to accurately recover the geometry of the underlying manifold. %Any dimensionality reduction technique used on data generated from the manifold should respect the manifold's geometry without artificially inducing or removing key features.
    }
    \label{fig:three_subplots}
\end{figure}

In the case of standard coordinate charts, one can define transition maps as a restriction of the composition $\varphi_\beta\circ\varphi_\alpha^{-1}$; however, this sort of composition of $\tilde{\varphi}_\beta$ and $\tilde{\varphi}_\alpha^{-1}$ is not generally possible for approximate coordinate charts, as the relationship between $\tilde{\mathcal{U}}_\alpha$ and $\tilde{\mathcal{U}}_\beta$ is ill defined. For this reason, in order to construct an atlas graph representation from approximate coordinate charts, we must adapt the notion of a transition map to this context. Given approximate coordinate charts with indices $\alpha,\beta$, we define the \textit{transition map} from $\mathcal{V}_\alpha$ to $\mathcal{V}_\beta$, i.e.,  $\tilde{\psi}_{\alpha\beta}:\mathcal{V}_\alpha\to\mathcal{V}_\beta$, by 
\begin{align*}
    %\tilde{\psi}_{\alpha\beta}:\mathcal{V}_\alpha&\to\mathcal{V}_\beta \\
    \tv&\mapsto L_\beta^\top\left(\frac{1}{2}M_\alpha K_\alpha\left(\tv\otimes\tv\right)+L_\alpha\tv+\vec{c}_\alpha-\vec{c}_\beta\right),
\end{align*}
and $L_\alpha,L_\beta,M_\alpha,K_\alpha,\vec{c}_\alpha$, and $\vec{c}_\beta$ are as in Equation \ref{eqn:quad_form_general_body}. While these maps do not compose (i.e., $\tilde{\psi}_{\alpha\beta}\circ\tilde{\psi}_{\beta\alpha}$ is not generally defined), first-order updates (Sec.~\ref{sec:primitives}) only require maps to be injective, guaranteed under certain Lipschitz conditions (Claim \ref{clm:lipschitz}, Appendix~\ref{app:injectivity}).

\subsection{Computing differential-geometric primitives on point cloud manifolds}\label{sec:primitives}
%Atlas graphs allow for not only the representation of general Riemannian manifolds, but also for the computation of differential-geometric primitives on general Riemannian manifolds~(Section~\ref{sec:atlas_graph}). 
To enable Riemannian optimization on approximate atlas graphs, we generalize the atlas graph computation of differential-geometric primitives, such as quasi-Euclidean updates, to the case of approximate atlas graphs. %We compute several differential-geometric primitives explicitly in terms of the local quadratic approximations of an approximate atlas graph. 
%\samnote{I don't think we need so many back references to previous sections, so I removed them. Also not sure we need the forward reminder in the next sentence again, but I left it for now.} 
We then demonstrate the application of these primitives to the manifold of high-contrast natural image patches (Secs.~\ref{sec:carlsson_example} and~\ref{sec:rpb}). 
%Carlsson manifold~
%(defined in Section~\ref{sec:carlsson_example}) 
%to learn a discriminating boundary between convex and concave image patches (Section~\ref{sec:rpb}).  
%% using the Riemannian principle boundary algorithm from \cite{yao_2020}.

\subsubsection{Quasi-Euclidean updates in approximate atlas graphs}\label{sec:euclidean_updates}
For standard atlas graphs, the action of a representative tangent vector $\vec{\tau}$ upon representative $\vec{\xi}\in\mathcal{V}$ is simply computed as $\vec{\xi}+\vec{\tau}$ under a quasi-Euclidean update (Sec.~\ref{sec:descent}). %, a definition that immediately generalizes to approximate atlas graphs.
%\samnote{This section is frustrating to read. Can you please explicitly present here what you are doing, i.e., rather than implicitly through back references? You need to tie together explicitly the previously presented, separate pieces rather than asking the reader to do this. Otherwise, your subsequent sentences and analysis are hard to interpret. Remind the reader explicitly, what does computing a quasi-Euclidean update look like on an approximate atlas graph? At the possible cost of some redundancy, I think for now it's better to restate things clearly here. We can try afterward to remove unhelpful redundancy.}
In an approximate atlas graph, quasi-Euclidean updates enable computation of geodesics or retractions at the cost of slight numerical inaccuracy.
%with simple floating-point arithmetic at a slight cost in numerical accuracy~(Section~\ref{sec:descent}). \samnote{What floating-point arthmetic are you talking about? I find the previous sentence very confusing.}
While numerical accuracy only depends upon the Christoffel symbols in the atlas-graph case,
%to the geodesic equation increases due to either curvature or poor parametrization by the approximate coordinate chart~(Appendix~\ref{app:quasi_instability}).
it depends upon both the Christoffel symbols and the error of the local quadratic approximation in approximate atlas graphs.

%In this paper, we do not perform an analysis of the error of quasi-Euclidean updates on approximate atlas graphs. 
The error contribution of local quadratic approximations can be reduced, at the potential expense of a larger number of coordinate charts, by ensuring that a neighborhood around each point in the manifold is well approximated by at least one quadratic approximation, i.e., that the manifold is ``well covered'' by approximate coordinate charts. While a theoretical analysis of this complex trade-off, which may depend on local curvature, is left for future work, we show in empirical examples that the trade-off can be managed effectively. Further, in the next section, we reduce the error contributions of both the Christoffel symbols and the quadratic approximations by carefully treating transition boundaries. %(Section \ref{sec:chart_boundaries}).
%transitioning between coordinate charts~(Section~\ref{sec:chart_boundaries}).

\subsubsection{Learning transition boundaries from an ambient SVM}\label{sec:chart_boundaries}
To reduce the numerical inaccuracy contributed by both quadratic approximations and the Christoffel symbols in quasi-Euclidean updates on approximate atlas graphs~(Section~\ref{sec:euclidean_updates}), we need a means of determining when to transition from one coordinate chart to another. Hence, we formalize the notion of a transition boundary. We say that a point $\vec{x}_i\in\mathbb{R}^D$ \emph{belongs}\footnote{As coordinate charts on a manifold can overlap, a point may not belong to a single chart exclusively. The notion of belonging used here is strictly computational.} to approximate coordinate chart $\left(\tilde{\mathcal{U}},\mathcal{V},\tilde{\varphi}\right)$ if $\vec{x}_i$ is better approximated by the quadratic $\tilde{\varphi}^{-1}$ than by the quadratic of any other approximate coordinate chart.
%\samnote{I think you may need to state the problem and your goal more clearly here. I thought you don't just want to learn which points belong to which approximate coordinate charts, according to the definition of ``belong'' you give above because those points might be intermingled in weird ways, right? As presented, it's unclear why you need an SVM type approach.}

For each approximate coordinate chart $\left(\tilde{\mathcal{U}},\mathcal{V},\tilde{\varphi}\right)$, we separate the points belonging to $\mathcal{V}$ from those belonging to other charts using a soft-margin support vector machine (SVM) with a quadratic kernel~(Appendix~\ref{sec:kernel_trick}). The separating hyperplane in the SVM's feature space reduces to the zero-locus of a quadratic in the original, ambient space. As our analysis shows (Lemma \ref{lem:quad_linear_vanish}, Appendix~\ref{sec:kernel_trick}), we can assume
%\footnote{This assumption requires us to translate ambient coordinates; however, this translation can be absorbed into the constant $\vec{c}$ in Equation (\ref{eqn:local_quadratic_approx}).}
that this quadratic takes the form $\vec{x}^\top A\vec{x}+c$ for some symmetric matrix $A$ and constant $c$. 
%Per Equation (\ref{eqn:quad_form_general}), within the coordinate chart $\mathcal{V}$, the relationship between tangential coordinates $\tv$ and ambient coordinates $\vec{x}$ is approximated by the equation
%\begin{equation}\label{eqn:local_quadratic_approx}
%    \vec{x}=\frac{1}{2}M\left[H\left(\tv\otimes\tv\right)+\vec{h}\right]+L\tv+\vec{c}.
%\end{equation}
Substituting Equation~\ref{eqn:quad_form_general_body} into the quadratic $\vec{x}^\top A\vec{x}+c$ gives a quartic polynomial in $\vec{\xi}$ that we evaluate to determine when to invoke a transition map. Specifically, a transition map is invoked when the quartic is positive. The evaluation of this quartic is sped up by computing it in the form 
\begin{equation}\label{eqn:simplified_quartic}
    \left(\tv\otimes\tv\right)^\top A_4\left(\tv\otimes\tv\right)+\tv^\top A_3\left(\tv\otimes\tv\right)+A_{2a}\left(\tv\otimes\tv\right)+\tv^\top A_{2b}\tv+A_1\tv+a,
\end{equation} 
where the following terms are precomputed:
%\begin{align*}
%    a&=c+\frac{1}{4}\vec{c}^\top M^\top AM\vec{c}+\vec{c}^\top M^\top A\vec{c}+\vec{c}^\top A\vec{c}, \\
%    A_1&=\vec{h}^\top MAL+2\vec{c}^\top AL, \\
%    A_{2a}&=\frac{1}{2}\vec{h}^\top M^\top AMH+\vec{c}^\top AMH, \\    A_{2b}&=L^\top AL, \\
%    A_3&=L^\top AMH, \\
%    A_4&=\frac{1}{4}H^\top M^\top AMH
%\end{align*}
\begin{equation*}
    \begin{array}{ccccc}
        a=c+\vec{c}^\top A\vec{c} & \hphantom{1cm} & A_1=2\vec{c}^\top AL & \hphantom{1cm} & A_{2a}=\vec{c}^\top AMH \\
        A_{2b}=L^\top AL & \hphantom{1cm} & A_3=L^\top AMH & \hphantom{1cm} & A_4=\frac{1}{4}H^\top M^\top AMH.
    \end{array}
\end{equation*}
%\begin{align*}
%    , \\
%    , \\
%    A_{2a}&=\vec{c}^\top AMH, \\    A_{2b}&=L^\top AL, \\
%    A_3&=L^\top AMH, \\
%    A_4&=\frac{1}{4}H^\top M^\top AMH
%   \end{align*}
In addition to speeding up computation of the quartic, precomputing these terms is desirable from the perspective of space complexity, as none of the matrices in the reduced quartic have shapes contingent on the ambient dimension of the data.

\subsubsection{Approximate path lengths of quasi-Euclidean updates}\label{sec:app_naive_dist}
In Riemannian optimization algorithms that iterate over manifold-valued data, a loss function may weigh the contribution of each datum according to its distance from the current iterate \citep[e.g., stochastic Riemannian gradient descent,][]{hosseini_and_sra}. In regimes where quasi-Euclidean updates approximate the exponential map well, the paths traversed by quasi-Euclidean updates approximate geodesic paths well. Therefore, the lengths of quasi-Euclidean paths with respect to the metric inherited from the ambient space can be used to estimate (albeit overestimate\footnote{Overestimation results from the definition of geodesics as length-minimizing paths on a Riemannian manifold}) geodesic distances. For $\tv_0,\tv_1\in\mathcal{V}_i$, we refer to
\begin{equation}\label{eqn:naive_approx_dist}
    \tilde{d}_i\left(\tv_0,\tv_1\right)=\int_0^1\left(\left(\vec{\xi}_1-\vec{\xi}_0\right)^\top\left(\vec{\xi}_1-\vec{\xi}_0\right)+\sum_{j=1}^{D-d}\left(\left(\vec{\xi}_1-\vec{\xi}_0\right)^\top K_{ij}\left((1-t)\tv_0+t\tv_1\right)\right)^2\right)^{\frac{1}{2}}dt
\end{equation}
as the \textit{na\"ive approximate distance} between $\tv_0$ and $\tv_1$. A closed form is in Appendix~\ref{app:naive_app_dist}. 

To define the naive approximate distances between two points $\vec{\xi}_1\in\mathcal{V}_1,\vec{\xi}_2\in\mathcal{V}_2$ that are not in the same coordinate chart 
%(Fig.~\ref{fig:dist_fig}) \samnote{Removing this figure ref here because it is standard practice to avoid referring to later numbered figs before first references to earlier numbered figs}
we use the shortest-path distance between their ambient representations $\vec{x}_1$ and $\vec{x}_2$, respectively, in a nearest-neighbor graph $G_{\text{NN}}$ created on points densely sampled from the approximate atlas graph (Algorithm~\ref{alg:dense_graph}), with edges weighted by na\"ive approximate distance.

\begin{algorithm}[t]
    \caption{Create a dense, unweighted subgraph for geodesic approximation}\label{alg:dense_graph}
    \begin{algorithmic}
        \Require Quad. approx. matrices $M_i,K_i,L_i,\vec{c}_i$ for all charts $i$\Comment{Equation \ref{eqn:quad_form_general_body}}
        \Require Chart radii $r_1,\ldots,r_k>0$
        \Require Boundary functions $f_1,\ldots,f_k:\mathbb{R}^d\to\mathbb{R}$\Comment{Equation~\ref{eqn:simplified_quartic}}
        \Require Vertex density parameter $\delta>0$\Comment{We use $\delta=0.1$}
        \Require Distance threshold parameter $\epsilon>0$\Comment{We use $\epsilon=0.6$}
        \State \texttt{V}$\gets\text{empty array}$
        \For{$i\gets1,\ldots,k$}
            \State $S\gets\left\{n\delta\left\lvert n\in\mathbb{Z}\cap\left(-\frac{r_i}{\delta},\frac{r_i}{\delta}\right)\right.\right\}$ \Comment{Evenly spaced points in $\mathbb{R}$}
            \For{$\vec{\xi}\in S^d$} \Comment{$S^d$ is the $d$th Cartesian power of $S$}
                \State $\vec{x}\gets\frac{1}{2}M_iK_i\left(\vec{\xi}\otimes\vec{\xi}\right)+L_i\vec{\xi}+\vec{c}_i$\Comment{Ambient representation of $\vec{\xi}$; Eq.~(\ref{eqn:quad_form_general_body})}
                \If{$f_i\left(\vec{x}\right)<0$}\Comment{Check membership of $\vec{x}$ in chart $i$; Eq.~(\ref{eqn:simplified_quartic})}
                    \State $\texttt{V.append}\left(\vec{x}\right)$
                \EndIf
            \EndFor
        \EndFor
        \State \texttt{E}$\gets\text{empty array}$
        \For{$\vec{x}\in\texttt{V}$}
            \For{$\vec{x}^\prime\in\texttt{V}$}
                \If{$\vec{x}\neq\vec{x}^\prime$}
                    \State $w\gets\left\lVert\vec{x}-\vec{x}^\prime\right\rVert_2$
                    \If{$w<\epsilon$} \Comment{Check that points are sufficiently close}
                        \State \texttt{E.append}$\left(\vec{x},\vec{x}^\prime\right)$
                    \EndIf
                \EndIf
            \EndFor
        \EndFor
        \State\Return $G_{\text{NN}}=(\texttt{V},\texttt{E})$
    \end{algorithmic}
\end{algorithm}

\subsubsection{Vector transport}\label{sec:vector_transport}
The computation of vector transports is essential for the implementation of Riemannian optimization routines for general manifolds \citep[e.g.,][]{hosseini_and_sra}. For approximating vector transports, we use two techniques: one for transporting a tangent vector between points within the same coordinate chart, and one for transporting a tangent vector between points that are close but belong to separate coordinate charts. We can assume these two cases without loss of generality, as vector transport between distant points can be approximated by iterated vector transports between near points, though this could lead to significant accumulation of numerical error. Moreover, several first-order Riemannian optimization techniques do not require vector transports across long distances \citep[e.g.,][]{hosseini_and_sra}. If points $\vec{x}$ and $\vec{x}^\prime$ belong to the same chart $\left(\mathcal{U}_\alpha,\mathcal{V}_\alpha,\tilde{\varphi}_\alpha\right)$, then we approximate the parallel transport 
%is given by Equation~\ref{eqn:app_vec_trans}.
%according to Section~\ref{sec:vector_transport}.
as $\vec{\xi}^\prime-\vec{\xi}$, where $\vec{\xi}$ and $\vec{\xi}^\prime$ are the representatives of $\vec{x}$ and $\vec{x}^\prime$, respectively, as in Section~\ref{sec:vector_transport_def}. If $\vec{x}\in\mathcal{V}_\alpha$ and $\vec{x}^\prime\in\mathcal{V}_\beta$ belong to distinct charts,
%and $\vec{x},\vec{x}^\prime$ are represented by $\tv\in\mathcal{V}_\alpha,\tv^\prime\in\mathcal{V}_\beta$, respectively,
then vector transport from $T_{\vec{x}}\mathcal{M}$ to $T_{\vec{x}^\prime}\mathcal{M}$ is approximated by taking representative tangent vector $\vec{\tau}$ to representative tangent vector $\left[D\tilde{\varphi}_\beta^{-1}\left(\tv^\prime\right)\right]^{-1}\left[D\tilde{\varphi}_\alpha^{-1}\left(\tv\right)\right]\left(\vec{\tau}\right)$, where $D\tilde{\varphi}^{-1}$ denotes the differential of $\tilde{\varphi}^{-1}$.
%$\left(\tilde{L}^\prime\right)^\dag\tilde{L}\vec{\tau}$, where $^\dag$ denotes the Moore-Penrose pseudoinverse \cite{ben_israel_2003}, and
% \begin{equation}\label{eqn:complex_vec_trans}
%     \left(\tilde{L}^\prime\right)^\dag\tilde{L}\vec{\tau},
% \end{equation}
%$\tilde{L}$ and $\tilde{L}^\prime$ are the differentials of Equation (\ref{eqn:quad_form_general_body}) given by
%$$\tilde{L}=M_\alpha K_\alpha\left(\tv\otimes I_d+I_d\otimes\tv\right)+L_\alpha,\hspace{1cm}\tilde{L}^\prime=M_\beta K_\beta\left(\tv^\prime\otimes I_d+I_d\otimes\tv^\prime\right)+L_\beta.$$

\subsubsection{Approximate Riemannian Logarithm}\label{sec:riem_log}

\begin{algorithm}[t]
    \caption{Approximate Riemannian logarithm}\label{alg:riem_log}
    \begin{algorithmic}
        \Require Dense subgraph $G_{\text{NN}}=(V\subset\mathcal{M},E\subset V\times V)$\Comment{Alg. \ref{alg:dense_graph}}
        \Require Weight function $W:E\to\mathbb{R}$ given by na\"ive approx. dist. \Comment{Sec.~\ref{sec:app_naive_dist}}
        \Require Quad. approx. matrices $M_i,K_i,L_i,\vec{c}_i$ for all charts $i$\Comment{Equation \ref{eqn:quad_form_general_body}}
        \Require basepoint $\vec{\xi}$ in chart $i$
        \Require point $\vec{\xi}^*$ in chart $i^*$, of which to take Riemannian logarithm
        \State $\vec{x}\gets\frac{1}{2}M_iK_i\left(\vec{\xi}\otimes\vec{\xi}\right)+L_i\vec{\xi}+\vec{c}_i$ \Comment{Ambient representation of $\vec{\xi}$}
        %\State $\vec{x}\gets\text{ambient representation of }\vec{\xi}$
        \State $\vec{x}^*\gets\frac{1}{2}M_{i^*}K_{i*}\left(\vec{\xi}\otimes\vec{\xi}\right)+L_{i^*}\vec{\xi}+\vec{c}_{i^*}$ \Comment{Ambient representation of $\vec{\xi}^*$}
        %\State $\vec{x}^*\gets\text{ambient representation of }\vec{\xi}^*$
        \State $\vec{v}\gets\text{point in }V\text{ closest to }\vec{x}$
        \State $\vec{v}^*\gets\text{point in }V\text{ closest to }\vec{x}^*$
        \State $\left(\vec{v},\vec{v}_1,\ldots,\vec{v}_j,\vec{v}^*\right)\gets\text{shortest path from }\vec{v}\text{ to }\vec{v}^*\text{ in }(V,E,W)$
        \LeftComment{$\vec{\xi}^\prime$ stores the cumulative Riem. log to be returned}
        \LeftComment{$\vec{\xi}_{\text{dq}}$ is the next point of which to compute the Riem. log locally; ``dq''$=$``dequeue''}
        \State $\xi^\prime\gets\vec{0}$ \Comment{Initialize Riem. log. to zero}
        \State $\vec{\xi}_{\texttt{dq}}\gets\vec{\xi}^*$ \Comment{First point of which to take local log. is $\vec{\xi}^\prime$}
        \State $i_\texttt{dq}\gets\text{chart index of }\vec{\xi}^*$
        %\State \texttt{dist}$\gets w(\vec{v}_j,\vec{v}^*)$
        \For{$\tilde{v}\gets\vec{v}^*,\vec{v}_j,\ldots,\vec{v}_1,\vec{v}$}
            \State $\tilde{i}\gets\text{chart index of }\tilde{v}$
            \State $\tilde{\xi}\gets\text{representation of }\tilde{v}\text{ in chart }\tilde{i}$
            \If{$i_\texttt{dq}=\tilde{i}$}\Comment{If both points are in the same chart}
                \State $\vec{\xi}^\prime\gets\vec{\xi}^\prime+\left(\vec{\xi}_\texttt{dq}-\tilde{\xi}\right)$ \Comment{Increment cumulative log. by local log.}
                %\State \texttt{dist}$\gets\texttt{dist}+\texttt{naive\_approximate\_distance}_{\tilde{i}}(\tilde{\xi},\vec{\xi}_\texttt{dq})$
                \State $\vec{\xi}_\texttt{dq}\gets\tilde{\xi}$ \Comment{The base of the local log. becomes operand of next local log.}
            \Else
                \LeftComment{We assume by density of $G_\text{NN}$ that there is a valid representative in chart $\tilde{i}$}
                \State $\vec{\xi}_\texttt{dq}\gets L_{\tilde{i}}^\top\left(\frac{1}{2}M_{i_\texttt{dq}}K_{i_\texttt{dq}}\left(\vec{\xi}_\texttt{dq}\otimes\vec{\xi}_\texttt{dq}\right)+L_{i_\texttt{dq}}\vec{\xi}_\texttt{dq}+\vec{c}_{i_\texttt{dq}}\right)$ \Comment{Rep. of $\vec{\xi}_\texttt{dq}$ in chart $\tilde{i}$}
                %\State $\vec{\xi}_\texttt{dq}\gets\text{representation of }\left(\vec{\xi}_\texttt{dq},i_\texttt{dq}\right)\text{ in chart }\tilde{i}$
                \State $\vec{\xi}^\prime\gets L_{\tilde{i}}^\top\left(\frac{1}{2}M_{i^\prime}K_{i^\prime}\left(\vec{\xi}^\prime\otimes\vec{\xi}^\prime\right)+L_{i^\prime}\vec{\xi}^\prime+\vec{c}_{i^\prime}\right)$ \Comment{Rep. of $\vec{\xi}^\prime$ in chart $\tilde{i}$}
                %\State $\vec{\xi}^\prime\gets\text{representation of }\vec{\xi}^\prime\text{ in chart }\tilde{i}$
                %\State \texttt{dist}$\gets\texttt{dist}+\texttt{naive\_approximate\_distance}_{\tilde{i}}(\tilde{\xi},\vec{\xi}_\texttt{dq})$
                \State $\vec{\xi}^\prime\gets\vec{\xi}^\prime+\left(\vec{\xi}_\texttt{dq}-\tilde{\xi}\right)$ \Comment{Increment cumulative log. by local log.}
                \State $i_\texttt{dq}\gets\tilde{i}$
            \EndIf
        \EndFor
        \State $\vec{\xi}^\prime\gets\vec{\xi}^\prime-\vec{\xi}$ \Comment{Lastly, increment log. by $\mathbf{Log}^\mathbf{Ret}_{\vec{\xi}}\left(\vec{\xi}^\prime\right)$}
        \State\Return $\vec{\xi}^\prime$
    \end{algorithmic}
\end{algorithm}

When performing variants of Riemannian gradient descent on a loss function determined by manifold-valued data \citep[as surveyed in][]{hosseini_and_sra}, the contribution of each datum to the Riemannian gradient of the loss function involves computing retraction logarithms~(Section~\ref{sec:differentiable_mfld}). For the quasi-Euclidean retraction on an atlas graph, retraction logarithms can be computed easily. Specifically, for points $\vec{\xi}_1, \vec{\xi}_2$ belonging to the same compressed chart $\mathcal{V}$, the retraction logarithm of $\vec{\xi}_2$ at $\vec{\xi}_1$, or $\mathbf{Log}^{\mathbf{Ret}}_{\vec{\xi}_1}(\vec{\xi}_2)$, is simply $\vec{\xi}_2-\vec{\xi}_1$; for points $\vec{\xi}_1, \vec{\xi}_2$ that do not belong to the same compressed chart, we approximate $\mathbf{Log}^{\mathbf{Ret}}_{\vec{\xi}_1}(\vec{\xi}_2)$ using the shortest path between $\vec{\xi}_1$ and $\vec{\xi}_2$ on the nearest neighbor graph $G_{\text{NN}}$ of densely sampled points from the approximate atlas graph (Sec.~\ref{sec:app_naive_dist}, Algorithm~\ref{alg:dense_graph}), by iteratively summing and vector-transporting the edgewise retraction logarithms from the endpoint of the path back to the start (Algorithm~\ref{alg:riem_log}).

\subsection{Atlas graph representation of high-contrast image patches}\label{sec:carlsson_example}
\begin{figure}[t]
    \centering
    \begin{subfigure}{0.5\columnwidth}
        \begin{overpic}[width=\columnwidth]{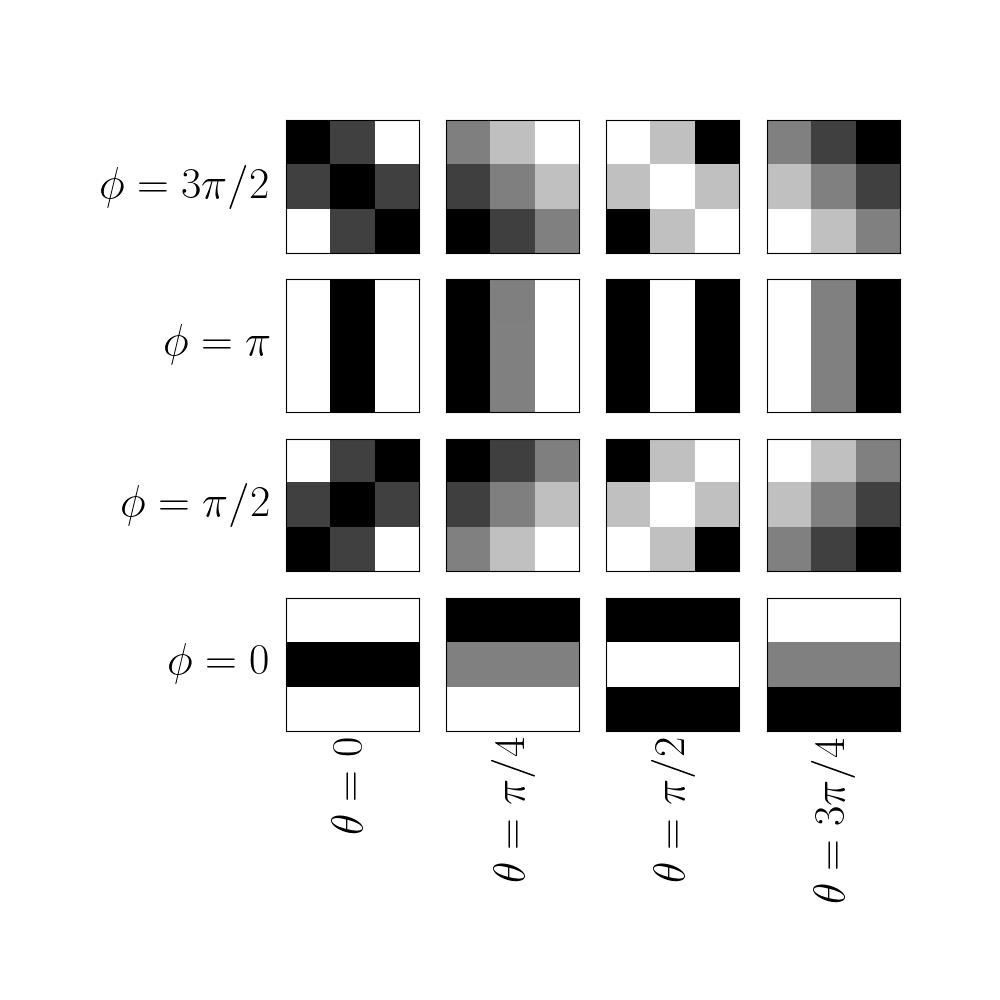}\put(5,95){\textbf{A}}\end{overpic}
    \end{subfigure}
    \begin{subfigure}{0.45\columnwidth}
        \begin{overpic}[width=\columnwidth]{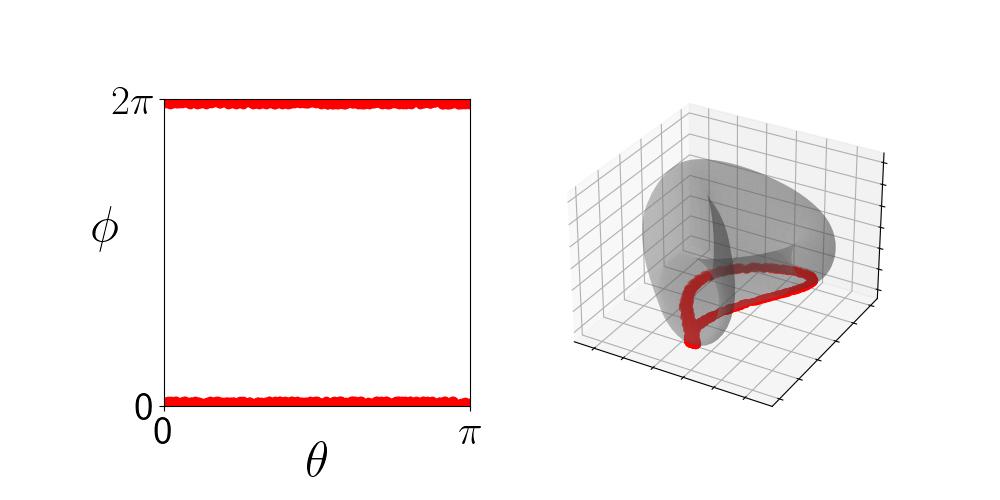}\put(0,50){\textbf{B}}\end{overpic} \\
        \begin{overpic}[width=\columnwidth]{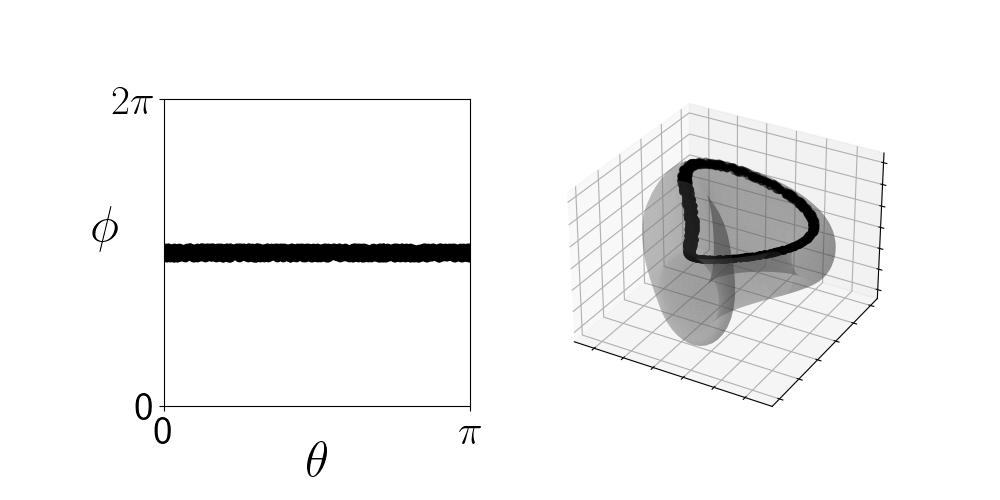}\put(0,50){\textbf{C}}\end{overpic}
    \end{subfigure}
    \caption{\textbf{In the manifold of Carlsson's $3 \times 3$ high-contrast image patches, convex and concave patches are disjoint classes separated by a natural boundary.} \textbf{A.} Examples of high-contrast patches in  the  manifold, parameterized as in Sec.~\ref{sec:high_contrast}, including convex patches (column ``$\theta=0$'') and concave patches (column  ``$\theta=\pi/2$''). \textbf{B.} Samples of convex patches (red), shown using the polar coordinate representation (left) and the Karcher visualization (right) of the Klein bottle. \textbf{C.} Analogous plots for concave patches (black).}
    \label{fig:pos_class}
\end{figure}
To demonstrate learning an atlas graph representation using approximate coordinate charts, we apply it to data points sampled from a manifold of high-contrast, natural image patches, which is homeomorphic to the Klein bottle \citep{carlsson}. % demonstrate learning an atlas graph representation from point cloud data preserves both the homology and pairwise geodesic distances of the manifold underlying the point cloud. 
%The manifold from which we sample points is a standard manifold representation of Carlsson's natural image patches that is homeomorphic to the Klein bottle ~\cite{carlsson}. 
These data are useful for methods development, as they lie on a manifold of previously characterized topology \citep{carlsson} that permits simple visualizations, as it has intrinsic dimensionality of two. 

First, we briefly review a parameterization of the manifold of high-contrast patches \citep{sritharan}. 
We then leverage this parameterization, our methodology (Sec.~\ref{sec:smoothlyembedded}), and a specific cover of the manifold to construct an atlas graph representation of approximate coordinate charts.  
Lastly, we demonstrate that this representation preserves the homology of the underlying manifold, as well as geodesic distances between points.

\subsubsection{Carlsson's high-contrast patches}\label{sec:high_contrast}

%We now review the geometry of the Carlsson manifold according to a construction given by \cite{sritharan}. 
The Klein bottle can be parameterized by the polynomials  $k_{\theta,\phi}:\mathbb{R}^2 \to\mathbb{R}$ such that
\begin{equation*}
    k_{\theta,\phi}(x,y) \mapsto\cos\phi\left(x\cos\theta+y\sin\theta\right)^2+\sin\phi\left(x\cos\theta+y\sin\theta\right)
\end{equation*}
for $(\theta,\phi)\in[0,\pi]\times[0,2\pi]$ \citep{sritharan}, which is identified with the space of $3\times3$ high-contrast image 
patches by restricting the arguments of $k_{\theta,\phi}$ to $\{-1,0,1\}\times\{-1,0,1\}$.
    
This parametrization naturally splits the $3\times3$ patches into convex and concave classes, inducing a corresponding classification task (Fig.~\ref{fig:pos_class}) that we address later (Sec.~\ref{sec:rpb}). Fixing $\theta=0$, we consider the polynomials $k_{0,\phi}(x,y)$,
\begin{equation*}
    k_{0,\phi}:(x,y)\mapsto x^2\cos\phi+x\sin\phi.
\end{equation*}
These polynomials, which are constant in $y$ and quadratic in $x$, fail to be monotonic with respect to $x$ on $[-1,1]$ if and only if $-2<\tan\phi<2$. Those that are not monotonic then fall into two groups, depending on whether $\cos\phi$ is positive or negative, corresponding to convexity and concavity, respectively, in $x$. We therefore refer to patches corresponding to polynomials $k_{\theta,\phi}$ that satisfy $-2<\tan\phi<2$ as \textit{convex} if $\cos\phi>0$ and \textit{concave} otherwise. \\

\begin{figure}[t]
    \centering
    \begin{overpic}[width=0.45\columnwidth,keepaspectratio]{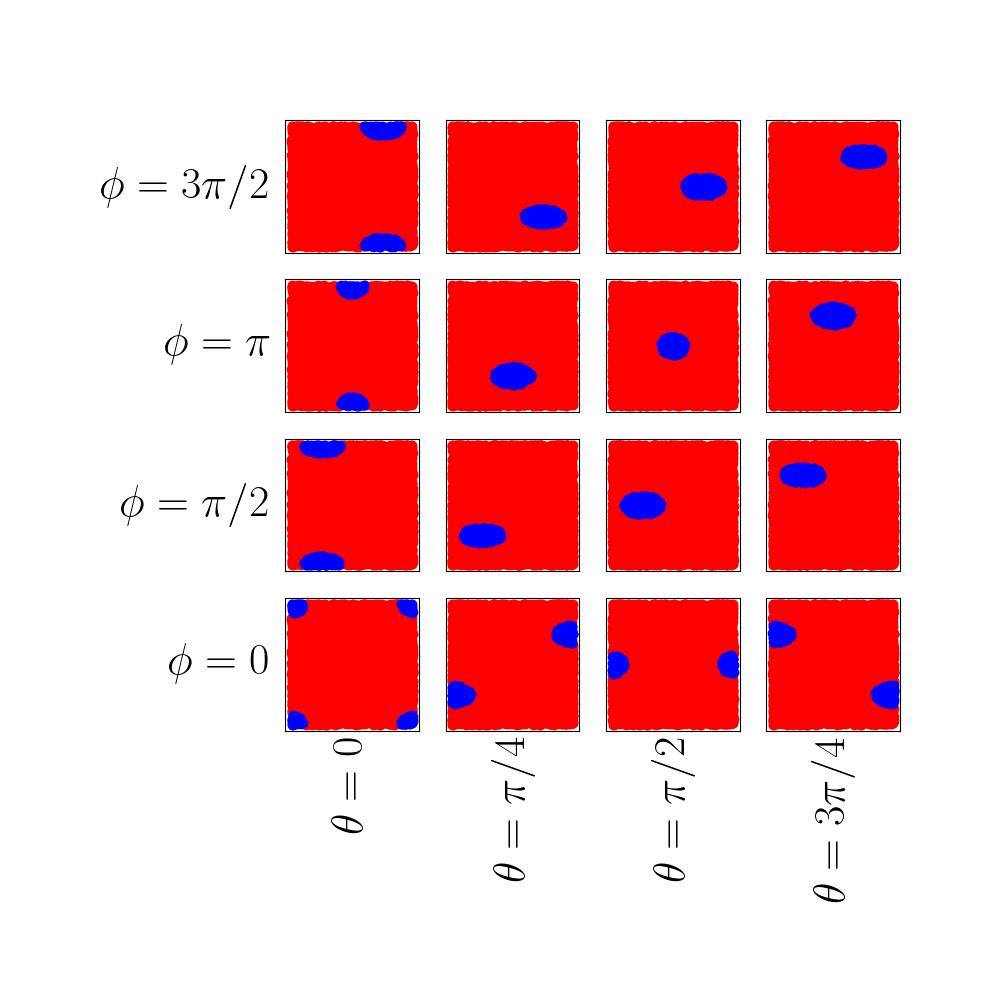} \put(5,95){\textbf{A}} \end{overpic}
    \begin{overpic}[width=0.45\columnwidth,keepaspectratio]{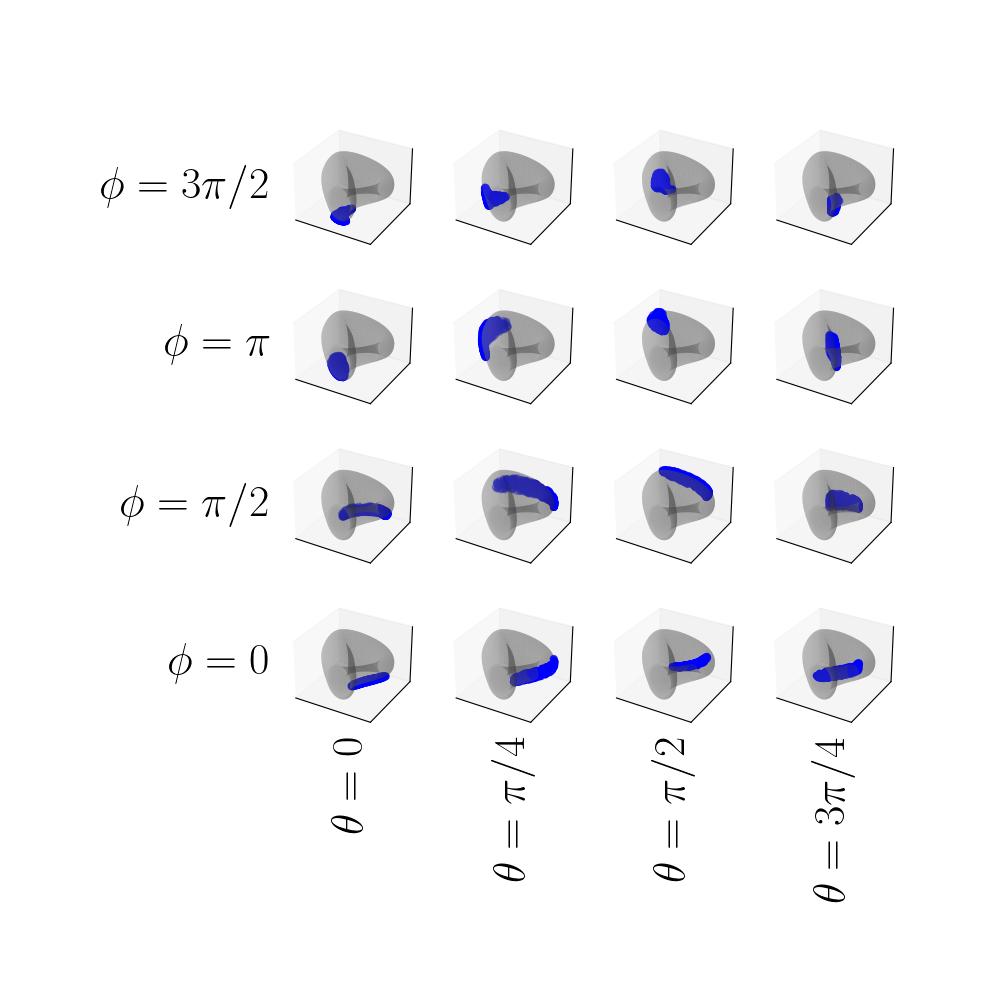}\put(5,95){\textbf{B}}
    \end{overpic}
    \caption{\textbf{Approximate coordinate charts %using Algorithm~\ref{alg:carlsson_atlas graph}
    used in the chart cover of the high-contrast image patches manifold} (Sec.~\ref{sec:atlas_graph_for_carlsson}). \textbf{A.} Approximate coordinate charts in the atlas graph representation of the manifold are depicted in polar coordinates and indexed by the polar coordinates of the chart center $k(\theta,\phi)$ (Sec.~\ref{sec:high_contrast}). For each chart, points are colored blue  if they belong to the learned neighborhood of the chart center, %of the point sample
    and red otherwise. \textbf{B.} The same coordinate charts are depicted on the Karcher representation of the Klein bottle. The full set of approximate coordinate charts appears in Fig.~\ref{fig:carlsson_atlas graph}.}
    \label{fig:carlsson_atlas_partial}
\end{figure}

\subsubsection{Constructing an approximate atlas graph of the manifold}\label{sec:atlas_graph_for_carlsson}
Using the notion of approximate coordinate charts, we construct an approximate atlas graph representation of the high-contrast $3 \times 3$ natural-image patches.
%, according to Algorithm~\ref{alg:carlsson_atlas graph}.
Similarly to our approach for constructing the atlas graph of the Grassman manifold from a known cover of the manifold (Sec.~\ref{sec:grass_example}), we construct an approximate atlas graph for the high-contrast image patches using a pre-defined cover of the manifold. Specifically, we select chart center points to be the centers of an $8\times 8$ grid in the polar coordinate representation of the Klein bottle. We then densely sample the manifold near these center points and construct local quadratic approximations from them
%(Fig.~\ref{fig:carlsson_atlas_partial}, Algorithm~\ref{alg:carlsson_atlas graph}).
(Fig.~\ref{fig:carlsson_atlas_partial}).
%\samnote{I'm confused – aren't these local quadratic approximations computed from sample points? You first need to discuss the sampled points that are given as the input ``point cloud'', no? Make sure it's clear whether these are used again later or these are separate.}

To investigate how well the approximate atlas graph preserves homology and geometry, we generated samples from the atlas graph representation and from the Klein bottle parameterization \citep{sritharan} (Appendix~\ref{sec:persistent_homology}). We computed persistent homology on each, as well as on the latter samples after transformation, via several state-of-the-art dimensionality reduction methods. The results show that, unlike the other methods, the atlas graph preserves homological features of the natural image patches, as measured by an aggregate bottleneck distance we define (Fig.~\ref{fig:h2_homology}, Appendix~\ref{app:bottleneck_heatmap}). Geodesic distances between pairs of points were also better preserved by the atlas graph representation than by other dimensionality reduction methods (Fig.~\ref{fig:dist_fig}, Appendix~\ref{sec:geodesic_distances}). Moreover, the local quadratic approximations used to construct the approximate coordinate charts preserve scalar curvature and higher eigenvalues of the Laplace-Beltrami operator \citep{sritharan}.

%The atlas graph approximation constructed here is used in Section \ref{sec:rpb} to implement a Riemannian generalization of the support vector machine (SVM) to Carlsson's natural image patches.
%\samnote{Should we combine Figs 5 and 6?}
\begin{figure}[t]
    \centering
    \begin{overpic}[width=0.33\columnwidth, trim=48 48 48 48, clip]{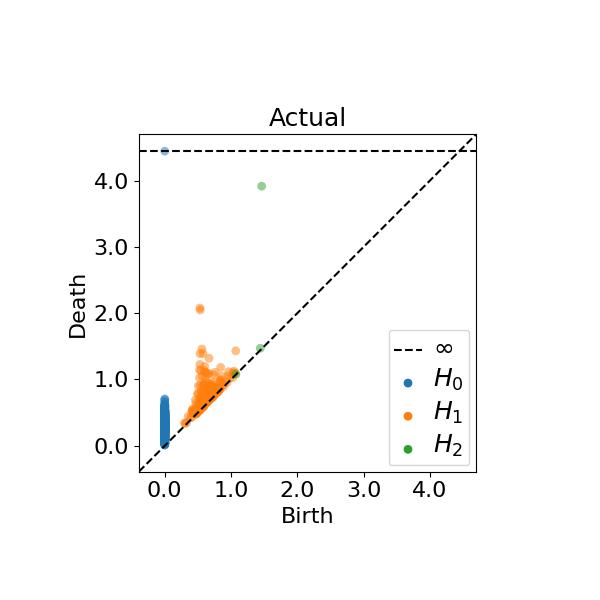}\put(5,95){\textbf{A}}\end{overpic}\begin{overpic}[width=0.33\columnwidth, trim=48 48 48 48, clip]{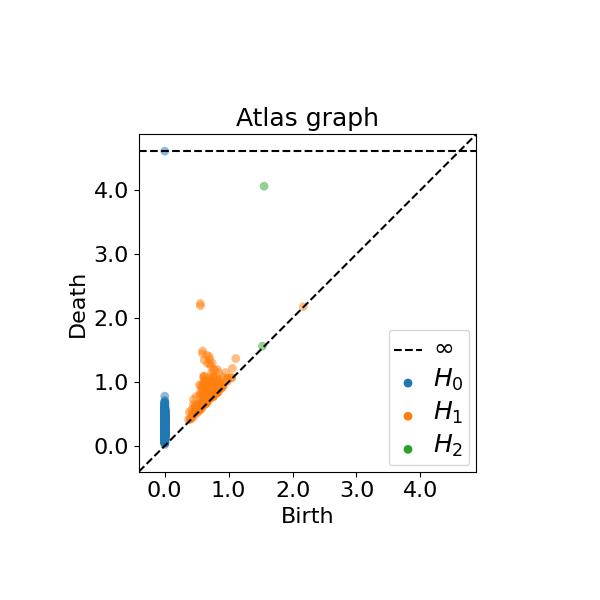}\put(5,95){\textbf{B}}\end{overpic}\begin{overpic}[width=0.33\columnwidth, trim=20 20 20 20, clip]{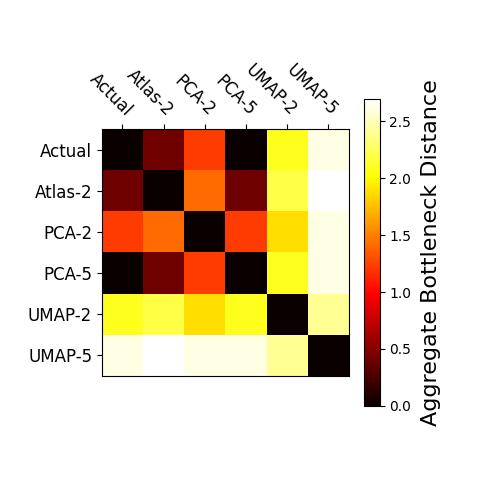}\put(5,95){\textbf{C}}\end{overpic}
    % \begin{overpic}[width=0.33\columnwidth]{graphics/klein/homology/h2_pca.jpg}\put(5,90){\textbf{C}}\end{overpic}\begin{overpic}[width=0.33\columnwidth]{graphics/klein/homology/h2_tsne.jpg}\put(5,90){\textbf{D}}\end{overpic}\begin{overpic}[width=0.33\columnwidth]{graphics/klein/homology/h2_umap.jpg}\put(5,90){\textbf{E}}\end{overpic}
    \caption{\textbf{Our atlas graph representation of the high-contrast image patches manifold preserves its H1 and H2 homology groups.} \textbf{A,B.} The persistence diagram computed from points sampled from the 
    parameterization of the Klein bottle (A) closely matches that computed from points sampled from our atlas graph representation (B) (Appendix~\ref{sec:persistent_homology}). \textbf{C.} The heatmap shows pairwise aggregate bottleneck distance (Appendix~\ref{app:bottleneck_heatmap}) between persistence diagrams generated from different dimensionality reductions (rows, columns). PCA-2, PCA-5: the first 2 and 5 principal components, respectively; UMAP-2, UMAP-5: 2-dimensional and 5-dimensional UMAP embeddings, respectively; Atlas-2: our atlas graph representation, which uses 2-dimensional coordinate charts.
    }\label{fig:h2_homology}
\end{figure}

% \begin{algorithm}[t]
%     \caption{Create atlas graph approximation of high-contrast natural image patches}\label{alg:carlsson_atlas graph}
%     \begin{algorithmic}
%         \Require $\left(\theta_1,\phi_1\right),\ldots,\left(\theta_j,\phi_j\right)$ sampled uniformly from $[0,\pi]\times[0,2\pi]$
%         \Require $r>0$
%         \For{$j^\prime\gets1,\ldots,j$}
%             \State $\vec{x}_j\gets K_0\left(\theta_j,\phi_j\right)$
%         \EndFor
%         \For{$\theta\gets0,\pi/8,\ldots,7\pi/8$}
%             \For{$\phi\gets0,\pi/4,\ldots,7\pi/4$}
%                 \State $\vec{c}_{\theta,\phi}\gets K_0(\theta,\phi)$
%                 \State $L_{\theta,\phi},M_{\theta,\phi},K_{\theta,\phi}\gets$ quad. approx. of points within radius $r$ of $\vec{x}_{j^\prime}$ \Comment{Sec. \ref{sec:quad_approx_pt_cloud}}
%             \EndFor
%         \EndFor
%         \State\Return $\left\{\left(L_{\theta,\phi},M_{\theta,\phi},K_{\theta,\phi}\right)\right\}_{\theta,\phi}$
%     \end{algorithmic}
% \end{algorithm}

\section{Riemannian optimization over a manifold learned from a point cloud} \label{sec:learned_manifold_optimization}
To test the capability of an atlas graph learned from a point cloud to enable Riemannian optimization and machine learning, we implemented and applied the Riemannian principal boundary (RPB) algorithm~\citep{yao_2020}, which generalizes the SVM to Riemannian geometry, to a manifold without closed form. 
%The algorithm requires computation of several differential-geometric primitives over general manifolds.Here,  %demonstrating the accurate recovery of the Riemannian geometry of a manifold from observed point cloud data.
%; and (3) the implementation and application of Riemannian optimization and machine learning tasks on a manifold without closed form. 

\subsection{Riemannian principal boundary to classify image patches}\label{sec:rpb}

%The generalization of such a technique to the Riemannian domain is desirable in practice due to several attempts in the literature to generalize machine learning techniques to low-dimensional representations of point cloud data in a principled way.

The natural image patches have a natural principal boundary between convex and concave patches (Fig.~\ref{fig:pos_class}), which we aim to identify using RPB algorithmic approach. %
For a manifold $\mathcal{M}$ of dimension two, the RPB algorithm learns a binary classifier by first learning one-dimensional ``boundary'' submanifolds $\Gamma,\Gamma^\prime\subset\mathcal{M}$ for each of the two classes, and then ``interpolating'' between these two boundaries to create a one-dimensional separating submanifold. This interpolation is done by characterizing $\Gamma$ and $\Gamma^\prime$ as ODEs parameterized by the class label, the solutions to which are called \emph{principal flows}. At each iteration of the algorithm, the solution curve to each principle flow is approximated locally as a short geodesic update in the direction of the first derivative given by the ODE. The boundary curve is simultaneously updated by parallel-transporting the first derivatives from the two principle flows, taking their weighted average, and taking a small geodesic update in the direction of this average. The resulting boundary curve, which is referred to as the \textit{principal boundary}, serves as a binary classifier. 

Computationally, we apply Euler's method over 2,000 iterations to simultaneously integrate the principal flows for the convex and concave patches, as well as the principal boundary ODE (Fig.~\ref{fig:gradual_path}). Our implementation of the RPB algorithm also includes slight modifications (Appendix \ref{sec:exp_carlsson}), designed to avoid degenerate cases where the boundary curve estimate collapses into the principal flow or a principle flow moves too far from the data. 

Our experiments show that the algorithm can successfully learn a boundary curve between the convex and concave patches, solely in terms of the geometry encoded in our atlas graph representation (Fig.~\ref{fig:gradual_path}). The learned boundary curve intuitively captures what would result from parallel-transporting the first derivative of either principle flow to a point equidistant between both principle flows, and following the unique induced geodesic.
%The gradual development of $\Gamma$, $\Gamma^\prime$, and the RPB is visualized using polar coordinates in Fig.~\ref{fig:gradual_path}. Fig.~\ref{fig:actual_path} shows $\Gamma$, $\Gamma^\prime$, and the RPB using the Karcher visualization.
\begin{figure}[t]
    \centering
    \begin{overpic}[width=0.23\columnwidth,trim=0cm 0cm 1cm 0cm, clip]{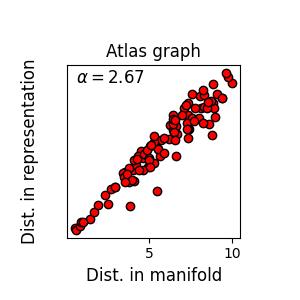}\put(5,95){\textbf{A}}\end{overpic}\begin{overpic}[width=0.23\columnwidth,trim=0cm 0cm 1cm 0cm, clip]{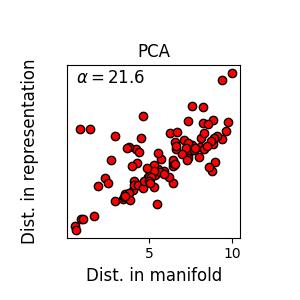}\put(5,95){\textbf{B}}\end{overpic}\begin{overpic}[width=0.23\columnwidth,trim=0cm 0cm 1cm 0cm, clip]{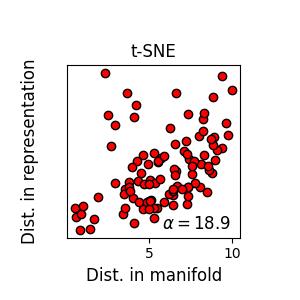}\put(5,95){\textbf{C}}\end{overpic}\begin{overpic}[width=0.23\columnwidth,trim=0cm 0cm 1cm 0cm, clip]{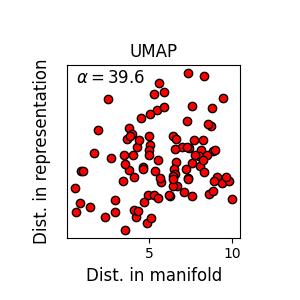}\put(5,95){\textbf{D}}\end{overpic}
    \caption{\textbf{Our atlas graph representation of the manifold of high-contrast image patches preserves pairwise geodesic distance better than other dimensionality reduction techniques.} \textbf{A.} For a random sample of points, the plot shows the manifold distance (x axis) between pairs of points versus the distance in the atlas graph representation (y axis), computed using na\"ive approximate distances (Sec.~\ref{sec:app_naive_dist}).     The metric distortion \citep[]{charikar_2018} is given by $\alpha$. \textbf{B--D}. Analogous plots for PCA, $t$-SNE, and UMAP representations (panel titles), with geodesic distances computed as in Appendix~\ref{sec:geodesic_distances}. 
    }
    \label{fig:dist_fig}
\end{figure}

\begin{figure}[t]
    \centering
    \begin{overpic}[width=0.23\columnwidth, trim={1cm, 0.5cm, 1cm, 2.5cm}]{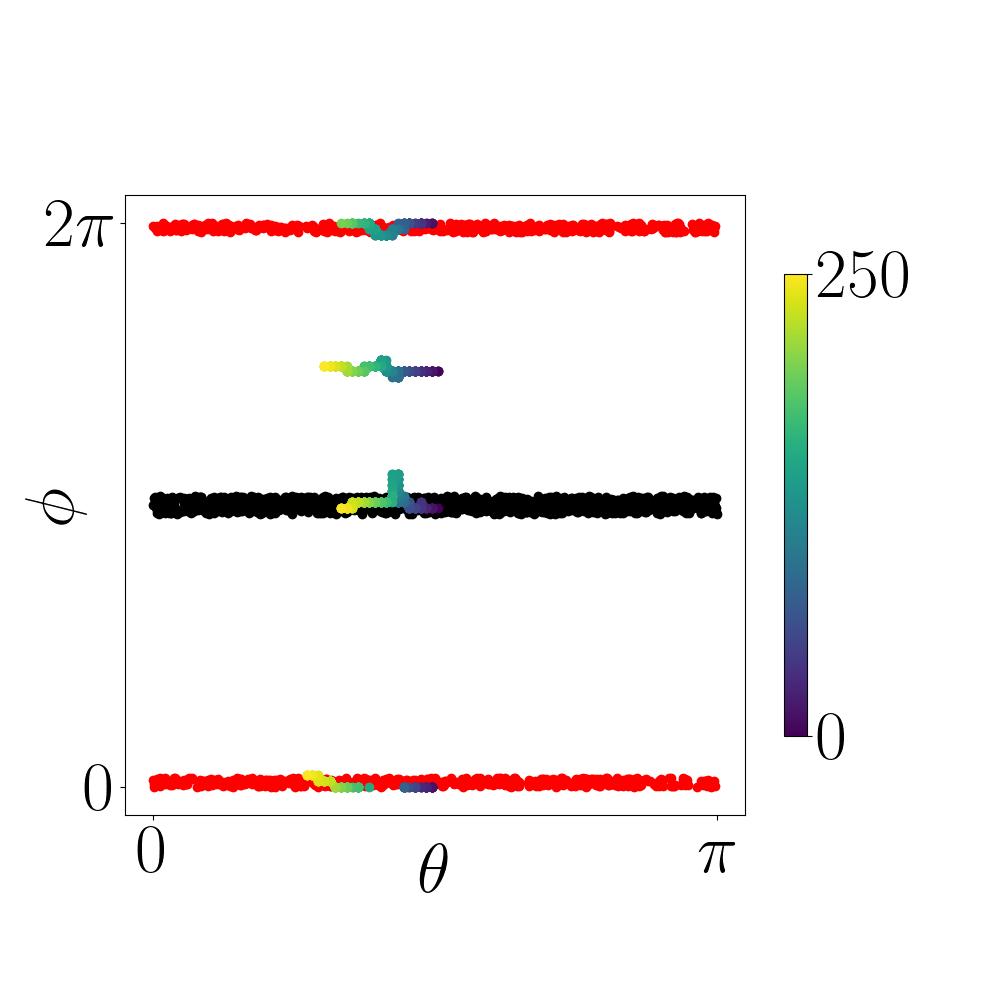}\put(-15, 0){\textbf{A}}\end{overpic}
    \includegraphics[width=0.23\columnwidth, trim={1cm, 0.5cm, 1cm, 2.5cm}]{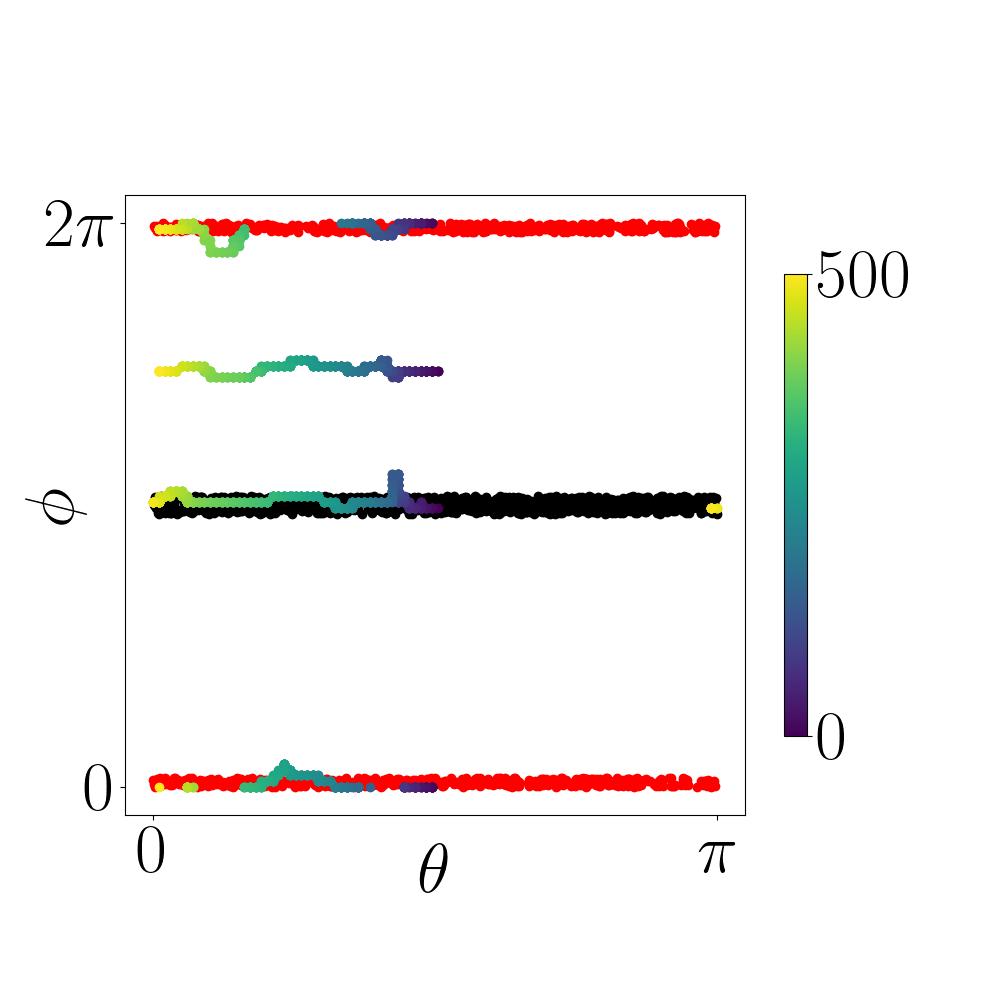}
    \includegraphics[width=0.23\columnwidth, trim={1cm, 0.5cm, 1cm, 2.5cm}]{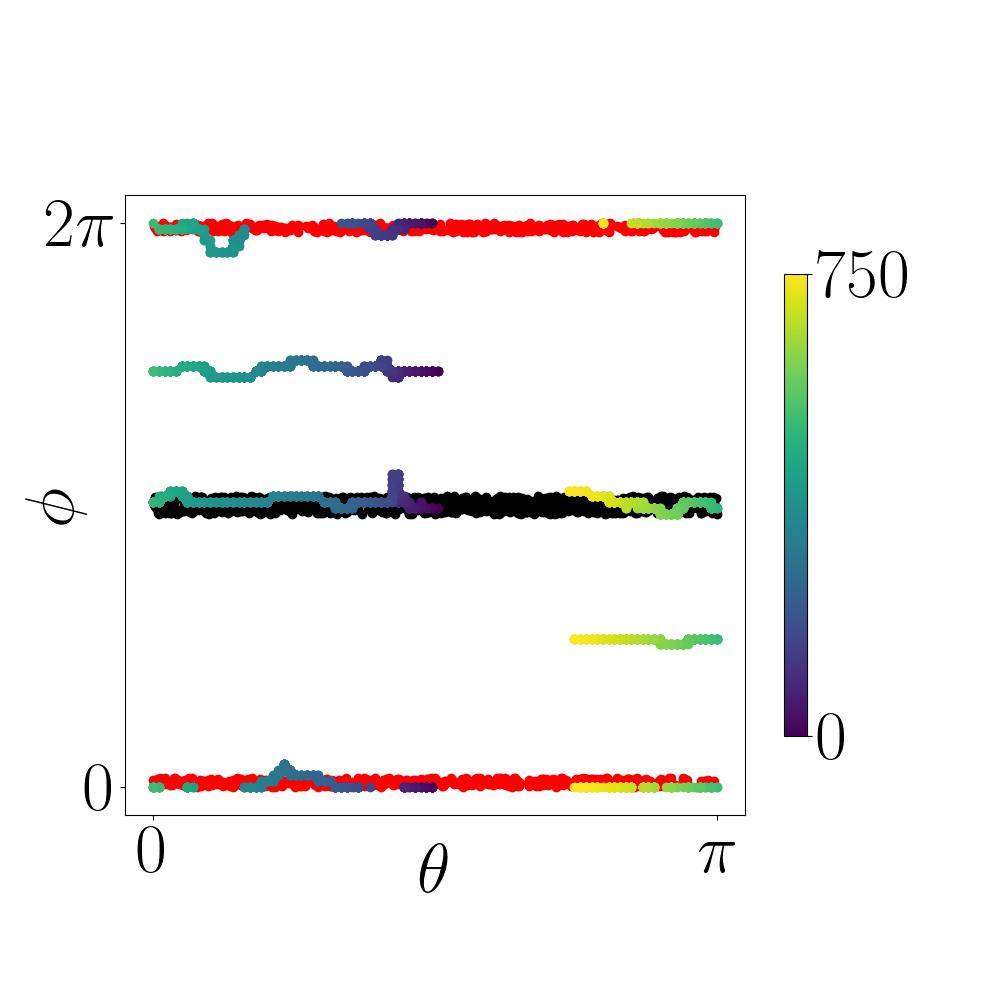}
    \includegraphics[width=0.23\columnwidth, trim={1cm, 0.5cm, 1cm, 2.5cm}]{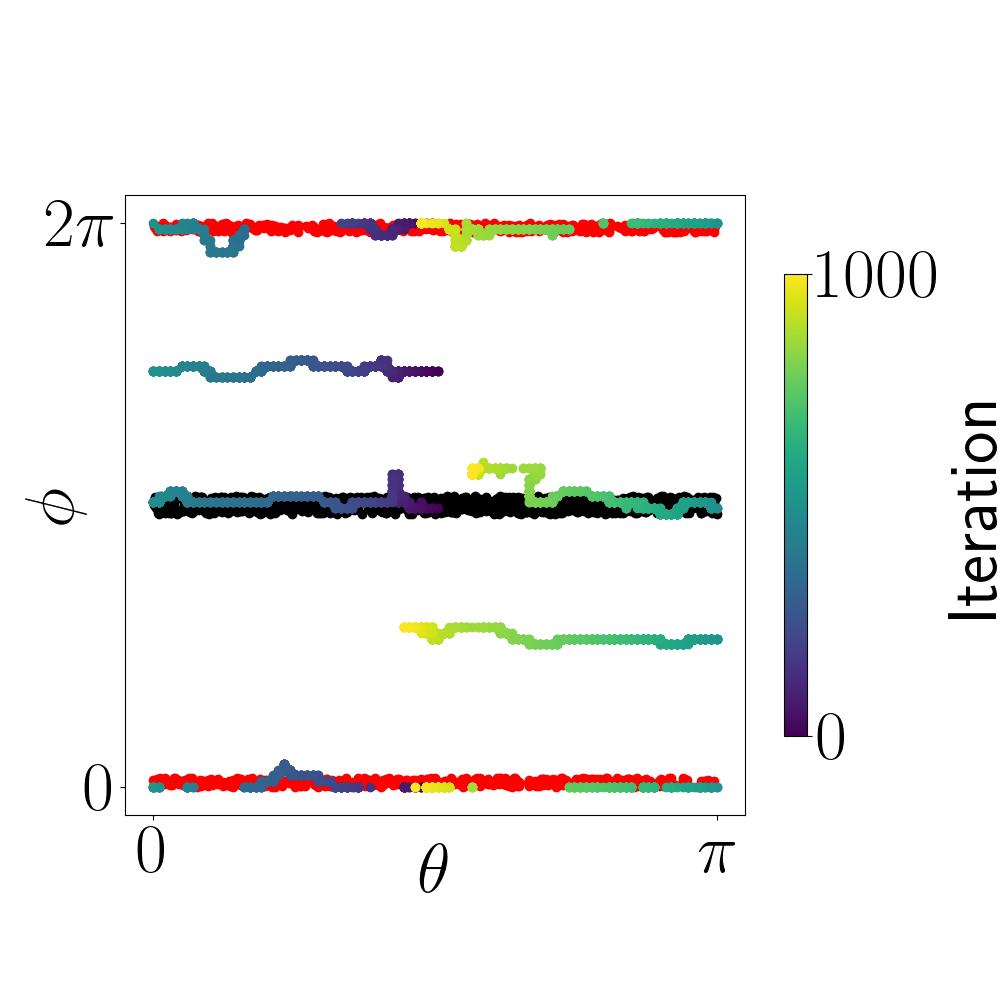} \\
    \includegraphics[width=0.23\columnwidth, trim={1cm, 0.5cm, 1cm, 2.5cm}]{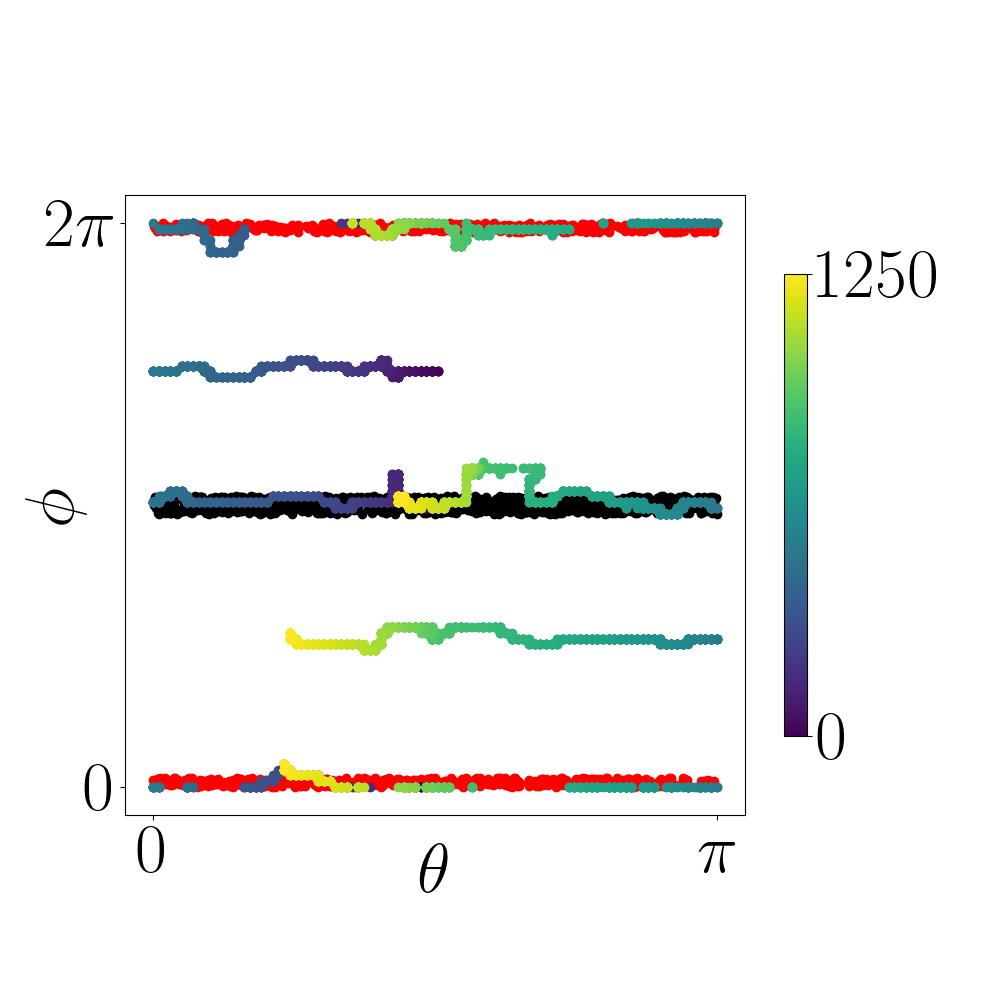}
    \includegraphics[width=0.23\columnwidth, trim={1cm, 0.5cm, 1cm, 2.5cm}]{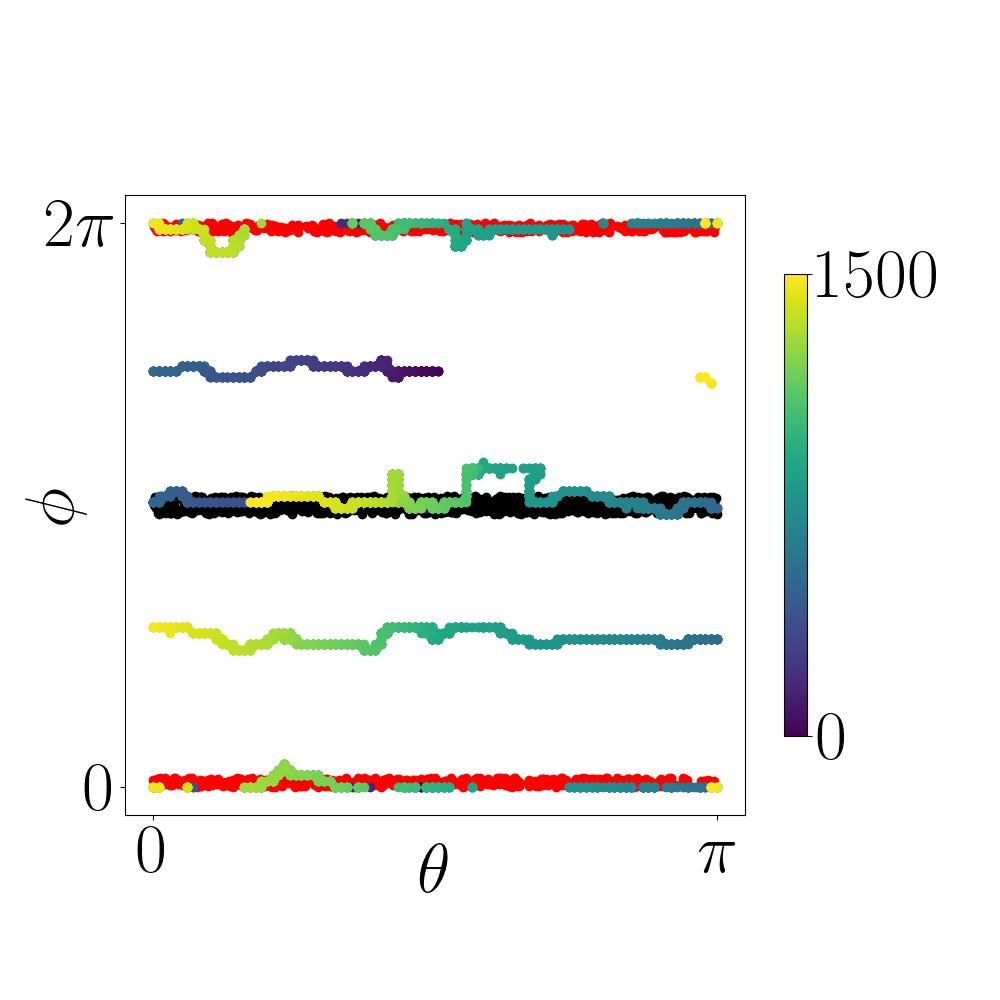}
    \includegraphics[width=0.23\columnwidth, trim={1cm, 0.5cm, 1cm, 2.5cm}]{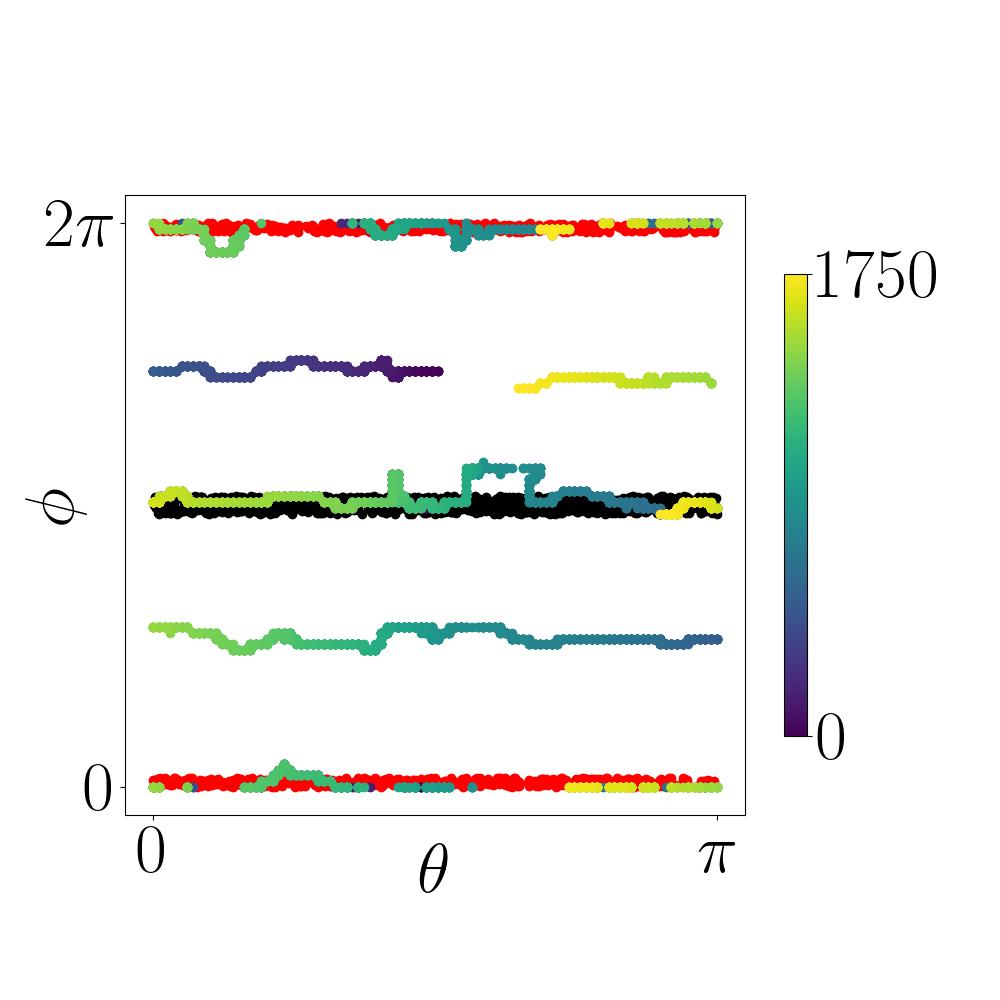}
    \includegraphics[width=0.23\columnwidth, trim={1cm, 0.5cm, 1cm, 2.5cm}]{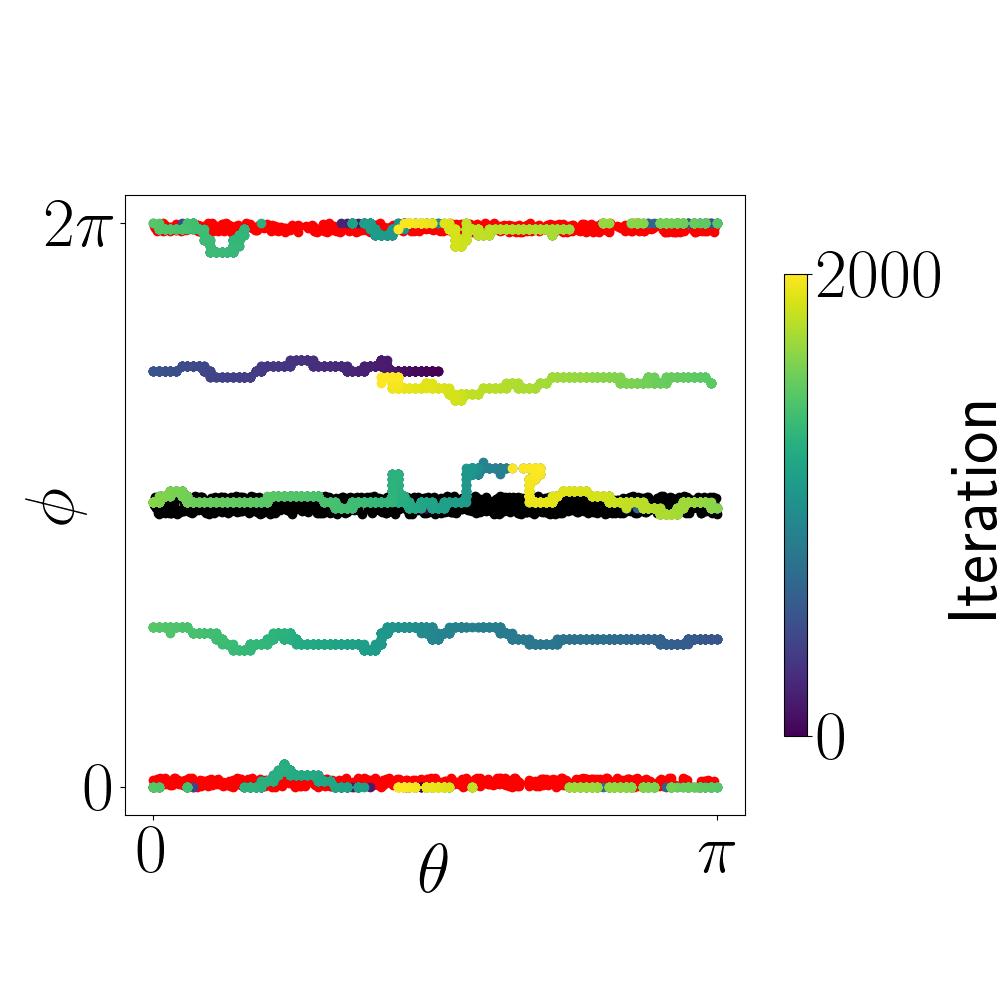}
    \begin{overpic}[width=0.3\columnwidth, trim={0.8cm, 1.8cm, 0.8cm, 1.8cm}]{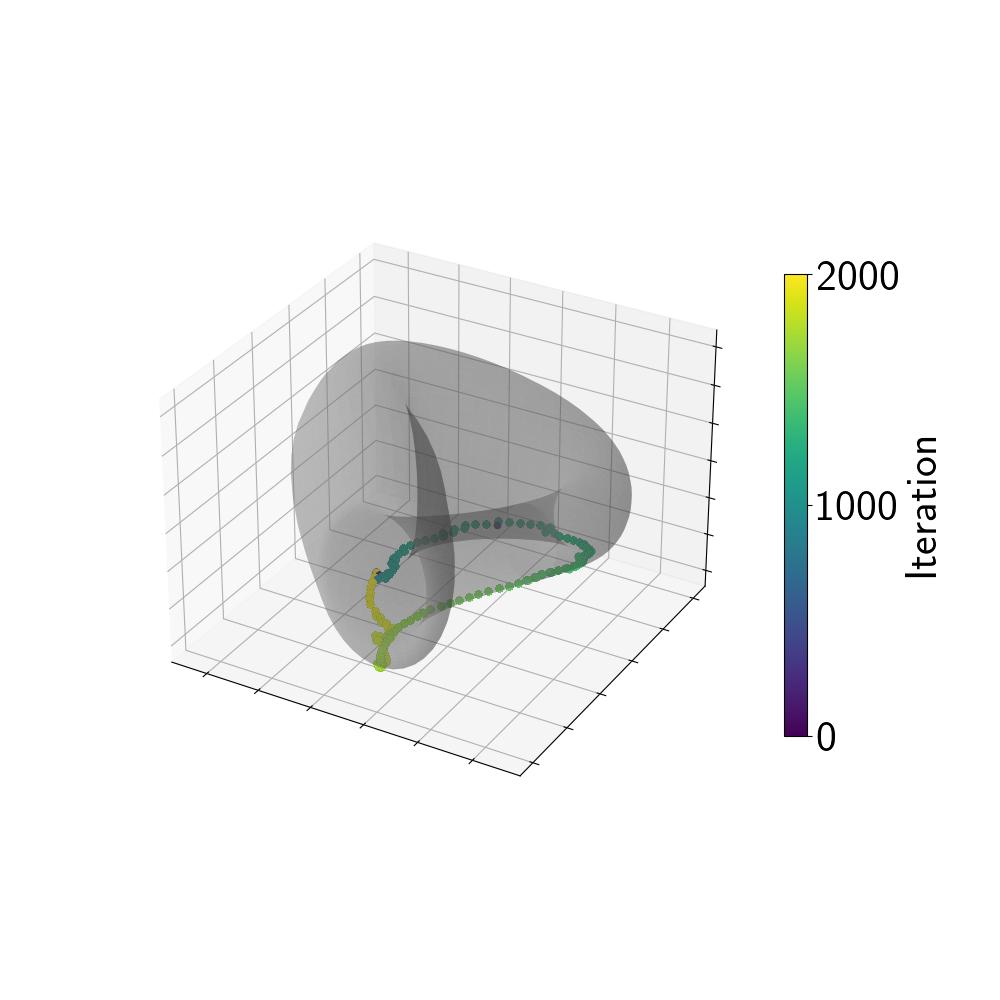}\put(5,75){\textbf{B}}\end{overpic}\begin{overpic}[width=0.3\columnwidth, trim={0.8cm, 1.8cm, 0.8cm, 1.8cm}]{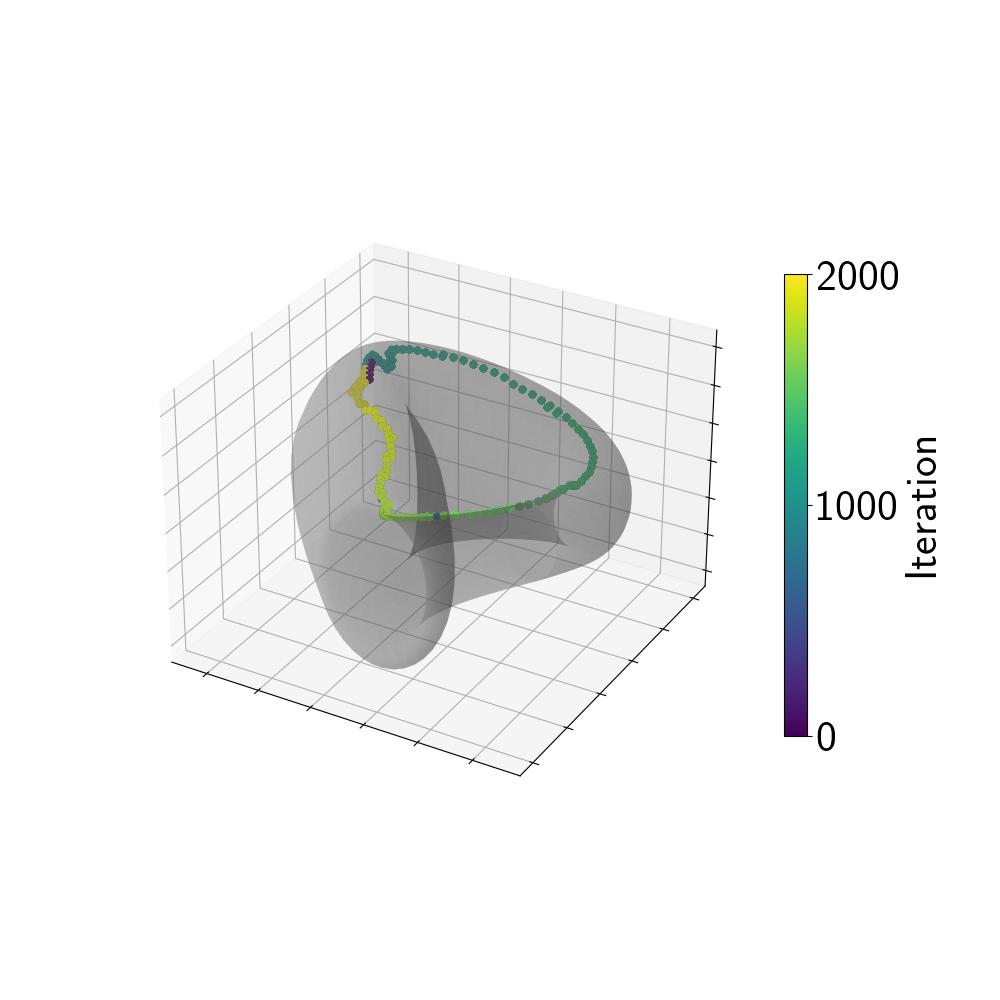}\put(5,75){\textbf{C}}\end{overpic}\begin{overpic}[width=0.3\columnwidth, trim={0.8cm, 1.8cm, 0.8cm, 1.8cm}]{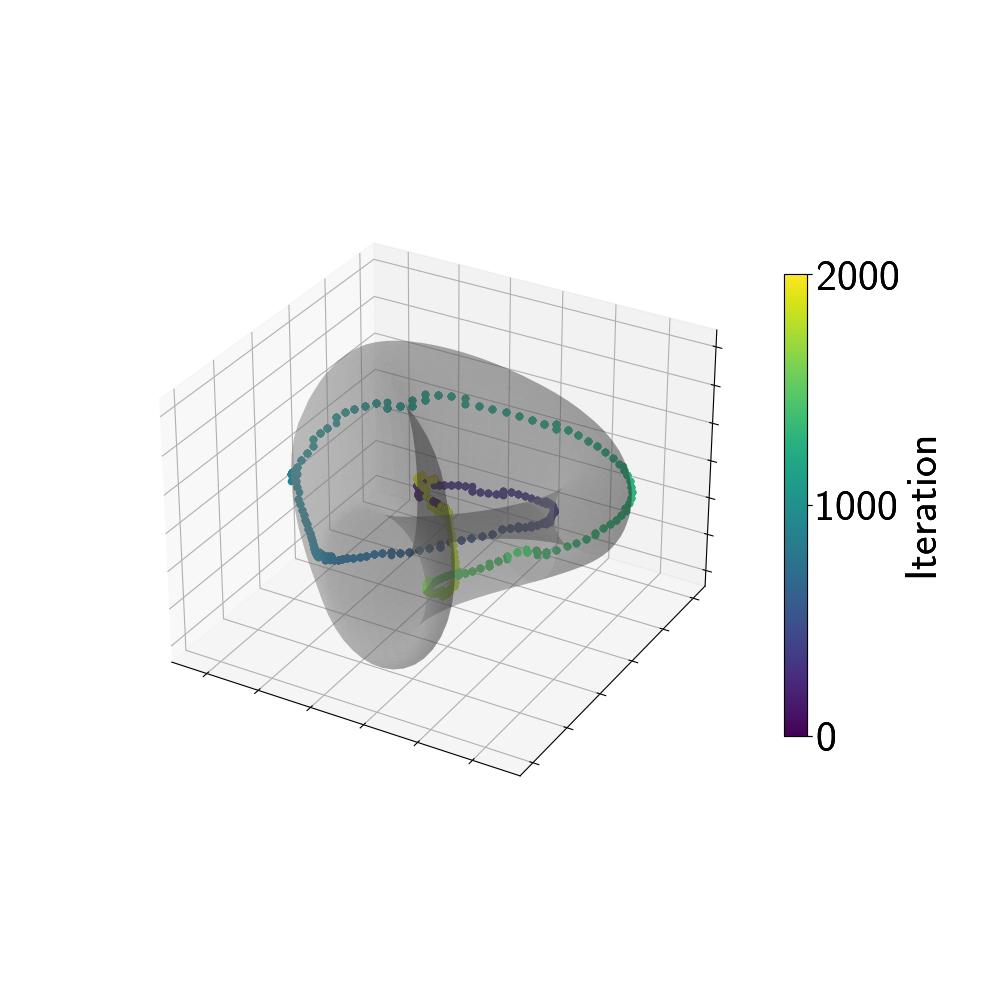}\put(5,75){\textbf{D}}\end{overpic}
    \caption{\textbf{The RPB algorithm applied to the atlas graph of the high-contrast image patches manifold infers representations of the convex and concave patch classes and learns a discriminating boundary.} \textbf{A.} Each panel is a polar parameterization of the manifold, with convex and concave sample points colored in red and black, respectively. Three paths are shown, corresponding to the simultaneous integrations of three ODEs: our modified versions of the Yao principal flow ODEs (Eq.~\ref{eqn:modified_yao}) for each of the convex and concave classes, and the RPB ODE. Points in each path are colored by the forward-Euler iteration at which they occur. The progression of the integration is shown in the sequence of panels from left to right, top to bottom, in multiples of 250 iterations. The final panel shows the complete solution, with the RPB curve acting as a separator between convex and concave patches. \textbf{B--D.} The complete paths for the convex patches (B), concave patches (C) and RPB algorithm (D), visualized on the Karcher representation of the Klein bottle.} 
    \label{fig:gradual_path}
\end{figure}

\section{Discussion}\label{sec:discussion}
Here we detail the advances introduced by this work, its limitations, and potential applications.  
\subsection{Advances} 
We present a universal manifold representation scheme, based on general principles of Riemannian manifolds. 
Our atlas graph scheme can be used to learn manifold representations from point cloud data, as well as to approximate differential-geometric primitives for Riemannian optimization and machine learning. 
We use the scheme to implement the RPB algorithm in the context of Carlsson's high-contrast image patches and show that it successfully estimates the intuitive principal boundary between the convex and concave classes of patches (Fig.~\ref{fig:gradual_path}). This represents the first application of a Riemannian optimization routine on a manifold with nontrivial differential-geometric primitives, using a representation learned from point cloud data. 

Given that our implementation depends upon parallel transports, path lengths, and 
Riemannian logarithms, all of which depend upon the particularities of the underlying Riemannian manifold, the successful recovery of the boundary suggests that the learned atlas graph closely approximates the actual Riemannian geometry. Indeed, all differential-geometric primitives in the RPB implementation are computed in terms of local quadratic approximations, with the exceptions of the Riemannian logarithms $\mathbf{Log}_xy$ and vector transports $\tau_{x\to y}\tv$ for distant $x$ and $y$. The latter computations rely upon finding shortest paths in a dense subgraph, which is itself constructed solely in terms of the atlas graph.
Moreover, using random samples from the atlas graph, we correctly recover the homology groups of, as well as geodesic distances in, the Klein bottle approximation of the manifold (Figs.~\ref{fig:h2_homology} and \ref{fig:dist_fig}) \citep{carlsson}. 

%While Carlsson’s high-contrast natural image patches are a good test case for determining if atlas graphs can learn a nontrivial Riemannian geometry solely from samples, it is reciprocally the case that Our procedure for learning an atlas graph representation of Carlsson’s patches depends on manifold-specific foreknowledge of good centers for coordinate charts~(Sec.~\ref{sec:atlas_graph_for_carlsson}).

\subsection{Limitations} For the atlas graph approach to be useful in exploratory data analysis, there must be a procedure for learning atlas graphs from point clouds without foreknowledge of ideal chart centers. Most attributes of the atlas graph, such as the dimensionality, can be learned from point cloud data using pre-existing techniques (e.g., mSVD \citep{little}). %General applications of atlas graphs to data representation problems can leverage multiscale SVD \citep{little} to learn the intrinsic dimensionality.
However, centers for atlas graph coordinate charts in particular are not easy to select. %it is not clear to the authors how to approach the problem of selecting coordinate chart centers for general manifolds.
In the case of our experiments on Carlsson's high-contrast patches: because the manifold is compact, we are able to infer good candidates for coordinate chart centers by performing $k$-medoids on a large number of patches. However, this approach is not generally applicable; it breaks down in the case that the manifold is noncompact or is sampled according to a measure other than the uniform measure of its closure. Further work is needed to select chart centers for arbitrary manifolds in a principled way, especially so as to appropriately balance number of charts with error incurred from quasi-Euclidean updates as a result of curvature. 
%For the Carlsson high-contrast patch experiments, our implementation of the RPB algorithm leveraged foreknowledge%of good centers for coordinate charts and their
%of the intrinsic dimensionality of the manifold.%, to learn chart transitions on the fly (Sec.~\ref{sec:atlas_graph_for_carlsson}). 
%Future, general applications of atlas graphs to data representation problems can leverage multiscale SVD \citep{little} to learn the intrinsic dimensionality but still require the development of more general algorithms to create atlas graph representations from point clouds. %Further, even though we have foreknowledge of the intrinsic dimensionality of the Klein bottle representation of Carlsson’s patches in our experiments,

\subsection{Applications} The atlas graph approach has potential applications in diverse areas of research. A large volume of work seeks to learn low-dimensional representations of transcriptomic space through dimensionality reduction on scRNA-seq data. Atlas graphs can accommodate state-of-the-art analytical methods for transcriptomic data---such as RNA velocity and trajectory inference---that have a natural interpretation as integrating a vector field on the surface of a manifold \citep{frank_rec}. In the field of computer vision, atlas graphs may help simplify manifold-layer computations in manifold neural networks and manifold autoencoding schemes, which have recently gained serious attention \citep{chakraborty2020manifoldnet}. Most generally, atlas graphs make machine learning on manifolds more accessible, rendering pre-existing algorithms like the RPB algorithm applicable and inviting work on new algorithms specifically utilizing atlas graphs.
%\section{Conclusion}\label{sec:conclusion}

One future development that would make atlas graphs even more broadly applicable is  an algorithm that can learn atlas graph representations of point cloud data with some notion of statistical guarantees. %While our results on the high-contrast image patches are considerable, selecting coordinate charts for our Carlsson patches experiment relies upon tacit knowledge regarding the structure of the Klein bottle. While manifold structures arise all the time in exploratory data analysis, it is not generally the case that there is foreknowledge as to what center points and ambient radii should be used for coordinate charts. 
Another interesting avenue is creating atlas graph representations of manifolds whose structure is known, but for which there do not yet exist useful numerical representation schemes. Among these are the Calabi-Yau manifolds, which arise in particle physics, and sparse and structured submanifolds of the special unitary group, which arise in parameterized quantum circuits \citep{vermeersch_2023,benedetti_2019}.

\section*{Acknowledgments}The authors are grateful to Isabella DeClue and Kyle Ruark for helping implement mSVD and other software tools during early project development. The authors also thank Shmuel Weinberger for his help in refining the manuscript.
%
%The authors are grateful to Ari and Elia Orecchia for their profound insight and patience.

\bibliographystyle{siamplain}
\bibliography{references}

\appendix
\input{ex_supplement.tex}

\end{document}

%% file: ex_shared.tex
\usepackage{lipsum}
\usepackage{amssymb}
\usepackage{amsmath}
\usepackage{amsfonts}
\usepackage{graphicx}
\usepackage{epstopdf}
\usepackage{algorithm}
\usepackage[noend]{algpseudocode}
\ifpdf
  \DeclareGraphicsExtensions{.eps,.pdf,.png,.jpg}
\else
  \DeclareGraphicsExtensions{.eps}
\fi
\usepackage{mathtools}
\usepackage{overpic}
\usepackage{graphbox}
\usepackage{caption}
\usepackage{subcaption}
\usepackage[numbers]{natbib}
\newcommand{\atlasfrechet}{\texttt{ATLAS\_Frechet}}
\newcommand{\tv}{\ensuremath{\vec{\xi}}}
\newcommand{\nv}{\ensuremath{\vec{\nu}}}
\DeclareMathOperator*{\arccot}{arccot}
\makeatletter
\NewDocumentCommand{\LeftComment}{s m}{%
  \Statex \IfBooleanF{#1}{\hspace*{\ALG@thistlm}}\(\triangleright\) #2}
\makeatother
\makeatletter
\newcommand*{\addFileDependency}[1]{% argument=file name and extension
\typeout{(#1)}% latexmk will find this if $recorder=0
\@addtofilelist{#1}
\IfFileExists{#1}{}{\typeout{No file #1.}}
}\makeatother

\usepackage{enumitem}
\setlist[enumerate]{leftmargin=.5in}
\setlist[itemize]{leftmargin=.5in}

\newsiamremark{remark}{Remark}
\newsiamremark{hypothesis}{Hypothesis}
\crefname{hypothesis}{Hypothesis}{Hypotheses}
\newsiamthm{claim}{Claim}

\headers{Manifold learning and optimization using tangent space proxies}{R. A. Robinett, L. Orecchia, and S. J. Riesenfeld}

\title{{Manifold learning and optimization using tangent space proxies}
\thanks{Submitted to the editors Sept.~10, 2024. The last two authors contributed equally.
\funding{L.O. is partially supported by NSF CAREER award 1943510. S.J.R. and R.A.R. are supported in part by the NSF-Simons National Institute for Theory and Mathematics in Biology, jointly funded by NSF award 2235451 and Simons Foundation award MP-TMPS-00005320. S.J.R. is a CZ Biohub Investigator. R.A.R. was supported in part by an NSF Graduate Research Fellowship (214001).}}
}

\author{Ryan A. Robinett\thanks{Department of Computer Science, University of Chicago, Chicago, IL 
  (\email{robinett@uchicago.edu}).}
\and Lorenzo Orecchia\thanks{Department of Computer Science, University of Chicago, Chicago, IL (\email{orecchia@uchicago.edu})}
\and Samantha J. Riesenfeld\thanks{Pritzker School of Molecular Engineering, University of Chicago, Chicago, IL; Department of Medicine, University of Chicago, Chicago, IL; Committee on Immunology, University of Chicago, Chicago, IL; Institute for Biophysical Dynamics, University of Chicago, Chicago, IL; CZ Biohub Chicago, LLC, Chicago, Illinois 60642; NSF-Simons National Institute for Theory and Mathematics in Biology, Chicago, IL 60611(\email{sriesenfeld@uchicago.edu}).}}

%% file: ex_supplement.tex
% % SIAM Supplemental File Template
% \documentclass[review,supplement,onefignum,onetabnum]{siamonline220329}

% \input{ex_shared}

% % \externaldocument[][nocite]{ex_article}
% \myexternaldocument{ex_article}
% \allowdisplaybreaks

% % Optional PDF information
% \ifpdf
% \hypersetup{
%   pdftitle={Supplementary Materials: An Example Article},
%   pdfauthor={D. Doe, P. T. Frank, and J. E. Smith}
% }
% \fi

% \begin{document}

% \maketitle

\section{Accuracy of quasi-Euclidean updates}

\setcounter{figure}{0} %\renewcommand{\figurename}{Fig.}
\renewcommand{\thefigure}{S\arabic{figure}}

\subsection{Quasi-Euclidean updates are local retractions}\label{app:quasi_euclidean}
%\begin{claimproof}
\begin{proof} (of Claim~\ref{clm:quasi_euclidean_approx}).  \label{prf:quasi_euclidean}
    %Quasi-Euclidean updates near the origin approximate the exponential map up to order $O\left(\left\lVert\vec{\tau}\right\rVert^2_{\overline{\mathfrak{g}}}\right)$. Further, restriction of quasi-Euclidean updates to a single coordinate chart $\left(\mathcal{U},\mathcal{V},\varphi\right)$ comprises a retraction.
        A Euclidean tangent vector $\vec{\tau}\in\mathbb{R}^d$ at $\vec{\xi}\in\mathcal{V}$ has unique representative tangent vector $\left[D\varphi^{-1}\right]_{\vec{\xi}}\left(\vec{\tau}\right)$ through precomposition by $\varphi^{-1}$, where $\left[D\varphi^{-1}\right]_{\vec{\xi}}$ denotes the differential of $\varphi^{-1}$ at $\vec{\xi}$. Under a quasi-Euclidean update, Euclidean tangent vector $\vec{\tau}$ takes $p$ to $\varphi^{-1}\left(\varphi(p)+\vec{\tau}\right)$. To show that $\varphi^{-1}\left(\varphi(p)+\vec{\tau}\right)$ approximates $\mathbf{Exp}_p\left(\left[D\varphi^{-1}\right]_{\varphi(p)}\left(\vec{\tau}\right)\right)$ (i.e., the action of the exponential map $\mathbf{Exp}_p$ at $p$ upon $\left[D\varphi^{-1}\right]_{\varphi(p)}\left(\vec{\tau}\right)$) up to order $O\left(\left\lVert\vec{\tau}\right\rVert^2_{\overline{\mathfrak{g}}}\right)$, it suffices to show that $\left(\varphi\circ\mathbf{Exp}_p\right)\left(\left[D\varphi^{-1}\right]_{\varphi(p)}\left(\vec{\tau}\right)\right)$ approximates $\varphi(p)+\vec{\tau}$. Because the map
        \begin{equation*}
            f:\vec{\tau}\mapsto\left(\varphi\circ\mathbf{Exp}_p\right)\left(\left[D\varphi^{-1}\right]_{\varphi(p)}\left(\vec{\tau}\right)\right)
        \end{equation*}
        is a map from $\mathbb{R}^d$ to $\mathbb{R}^d$, this can be done through Taylor approximation about $\vec{\tau}=\vec{0}$.

        All that remains is to compute the constant and linear terms of the Taylor expansion. It is straightforward to compute that $f\left(\vec{0}\right)=\varphi(p)$, and the linear term is given by the following:
        \begin{align*}D\left[\varphi\circ\mathbf{Exp}_{p}\right]_{\varphi(p)}\left(\left[D\varphi^{-1}\right]_{\varphi(p)}\left(\vec{\tau}\right)\right)&=\left[D\varphi_{\mathbf{Exp}_p\left(\vec{0}\right)}\cdot D\mathbf{Exp}_{(p)\left(\varphi(p)\right)}\right]\left(\left[D\varphi^{-1}\right]_{\varphi(p)}\left(\vec{\tau}\right)\right) \\
            &=\left[D\varphi_p\cdot\mathbf{Id}\right]\left(\left[D\varphi^{-1}\right]_{\varphi(p)}\left(\vec{\tau}\right)\right) \\
            &=\left[D\varphi\right]_p\left(\left[D\varphi^{-1}\right]_{\varphi(p)}\left(\vec{\tau}\right)\right) \\
            &=\vec{\tau}
        \end{align*}
        The coefficients of the Taylor approximation demonstrate that restrictions of quasi-Euclidean updates to a single coordinate chart comprise a retraction according to the conditions of Definition 19.2 in Hosseini and Sra~\citep{hosseini_and_sra}.
    \end{proof}
%\end{claimproof}

\subsection{Second order error of quasi-Euclidean updates}\label{app:quasi_instability}

In Claim~\ref{clm:quasi_euclidean_approx}, we show that
\begin{equation*}\left(\varphi\circ\mathbf{Exp}_p\right)\left(\left[D\varphi^{-1}\right]_{\varphi(p)}\left(\vec{\tau}\right)\right)=\varphi(p)+\vec{\tau}+O\left(\left\lVert\vec{\tau}\right\rVert^2_{\overline{\mathfrak{g}}}\right)
\end{equation*}
for all coordinate charts $(\mathcal{U},\mathcal{V},\varphi)$. Here, we discuss the dependence of the error term $O\left(\left\lVert\vec{\tau}\right\rVert^2_{\overline{\mathfrak{g}}}\right)$ on the choice of coordinate chart. % or general $\varphi$. 
In shorthand $\vec{\xi}(1):=\left(\varphi\circ\mathbf{Exp}_p\right)\left(\left[D\varphi^{-1}\right]_{\varphi(p)}\left(\vec{\tau}\right)\right)$, and the geodesic equation determining $\vec{\xi}(1)$ can be written in Einstein notation as
\begin{equation*}\ddot{\xi}^\lambda+\Gamma^\lambda_{\mu\nu}\dot{\xi}^\mu\dot{\xi}^\nu=0,
\end{equation*}
where $\dot{\xi}^\lambda$ and $\ddot{\xi}^\lambda$ denote first and second derivatives of $\xi^\lambda$, respectively, and $\Gamma^\lambda_{\mu\nu}$ is the Christoffel symbol of the second kind \citep[e.g.,][]{absil}.
% , given by
% $$\Gamma^\lambda_{\mu\nu}=\frac{1}{2}\overline{g}^{\lambda\rho}\left(\frac{d\overline{g}^\rho\mu}{d\xi^\nu}+\frac{d\overline{g}^{\rho\nu}}{d\xi^\mu}-\frac{d\overline{g}^{\mu\nu}}{d\xi^\rho}\right),$$
% where $\overline{g}^{\alpha\beta}$ is the metric tensor of the Riemannian metric $\overline{\mathfrak{g}}$ from Sec.~\cite[Sec.~\ref{sec:atlas_graph_met}]{bishop_goldberg_1980}.
The ODE has initial conditions $\vec{\xi}(0)=\varphi(p)$ and $\dot{\vec{\xi}}(0)=\vec{\tau}$. Therefore,
% \begin{equation}
%     \dot{\xi}^\lambda(t)=\int_0^t\Gamma^\lambda_{\mu\nu}\dot{\xi}^\mu\dot{\xi}^\nu dt+\vec{\tau}^\lambda
% \end{equation}
\begin{equation}\label{eqn:quasi_difference_full}
    \xi^\lambda(1)=\varphi(p)+\vec{\tau}^\lambda+\int_0^1(1-t)\Gamma^\lambda_{\mu\nu}\dot{\xi}^\mu\dot{\xi}^\nu dt
\end{equation}
% \begin{equation}\label{eqn:quasi_difference_full}\xi^\lambda(1)=\varphi(p)^\lambda+\tau^\lambda-\int_0^1\Gamma^\lambda_{\mu\nu}\xi^\mu\xi^\nu dt.
% \end{equation}

Equation~\ref{eqn:quasi_difference_full} demonstrates that little can claimed immediately about the $O\left(\left\lVert\vec{\tau}\right\rVert^2_{\overline{\mathfrak{g}}}\right)$ term; it could be zero, or it could grow arbitrarily fast with respect to $\left\lVert\vec{\tau}\right\rVert^2_{\overline{\mathfrak{g}}}$. Christoffel symbols do not transform tensorially; therefore, they can either vanish or grow arbitrarily large depending on choice of coordinate chart. If the coordinate chart $\left(\mathcal{U},\mathcal{V},\varphi\right)$ comprises a Riemannian normal coordinate system centered at $p\in\mathcal{U}$, then the Christoffel symbol $\Gamma^\lambda_{\mu\nu}$ vanishes for all representative tangent vectors $\vec{\tau}$, and the quasi-Euclidean update reproduces the exponential map exactly \citep{bishop_goldberg_1980}. 
However, even in the case of Riemannian normal coordinates, the Christoffel symbol does not universally vanish when the Riemannian normal coordinates are centered at $q\neq p$, as long as the manifold has nonzero curvature. Alternatively: say the Christoffel symbol is nontrivial at $\varphi(p)\in\mathcal{V}$. For any positive $c$, we can replace $\varphi$ with $c\varphi$, and the Christoffel symbol causes the quadratic coefficient in the Taylor approximation to be arbitrarily large or small. While we do not explore the growth of the $O\left(\left\lVert\vec{\tau}\right\rVert^2_{\overline{\mathfrak{g}}}\right)$ term in practice, our intuition is that this error term is best mitigated by switching coordinate charts more often, especially in regions of higher curvature.

% \samnote{I commented out the rest of the paragraph. Would instead add a note here saying that one can limit the effect of curvature by using more charts, but we do not answer here what relationship defines this trade-off in theory or in practice.}
%Alternatively: say the Christoffel symbol is nontrivial at $\varphi(p)\in\mathcal{V}$. For any positive $c$, we can replace $\varphi$ with $c\varphi$, and the Christoffel symbol causes the quadratic coefficient in the Taylor approximation to be arbitrarily large or small. Therefore, unlike the zeroth- and first-order terms, very little can be said about the quadratic term of the Taylor approximation for general coordinate charts $(\mathcal{U},\mathcal{V},\varphi)$, beyond that there will always exist points $p\in\mathcal{U}$ for which the Christoffel symbols do not vanish due to curvature. What can be said is that unless the choice of coordinate charts comprises Riemannian normal coordinates---which would require knowledge of geodesics---then the quadratic term of the Taylor expansion would be nonzero.

\section{Approximation error of vector transport on atlas graphs}\label{app:vec_trans}
Here, we look at the error term of the identity vector transport from Sec.~\ref{sec:vector_transport_def} as an approximation of the parallel transport of a tangent vector along a quasi-Euclidean update. Let $\left(\mathcal{U},\mathcal{V},\varphi\right)$ be a coordinate chart of Riemannian manifold $\mathcal{M}$ with points $p,q\in\mathcal{U}$ satisfying $p\neq q$. Further, let the parallel transport $\mathcal{P}_{p\to q}^{\mathbf{Ret}_p}$ and vector transport $\mathcal{T}_{p\to q}$ be as defined in Sec.~\ref{sec:vector_transport_def}, let $\tau:=\varphi(q)-\varphi(p)$, and let there be a tangent vectors $\vec{w}\in T_p\mathcal{M}$ with representative tangent vector $\vec{\sigma}_0:=\left[D\varphi\right]_p\left(\vec{w}\right)$. Lastly, let $\Gamma^\lambda_{\mu\nu}$ be the Christoffel symbol of the second kind. The parallel transport of vector field $\sigma$ along path the quasi-Euclidean update from $\varphi(p)$ to $\varphi(q)$ is given \citep[e.g.,][]{absil} by
\begin{equation*}
    \tau^j\partial_j\left(\sigma^l\right)\partial_l+\tau^j\sigma^k\Gamma^l_{jk}\partial_l=0.
\end{equation*}
Recognizing that the time derivative of the $j$th component of $\sigma$, which we denote by $\dot{\sigma}^j$, is equal to $\tau^l\partial_l\left(\sigma^j\right)$, we see that the $j$th component of $\mathcal{P}_{p\to q}^{\mathbf{Ret}_p}\vec{w}$ is given by
\begin{equation}\label{eqn:component_of_par}
    \left[\mathcal{P}_{p\to q}^{\mathbf{Ret}_p}\vec{w}\right]^j=\left(\left[D\varphi^{-1}\right]_{\varphi(q)}\right)^j_k\sigma_0^k-\left(\left[D\varphi^{-1}\right]_{\varphi(q)}\right)^j_k\int_0^1\sigma^l\Gamma^k_{ml}\tau^mdt.
\end{equation}
Therefore, the following series of deductions ends with a convenient form for the $j$th component of the difference between $\mathcal{P}_{p\to q}^{\mathbf{Ret}_p}\vec{w}$ and $\mathcal{T}_{p\to q}\vec{w}$.
\begin{align*}
    \left[\mathcal{T}_{p\to q}\vec{w}-\mathcal{P}_{p\to q}^{\mathbf{Ret}_p}\vec{w}\right]^j&=\left(\left[D\varphi^{-1}\right]_{\varphi(q)}\right)^j_k\int_0^1\sigma^l\Gamma^k_{ml}\tau^mdt \\
    &=\left(\left[D\varphi^{-1}\right]_{\varphi(q)}\right)^j_k\tau^m\int_0^1\sigma^l\Gamma^k_{ml}dt \\
    &=\left(\left[D\varphi^{-1}\right]_{\varphi(q)}\right)^j_k\left(\tau^m\sigma_0^l\int_0^1\Gamma^k_{ml}dt+\tau^m\int_0^1\left(\sigma^l-\sigma_0^l\right)\Gamma^k_{ml}dt\right)
\end{align*}
The $\left[D\varphi^{-1}\right]$ term is absorbed according to Equation~\ref{eqn:g_bar} when computing the Riemannian metric, and the term $\tau^m\sigma_0^l\int_0^1\Gamma^k_{ml}dt$ is of order\footnote{For better understanding of the contribution of the Christoffel symbol to the coefficient of this term, see Sec.~\ref{app:quasi_instability}.} $O\left(\left\lVert\vec{\sigma}_0\right\rVert_{\overline{\mathfrak{g}}}\left\lVert\vec{\tau}\right\rVert_{\overline{\mathfrak{g}}}\right)$. To see that $\mathcal{T}_{p\to q}\vec{w}$ approximates $\mathcal{P}_{p\to q}^{\mathbf{Ret}_p}\vec{w}$ up to order $O\left(\left\lVert\vec{\sigma}_0\right\rVert_{\overline{\mathfrak{g}}}\left\lVert\vec{\tau}\right\rVert^2_{\overline{\mathfrak{g}}}\right)$, it remains to show that this is the order of the term $\tau^m\int_0^1\left(\sigma^l-\sigma_0^l\right)\Gamma^k_{ml}dt$. This is accomplished in the following series of deductions, which make use of Lemma~\ref{lem:sig_tau_induction} and the Taylor expansion of the integrand.

\begin{align*}
    \tau^m\int_0^1\left(\sigma^l-\sigma_0^l\right)\Gamma^k_{ml}dt&=\tau^m\int_0^1\left(t\dot{\sigma}^l_0+\sum_{d=2}^\infty\frac{t^d}{d!}\left(\sigma^{(d)}_0\right)^l\right)\Gamma^k_{ml}dt \\
    &=\int_0^1\left(tO\left(\left\lVert\vec{\sigma}_0\right\rVert_{\overline{\mathfrak{g}}}\left\lVert\vec{\tau}\right\rVert^2_{\overline{\mathfrak{g}}}\right)+\sum_{d=2}^\infty \frac{t^d}{d!}O\left(\left\lVert\vec{\sigma}_0\right\rVert_{\overline{\mathfrak{g}}}\left\lVert\vec{\tau}\right\rVert^{d+1}_{\overline{\mathfrak{g}}}\right)\right)dt \\
    &=O\left(\left\lVert\vec{\sigma}_0\right\rVert_{\overline{\mathfrak{g}}}\left\lVert\vec{\tau}\right\rVert^2_{\overline{\mathfrak{g}}}\right)
\end{align*}

\begin{lemma}\label{lem:sig_tau_induction}
    Let $\vec{\sigma}_0^{(d)}$ be the $d$th derivative of $\vec{\sigma}$ with respect to time, evaluated at time $t=0$. For all $d\in\mathbb{N}_{>0}$, it holds that $\vec{\sigma}_0^{(d)}$ is linear in $\sigma_0$ and homogeneously of order $d$ in both $\vec{\tau}$ and the Christoffel symbols of the second kind.
    \begin{proof}
        This is easy to prove by induction, with the base case $d=1$ being given by the parallel transport equation.
    \end{proof}
\end{lemma}

\section{Online Fr\'echet mean estimation on the Grassmann manifold (Sec.~\ref{sec:grass_example}-\ref{sec:grass_experiment})}\label{sec:frechet_mean_estimation_methods}
\begin{algorithm}[t]
    \caption{\texttt{ATLAS\_transition\_map} (transition map on Grassmann atlas graph, used in Alg.~\ref{alg:grass_quasi_euclidean})}\label{alg:grass_trans_map}
    \begin{algorithmic}
        \Require $A\in\mathbb{R}^{(n-k)\times k}$, permutation indices $i_1,\ldots,i_k$
        \Require $Q_A$ \Comment{output by \texttt{QA\_from\_A}$(A)$, Alg.~\ref{alg:qa_from_a}}
        \State $P\gets\texttt{permutation\_matrix}\left(i_1,\ldots,i_k\right)$ \Comment{$O(n)$ time}
        \State $Y\gets Q_AP\left(\frac{I_k}{A}\right)$ \Comment{$O(n^3+n^2k)$ time}
        \State $i_1,\ldots,i_k\gets\texttt{ATLAS\_identify\_chart}\left(Y\right)$ \Comment{Alg.~\ref{alg:grass_identify}; $O(n^2k+nk^2+k^3)$ time}
        \State $\tilde{A}\gets\texttt{ATLAS\_ingest\_matrix}\left(Y,i_1,\ldots,i_k\right)$ \Comment{Alg.~\ref{alg:grass_ingest}; $O(nk^2+k^3)$ time}
        \State $\tilde{Q}_A\gets\texttt{QA\_from\_A}\left(\tilde{A}\right)$ \Comment{Alg.~\ref{alg:qa_from_a}; $O(n^3+n^2k+nk^2+k^3)$ time}
        \State $Q_A\gets P\tilde{Q}_AP^\top$ \Comment{$O(n^3)$ time}
        %\State $Q_{A,U}\gets$ restriction of $Q_A$ to columns in $\{i_1,\ldots,i_k\}$ \Comment{$O(1)$ time}
        %\State $Q_{A,L}\gets$ restriction of $Q_A$ to columns not in $\{i_1,\ldots,i_k\}$ \Comment{$O(1)$ time}
        \State $A\gets\mathbf{0}_{n-k,k}$ \Comment{$O(1)$ time}
        %\State \Return $A,Q_A,Q_{A,U},Q_{A,L}$
        \State \Return $A,Q_A$
    \end{algorithmic}
\end{algorithm}
\subsection{Transition maps in the atlas graph representation of the Grassmannian}\label{app:grass_atlas_graph}
Here, we explain when we invoke transition maps between charts on the atlas graph representation of the Grassmann manifold $\mathbf{Gr}_{n,k}$ given in Sec.~\ref{sec:grass_example}. For a given Ehressman chart, the adjacent charts are determined by the permutation Eq.~\ref{eqn:permutation_indices}, one existing for every dimension of the manifold. The condition under which we change charts within this atlas graph representation is given in Claim~\ref{clm:when_transition}. For this reason, a transition boundary exists for each coordinate such that, if the coordinate exceeds an absolute value of one, the transition map is invoked.
%As shown in Claims~\ref{clm:grass_dist_0} and \ref{clm:grass_dist_other}, $\mathbf{col}V_t$ is equidistant between $\mathbf{col}U$ and $\mathbf{col}\left(QU\right)$ when $t=\pm\frac{\pi}{4}$.
\begin{claim}\label{clm:when_transition}
    An element $A$ of a fixed Ehressman chart is closer to the center of the chart than to the center of any other Ehressman chart if all of the elements of $A$ have absolute value less than one.
\end{claim}
\begin{proof}
    The claim is proved for the Ehressman chart $(\mathcal{U}_0,\mathcal{V}_0.\varphi_0)$ by Claims~\ref{clm:grass_dist_0} and \ref{clm:grass_dist_other}. For the remaining Ehressman charts, the Claim is proved by observing we can conjugate elements of $\mathcal{U}_0$ by the permutation matrices used to define the remaining Ehressman charts.
\end{proof}

In the online Fr\'echet mean experiment, transition maps between coordinate charts are implemented according to Algorithm~\ref{alg:grass_trans_map}, whose runtime scales with complexity $O(n^3+n^2k+nk^2+k^3)$ but is constant with respect to the number of compressed charts traversed by the atlas graph. Since the quasi-Euclidean updates have $O(nk)$ time complexity, while the first-order update schemes for GiFEE, MANOPT, and MANOPT-RET have $\Omega(nk)$ time-complexity, the ATLAS framework can outspeed the other schemes when it can make infrequent use of high-cost chart transitions.

\subsection{Miscellaneous subroutines}\label{app:grass_misc_subroutines}
The online Fr\'echet mean estimation on the atlas graph (Algorithm~\ref{alg:grass_quasi_euclidean}) depends on the subroutines \texttt{ATLAS\_identify\_chart} (Algorithm~\ref{alg:grass_identify}), which identifies the closest Ehressman chart,  and \texttt{ATLAS\_ingest\_matrix} (Algorithm~\ref{alg:grass_ingest}), which gives the representation of  matrix in the current coordinate chart in the atlas graph.
% All first-order Grassmann frameworks in the online Fr\'echet mean experiments (Fig.~\ref{fig:first_order_subroutines}) make integral use of the subroutine \texttt{StiefelFromQR} (Algorithm~\ref{alg:QR}).

\begin{figure}[t]
    \centering
    \fbox{\begin{subfigure}{0.45\columnwidth}
        \caption{\texttt{GifeeLog}}
        \begin{algorithmic}
            \footnotesize
            \Require $\mathcal{X},\mathcal{Y}\in\mathbf{Gr}_{n,k}$
            \Require $X\in\mathbb{R}^{n\times k},\mathbf{colproj}(X)=\mathcal{X}$
            \Require $Y\in\mathbb{R}^{n\times k},\mathbf{colproj}(Y)=\mathcal{Y}$
            \State $A\gets\left(I-X\left(X^\top X\right)^{-1}X^\top\right)Y\left(X^\top Y\right)^{-1}$
            \State $U,\Sigma,V^\top\gets\texttt{ThinSVD}(A)$
            \State \Return $U,\Sigma,V^\top$
            %\State $\Theta\gets\arctan\Sigma$
            %\State \Return $U\Theta V^\top$
        \end{algorithmic}
    \end{subfigure}}
    \fbox{\begin{subfigure}{0.45\columnwidth}
        \caption{\texttt{ManoptLog}}
        \begin{algorithmic}
            \footnotesize
            \Require $\mathcal{X},\mathcal{Y}\in\mathbf{Gr}_{n,k}$
            \Require $X\in\mathbb{R}^{n\times k},\mathbf{colproj}(X)=\mathcal{X}$
            \Require $Y\in\mathbb{R}^{n\times k},\mathbf{colproj}(Y)=\mathcal{Y}$
            \State $A\gets\left(I-X\left(X^\top X\right)^{-1}X^\top\right)Y\left(X^\top Y\right)^{-1}$
            \State $U,\Sigma,V^\top\gets\texttt{ThinSVD}(A)$
            \State $\Theta\gets\arctan\Sigma$
            \State \Return $U\Theta V^\top$
        \end{algorithmic}
    \end{subfigure}} \\
    \fbox{\begin{subfigure}{0.45\columnwidth}
        \caption{\texttt{GifeeExp}}
        \begin{algorithmic}
            \footnotesize
            \Require $\mathcal{X}\in\mathbf{Gr}_{n,k}$
            \Require $X\in\mathbb{R}^{n\times k},\mathbf{colproj}(X)=\mathcal{X}$
            \Require $U,\Sigma,V^\top$ from \texttt{GifeeLog}
            \Require iteration number $i>0$
            \State $\Theta\gets\arctan\Sigma$
            \State $\tilde{Z}\gets UV\cos(\Theta/i)+U\sin(\Theta/i)$
            \State $Z,R\gets$ QR decomposition of $\tilde{Z}$
            \State \Return $Z$
        \end{algorithmic}
    \end{subfigure}}
    \fbox{\begin{subfigure}{0.45\columnwidth}
        \caption{\texttt{ManoptExp}}
        \begin{algorithmic}
            \footnotesize
            \Require $\mathcal{X}\in\mathbf{Gr}_{n,k}$
            \Require $X\in\mathbb{R}^{n\times k},\mathbf{colproj}(X)=\mathcal{X}$
            \Require $L$ from \texttt{ManoptLog}
            \Require iteration number $i>0$
            \State $U,\Sigma,V^\top\gets\texttt{ThinSVD}(L/i)$
            \State $\tilde{Z}\gets XV\cos(\Sigma) V^\top+U\sin(\Sigma) V^\top$
            \State $Z,R\gets$ QR decomposition of $\tilde{Z}$
            \State \Return $Z$
        \end{algorithmic}
    \end{subfigure}} \\
    \fbox{\begin{subfigure}{0.45\columnwidth}
        \caption{\texttt{ManoptRet}}
        \begin{algorithmic}
            \footnotesize
            \Require $\mathcal{X}\in\mathbf{Gr}_{n,k}$
            \Require $X\in\mathbb{R}^{n\times k},\mathbf{colproj}(X)=\mathcal{X}$
            \Require $L$ from \texttt{ManoptLog}
            \Require iteration number $i>0$
            \State $U,\Sigma,V^\top\gets\texttt{ThinSVD}(X+L/i)$
            \State $\tilde{Z}\gets UV^\top$
            \State $Z,R\gets$ QR decomposition of $\tilde{Z}$
            \State \Return $Z$
        \end{algorithmic}
    \end{subfigure}}
    %\caption{\textbf{Riemannian logarithm and retraction implementations used by the non-ATLAS first-order update schemes in Sec.~\ref{sec:grass_experiment}.} \textbf{a.} The Riemannian logarithm implementation used in the GiFEE algorithm. \textbf{b.} The Riemannian logarithm implementation used in both MANOPT and MANOPT-RET update schemes. \textbf{c.} The exponential map implementation used in the GiFEE algorithm. \textbf{d.} The exponential map implementation used in the MANOPT update scheme. \textbf{e.} The retraction implementation used in the MANOPT-RET update scheme.}
    \caption{\textbf{Riemannian logarithm and retraction algorithms used by the non-ATLAS first-order update schemes in Sec.~\ref{sec:grass_experiment}.}}\label{fig:first_order_subroutines}
\end{figure}

\begin{algorithm}[t]
    \caption{\texttt{QA\_from\_A} ($Q_A$ for $A\in\mathcal{V}_0$ according to Claim \ref{clm:Q_A})}\label{alg:qa_from_a}
    \begin{algorithmic}
        \Require $A\in\mathbb{R}^{n-k,k}$
        \State $Q_A\gets\left(\begin{array}{c|c}
            \sqrt{\left(I_k+A^\top A\right)^{-1}} & -A^\top\sqrt{\left(I_{n-k}+AA^\top\right)^{-1}} \\
            \hline
            A\sqrt{\left(I_k+A^\top A\right)^{-1}} & \sqrt{\left(I_{n-k}+AA^\top\right)^{-1}}
        \end{array}\right)$ \Comment{$O(n^3+n^2k+nk^2+k^3)$ time}
        \State \Return $Q_A$
    \end{algorithmic}
\end{algorithm}

\begin{algorithm}[t]
    \caption{\texttt{ATLAS\_identify\_chart} (identify chart in Ehresmann atlas whose center is closest to $\mathbf{colspan}X$)}\label{alg:grass_identify}
    \begin{algorithmic}
        \Require $X\in\mathbb{R}^{n\times k}$, full rank
        \State $P\gets X\left(X^\top X\right)^{-1}X^\top$ \Comment{$P=\mathbf{colproj}X$; $O(n^2k+nk^2+k^3)$ time}
        \For{$j\in\{1,\ldots,k\}$} \Comment{$O(nk)$ time}
            \State $i_j\gets i\text{ such that }P_{ii}\text{ is }j\text{th largest diagonal entry of }P$
            \State $i_1,\ldots,i_k\gets i_1,\ldots,i_k$ in increasing order
        \EndFor
        \State \Return $i_1,\ldots,i_k$
    \end{algorithmic}
\end{algorithm}

\begin{algorithm}[t]
    \caption{\texttt{ATLAS\_ingest\_matrix} (ingest Grassmann element represented as full-rank matrix into Ehresmann chart)}\label{alg:grass_ingest}
    \begin{algorithmic}
        \Require $X\in\mathbb{R}^{n\times k}$, full rank
        \Require $1\leq i_1\leq\ldots\leq i_k\leq n$ specifying Ehresmann chart
        \State $X_U\gets$ restriction of $X$ to rows in $\{i_1,\ldots,i_k\}$ \Comment{$O(1)$ time}
        \State $X_L\gets$ restriction of $X$ to rows not in $\{i_1,\ldots,i_k\}$ \Comment{$O(1)$ time}
        \State $A\gets X_LX_U^{-1}$ \Comment{$A=\varphi_{i_1,\ldots,i_k}\left(\mathbf{colproj}\left(X\right)\right)$; $O(nk^2+k^3)$}
        \State \Return $A$
    \end{algorithmic}
\end{algorithm}

\subsection{Non-uniform sampling of the Grassmann manifold by geodesic power scaling}\label{sec:geodesic_power_distribution}
For the Grassmann experiments, points for online sampling were generated by sampling randomly from the Grassmann manifold. To generate points from a distribution with a well-defined Fr\'echet mean, such as the uniform measure on balls of a fixed radius on the Grassmannian, we used the following method, which we call \textit{geodesic power scaling}. 

Fix $p>1$ and $\mathcal{X}\in\mathbf{Gr}_{n,k}$. Points are sampled from the \textit{geodesic power distribution} $\mathbf{GPD}(\mathcal{X}, p)$ as follows:
\begin{enumerate}
    \setlength\itemsep{0em}
    \item A point $\mathcal{Y}$ is sampled uniformly from $\mathbf{Gr}_{n,k}$.
    \item The Grassmann distance $\delta$ between $\mathcal{X}$ and $\mathcal{Y}$ is computed.
    \item A new point $\mathcal{Y}^\prime$ is computed as $\mathcal{Y}^\prime=\exp_\mathcal{X}\left(\left(\frac{\delta}{\delta_{\max}}\right)^p\log_\mathcal{X}\mathcal{Y}\right)$, where $\delta_{\max}$ is the largest possible Grassmann distance between two points on $\mathbf{Gr}_{n,k}$.\footnote{$\delta_{\max}=\frac{\pi}{2}\sqrt{\max\{k,n-k\}}$} 
\end{enumerate}

Whenever we sample from $\mathbf{GPD}(\mathcal{X},p)$, we assume that the sample takes the form of a matrix $X\in\mathbb{R}^{n\times k}$ such that $\mathcal{X}$ is the column space of $X$. Increasing the scaling exponent $p$ reduces the entropy of the distribution by concentrating probability around $\mathcal{X}$; as $p\to\infty$, $\mathbf{GPD}(\mathcal{X},p)$ approaches the Dirac delta based at $\mathcal{X}$. Note that the distribution is invariant under action by the special orthogonal group, as long as the action fixes $\mathcal{X}$, giving the distribution ``rotational symmetry'' about $\mathcal{X}$. Moreover, $\mathbf{GPD}(\mathcal{X},p)$ has unique population Fr\'echet mean $\mathcal{X}$ for all $p>1$.

\section{Atlas graph representation of high-contrast image patches}\label{sec:carlsson_computations}
\begin{figure}[t]
    \centering
    \begin{overpic}[width=0.5\columnwidth]{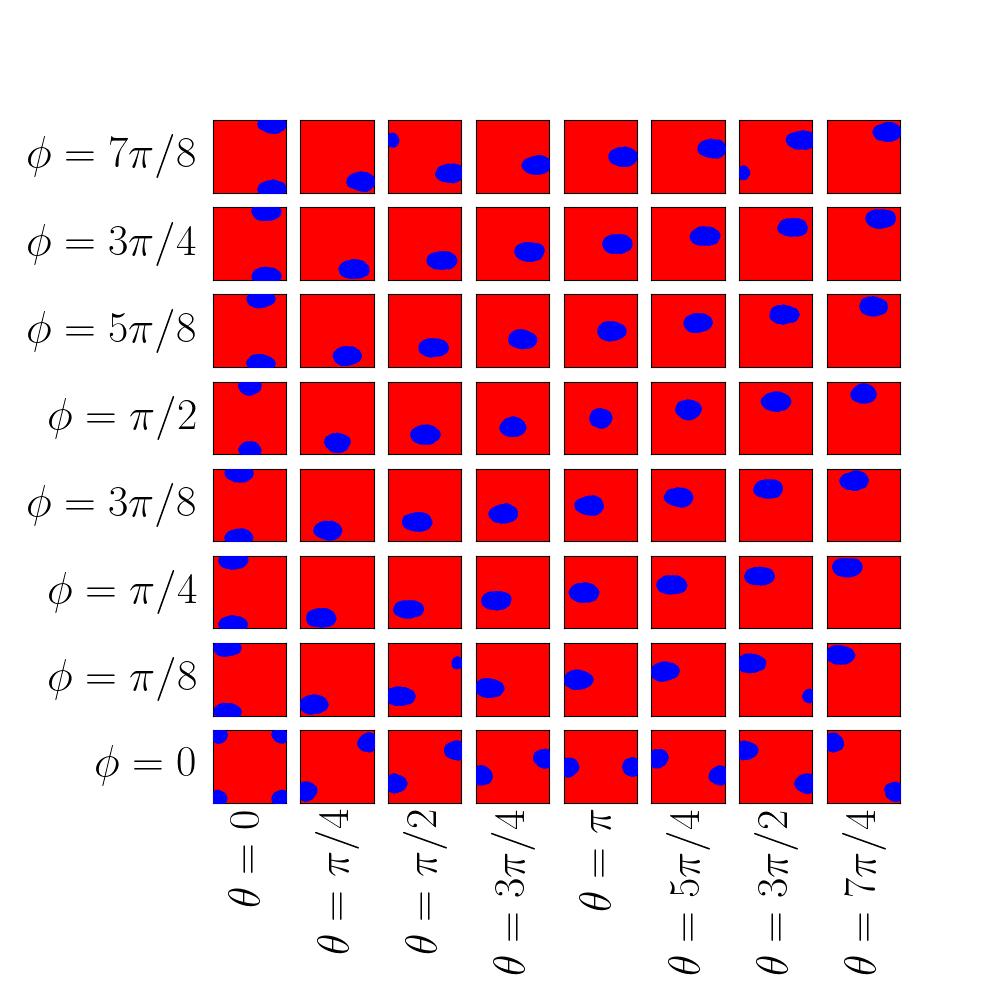}\put(5,95){\textbf{A}}\end{overpic}\begin{overpic}[width=0.5\columnwidth]{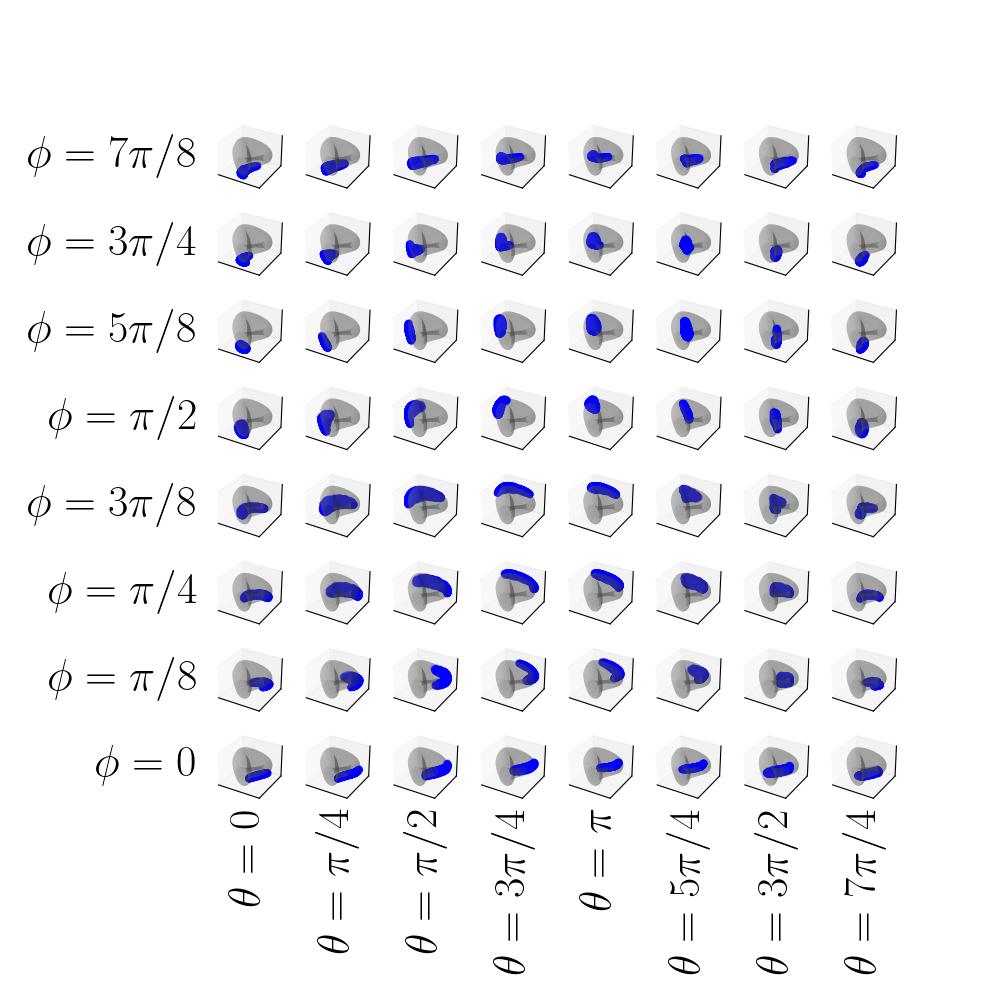}\put(5,95){\textbf{B}}\end{overpic}
    %\vspace{-1cm}
    \caption{\textbf{Our approximate coordinate charts on the Carlsson manifold cover the entire manifold (Sec.~\ref{sec:atlas_graph_for_carlsson}).} (\textbf{A}) Coordinate charts used in our atlas graph representation of the Carlsson manifold, depicted in polar coordinates. Blue denotes points belonging to the neighborhood from the point sample, while red denoted points not belonging to the neighorhood from the point sample. (\textbf{B}) The same coordinate charts depicted on the Karcher representation of the Klein bottle.}
    \label{fig:carlsson_atlas graph}
\end{figure}

\subsection{Persistence diagram computation (Sec. \ref{sec:atlas_graph_for_carlsson}, Fig. \ref{fig:h2_homology}A,B)}\label{sec:persistent_homology}
We compute persistent homology on point clouds generated from our atlas graph representation of high-contrast image patches, as well as on points sampled from a parametrization from Sritharan \textit{et al.} \citep{sritharan} of the same manifold (Fig.~\ref{fig:h2_homology}).

Point samples from the atlas graph were generated by performing the following process for each of the 64 charts $\left(\tilde{\mathcal{U}},\mathcal{V},\tilde{\varphi}\right)$ in our atlas graph representation (Sec.~\ref{sec:primitives}):
\begin{itemize}
    \setlength\itemsep{0em}
    \item For chart $\mathcal{V}$ of radius $r$, a point $\tv$ is sampled uniformly from the ball of radius $r$ in $\mathcal{V}$;
    \item $\tv$ is rejected if it does not belong to the coordinate chart according to the kernel SVM classifier given by Equation~\ref{eqn:simplified_quartic};
    \item if $\tv$ is not rejected, return $\tilde{\varphi}^{-1}\left(\tv\right)\in\mathbb{R}^9$.  
\end{itemize}
For the Sritharan parametrization of the Klein bottle, a point is sampled by sampling coordinates $(\theta,\phi)$ from the uniform measure on $[0,\pi]\times [0,2\pi]$, mapping the coordinates into $\mathbb{R}^{3\times 3}$ by the restriction of the map $k_{\theta,\phi}$ to $\left\{-1,0,1\right\}\times\left\{-1,0,1\right\}$ (Sec.~\ref{sec:carlsson_example}), and reshaping the result as an element of $\mathbb{R}^9$. 

Ten points were sampled for each of the 64 charts in the atlas graph, for a total of 640 points, and $640$ points were sampled from Sritharan's parametrization of the Klein bottle. 

Persistent homology was computed on the samples by Vietoris-Rips complex using the Python package Ripser \citep[version 0.6.0,][]{ripser}. 

%For PCA, t-SNE, and UMAP, the transformation was computed on the points sampled from the Sritharan parametrization. The persistent homology was then computed on the transformed points by Vietoris-Rips complex using Ripser~\cite{ripser}.

\subsection{Pairwise aggregate bottleneck distance computation (Sec.~\ref{sec:atlas_graph_for_carlsson}, Fig.~\ref{fig:h2_homology}C)}\label{app:bottleneck_heatmap}
Let $P_0,P_1,P_2$ be persistence diagrams generated from data $X$ corresponding to $H_0$, $H_1$, and $H_2$ features, respectively. Let $Q_0,Q_1,Q_2$ be generated similarly from data $Y$. We define the \textit{aggregate bottleneck distance} between diagram triplets $\left(P_0,P_1,P_2\right)$ and $\left(Q_0,Q_1,Q_2\right)$ as $\sqrt{\sum_{i=0}^2d^2_\text{B}\left(P_i,Q_i\right)}$, where $d_\text{B}$ denotes the bottleneck distance \citep[e.g.,][]{oudot}. This distance function %grows roughly linearly with the 
reflects the bottleneck distance between persistence diagrams in each dimension, making it a natural way to measure the preservation of $H_0$, $H_1$, and $H_2$ features by different dimensionality reduction measures.

To create the heatmap in Fig.~\ref{fig:h2_homology}C, the same points were used as described in Appendix~\ref{sec:persistent_homology}. PCA (restricted to the top two or five principal components) and a two-dimensional UMAP were each computed on the points sampled from the Sritharan parametrization. We excluded $t$-SNE from the bottleneck distance computations, due to its very high aggregate bottleneck distances to to all other methods.

\subsection{Geodesic distance computations (Fig.~\ref{fig:dist_fig})}\label{sec:geodesic_distances}
For the experiment comparing geodesic distances, we approximated true geodesic distance as follows. First, we generated sample points from the Sritharan parametrization by mapping a $1000\times 1000$ grid of evenly spaced points in the set $[0,\pi]\times[0,2\pi]$ into $\mathbb{R}^{3\times3}$ through the restriction of $k_{\theta,\phi}$ to $\{-1,0,1\} \times \{-1,0,1\}$ (Sec.~\ref{sec:carlsson_example}), and then reshaped these as vectors in $\mathbb{R}^9$. These points were used to create a $5$-nearest-neighbors graph, with each edge weighted by the Euclidean distance between its endpoints. True geodesic distance between a given pair of points was then approximated by computing a shortest path in this graph. 

For PCA, $t$-SNE, and UMAP, geodesic distances were computed via an analogous approach. That is, each of these transformations was used to map the sampled grid points into a new, separate representation, on which $5$-nearest-neighbor graphs with distance-weighted edges were computed. Note that these graphs have one million nodes, while Algorithm~\ref{alg:dense_graph} results in a graph of only 28,700 nodes for the values of $\epsilon$ and $\delta$ specified; therefore, naively, one might expect geodesic distances to be better preserved by PCA, $t$-SNE, and UMAP in the transformed space than by the atlas graph. For a given transformation (i.e., PCA, $t$-SNE, and UMAP) and a pair of grid points, the geodesic distance between two points was approximated as the shortest path between them in the corresponding $5$-nearest-neighbor graph. 

To investigate how well geodesic distances were preserved by the atlas graph, PCA, $t$-SNE, and UMAP, 100 pairs of points were randomly sampled without replacement from the one million grid points in the Sritharan parametrization. For the atlas graph, which was generated %using Algorithm~\ref{alg:carlsson_atlas graph}.
as described in Sec.~\ref{sec:atlas_graph_for_carlsson}, these 100 pairs of points were ingested into the atlas graph, and the geodesic distance between each pair of points was approximated as %by using a dense subgraph (as computed in Algorithm~\ref{alg:dense_graph}) to compute
the na\"ive approximate distance between the points (Sec.~\ref{sec:app_naive_dist}).
For PCA, $t$-SNE, and UMAP, each point was added into the 5-nearest-neighbor graph as a node connected solely to its closest point in the graph, and geodesic distances were approximated between pairs of points using shortest paths in the graph.

\subsection{Packages used for non-atlas-graph dimensionality reduction}\label{app:non_atlas_packages}
PCA and $t$-SNE were computed using scikit-learn (version 1.3.2)~\cite{scikit-learn}, while UMAP was computed using the official UMAP package for Python (version 0.5.1)~\cite{mcinnes}. We ran $t$-SNE was  with \texttt{perplexity=5.0} and otherwise default arguments. UMAP was computed with \texttt{n\_neighbors=5} and otherwise default arguments.

\section{Computing Riemannian principal boundaries (Sec.~\ref{sec:rpb})}\label{sec:exp_carlsson}
% \begin{figure}[t]
%     \centering
%     \begin{overpic}[width=0.48\columnwidth, trim={2cm, 0, 1cm, 0}]{grassmann_graphics_nuevo/iterations_cropped.jpg}\put(5,70){\textbf{A}}\end{overpic}\begin{overpic}[width=0.48\columnwidth, trim={2cm, 0, 1cm, 0}]{grassmann_graphics_nuevo/iterations_empirical_cropped.jpg}\put(5,70){\textbf{B}}\end{overpic}
%     %\includegraphics[width=0.48\columnwidth, trim={2cm, 0, 2cm, 0}]{grassmann_graphics_nuevo/iterations_cropped.jpg}
%     %\includegraphics[width=0.48\columnwidth, trim={2cm, 0, 2cm, 0}]{grassmann_graphics_nuevo/iterations_empirical_cropped.jpg} \\
%     \includegraphics[width=\columnwidth, trim={0cm, 2cm, 0cm, 2cm}]{grassmann_graphics_nuevo/legend.jpg}\vspace{0.5cm}
%     \caption{\textbf{The accuracy loss from using quasi-Euclidean updates is negligible in comparison to other first-order update schemes on the Grassmann manifold.} (\textbf{A.}) The Grassmann distance between each iterate and the population Fr\'echet mean of $\mathbf{GPD}(\mathcal{X},p)$ with respect to number of iterations, for all first-order implementations, for $n\leq1000$, $k\leq10$, $p\in\{2,3\}$. (\textbf{B.}) The Grassmann distance between each iterate and the empirical Fr\'echet mean of $\mathbf{GPD}(\mathcal{X},p)$ with respect to number of iterations, for all first-order implementations, for $n\leq1000$, $k\leq10$, $p\in\{2,3\}$. The observations constituting the empirical distribution are drawn from $\mathbf{GPD}(\mathcal{X},p)$ as in Appendix \ref{sec:geodesic_power_distribution}.}
%     \label{fig:population_frechet_error}
% \end{figure}

\subsection{Parameter choices and modifications}\label{sec:params_and_mods}
We implemented the RPB algorithm as described in \citep{yao_2020}, with the exceptions described here. 

In Yao, \textit{et al.}~\cite{yao_2020}, Sec.~2.2, a univariate kernel $\kappa_h$ is used to define which points are included in the computation of the local covariance matrix $\Sigma_h$. For this kernel, we use the indicator function for the ball of radius $h$.

The original algorithm computes a weighted average of the first derivatives of the two principal flows, where $\lambda_\delta(t)$ is the weight at iteration $t$ \citep[Equation~10]{yao_2020}. In our implementation, we assume that $\lambda_\delta(t)=1/2$ for all $t$. This choice was motivated by both simplicity and practical considerations, i.e., it helped avoid issues where the boundary could collapse into one of the principal flows. 

The differential equation for a principle flow $\gamma^+$ \citep[induced by Equation~5]{yao_2020} effectively follows the vector field defined by the top eigenvector of the local covariance matrix $\Sigma_h$. To avoid oscillations in the principle flow direction, due either to the insensitivity of eigenvector computations to multiplication by $-1$ or non-smooth changes in the top eigenvectors, we enforce a positive inner product between tangent vectors in adjacent iterates. 
%the vector field can be nonsmooth in the case that the top eigenvalue has multiplicity greater than one on some set, even a set of measure zero, e.g., there are points at which the first and second eigenspaces ``switch places.'' %For this reason, we find integrating the differential equation for $\gamma^+$ as given expressly in the paper to be difficult due to such eigenspace degeneracies. 
% ; degenerate top eigenspaces can cause rotations in the vector field by ninety degrees, which precludes one from further integrating $\gamma^+$ in a meaningful way. 
Further,
%due to inevitable nonuniformity from random sampling, 
if the support of the data is sufficiently sparse, in practice, the top eigenvector field for the principle flow will be dominated by noise and will eventually cause the boundary curve to move away from the data, which also causes $\Sigma_h$ to become undefined. 
To prevent this from happening, we use the following modification to correct the principle flow solution by moving it towards the mean of the local data:

%\samnote{Come back and math check this equation/rule.} 
Let $W$ be the top eigenvector field of the local covariance matrix $\Sigma_h$ \citep[as in Equation~4,][]{yao_2020}. Instead of following the update rule $\dot{\gamma}=W(\gamma),$ we instead follow the rule
\begin{equation}\label{eqn:modified_yao}
    \dot{\gamma}=W(\gamma)+\alpha\left(I-W(\gamma)W(\gamma)^\top\right)\left(\frac{1}{\sum_i\kappa_h\left(x_i,\gamma\right)}\sum_i\log_\gamma x_j\right),
\end{equation}
where $\alpha>0$ is a correction factor that moves $\dot{\gamma}$ toward the mean Riemannian logarithm of nearby sample points, projected onto the orthocomplement $I-W(\gamma)W(\gamma)^\top$ of the top eigenvector of the local covariance matrix. 

% \begin{algorithm}[t]
%     \begin{algorithmic}
%         \caption{Approximate Fr\'echet mean using naive approximate distance}\label{alg:atlas_frechet}
%         \Require chart indices $\{1,\ldots,k\}$; $\vec{x}_0,\ldots,\vec{x}_{n-1}\in\mathcal{M}$; $\epsilon>0$
%         \State $\delta\gets\infty$
%         \While{$\delta>\epsilon$}
%             \For{$j\gets0,\ldots,n-1$}
%                 \State $i_j\gets$ chart assignment for $\vec{x}_j$
%                 \State $\delta_j\gets0$
%                 \For{$j^\prime\gets0,\ldots,n-1$}
%                     \State $\vec{\xi}_{jj^\prime}\gets\log_{i_j}\vec{x}_{j^\prime}$
%                     \State $\delta_j\gets\max\left\{\delta_j,\tilde{d}_{i_j}\left(\vec{\xi}_{jj},\vec{\xi}_{jj^\prime}\right)\right\}$
%                 \EndFor
%                 \State $\vec{\xi}^\prime_j\gets\left(\sum_{j^\prime=0}^{n-1}\vec{\xi}_{jj^\prime}\right)/n$
%             \EndFor
%         \State $\delta\gets\max_j\delta_j$
%         \For{$j\gets0,\ldots,n-1$}
%             \State $\vec{x}_j\gets\exp_{i_j}\vec{\xi}^\prime_j$
%         \EndFor
%         \EndWhile
%         \State \Return $\vec{x}_0$
%     \end{algorithmic}
% \end{algorithm}

\subsection{Transition boundaries}\label{sec:transition_boundaries}
Transition maps on our atlas graph representation of Carlsson's natural image patches are constructed using the methodology behind Expression (\ref{eqn:simplified_quartic}). Coordinate charts are constructed by a least-squares quadratic fit as in \citep{sritharan}. The points for the $i$th fit were those within radius\footnote{We found that $r_i=1.1$ for all $i$ was sufficient to cover the Carlsson manifold, whilst being small enough to maintain topological triviality for all $i$.} $r_i$ of center $\vec{p}_i$. To be pedantic: this is the set of all sample points $\vec{x}_j$ satisfying the quadratic $\vec{y}_j^\top\vec{y}_j<r_i^2$ under the change of coordinates $\vec{y}_j=\vec{x}_j-\vec{p}_j$. Within the $i$th coordinate chart, then, Expression (\ref{eqn:simplified_quartic}) yields the inequality
\begin{equation}\label{eqn:transition_boundary_klein}
    \frac{1}{4}\left(\tv\otimes\tv\right)^\top K_i^\top K_i\left(\tv\otimes\tv\right)+\vec{h}_i^\top K_i\left(\tv\otimes\tv\right)+\tv^\top\tv+\vec{h}_i^\top\vec{h}_i-r_i^2<0,
\end{equation}
where $L_i,M_i,K_i$ are as in Equation~\ref{eqn:quad_form_general}.

\section{Theoretical results for the Grassmann manifold}

\subsection{Ingesting columnspanning matrix into atlas graph of the Grassmannian}\label{sec:ingestion}
Say $X\in\mathbb{R}^{n\times k}$ has full rank. Thinking of $\mathbf{Gr}_{n,k}$ as the manifold of $n\times n$ orthogonal projection matrices of rank $k$, we know that the columnspace of $X$ is uniquely represented in $\mathbf{Gr}_{n,k}$ by $X\left(X^\top X\right)^{-1}X^\top$. Finding the chart to which $X$ belongs is tantamount to finding the ``centerpoint'' projection matrix to which $X\left(X^\top X\right)^{-1}X^\top$ is closest. This, in turn, is equivalent to finding the centerpoint projection matrix with which $X\left(X^\top X\right)^{-1}X^\top$ has the highest Frobenius inner product. This is accomplished by finding the $k$ largest diagonal entries of $X\left(X^\top X\right)^{-1}X^\top$, as demonstrated by the following series of deductions.
\begin{align*}
    \left\langle X\left(X^\top X\right)^{-1}X^\top,\sum_{j=1}^k\vec{e}_{i_j}\vec{e}_{i_j}^\top\right\rangle_{\text{Fr}}&=\mathbf{Tr}\left[X\left(X^\top X\right)^{-1}X^\top\sum_{j=1}^k\vec{e}_{i_j}\vec{e}_{i_j}^\top\right] \\
    &=\sum_{j=1}^k\mathbf{Tr}\left[X\left(X^\top X\right)X^\top\left(\vec{e}_{i_j}\vec{e}_{i_j}^\top\right)\right] \\
    &=\sum_{j=1}^k\left[X\left(X^\top X\right)^{-1}X^\top\right]_{i_ji_j}
\end{align*}
Note that the diagonal entries of $X\left(X^\top X\right)^{-1}X^\top$ are guaranteed to be nonnegative by the positive-semidefiniteness of $X\left(X^\top X\right)^{-1}X^\top$.

More generally, let $Q$ and $P$ be orthogonal matrices such that
\begin{equation*}
    Q\left(\begin{array}{c|c}
        I_k & \mathbf{0}_{k,n-k} \\
        \hline
        \mathbf{0}_{k,n-k} & \mathbf{0}_{n-k,n-k}
    \end{array}\right)Q^\top=P.
\end{equation*}
We ingest a full-rank matrix $X\in\mathbb{R}^{n\times k}$ into the chart centered at $P$ by the map
\begin{equation}\label{eqn:ingestion_general}
    X\mapsto X_LX_U^{-1},
\end{equation}
where $X_U,X_L$ are given by
\begin{equation*}
    X_U=\left(I_k\mid\mathbf{0}_{k,n-k}\right)Q^\top X,\hspace{2cm}X_L=\left(\mathbf{0}_{n-k,k}\mid I_{n-k}\right)Q^\top X.
\end{equation*}

\subsection{Miscellaneous Grassmann computations} \label{app:misc_grass_comp}
\begin{claim}\label{clm:grass_dist_0}
    Let $\varphi_0$ be the Ehresmann coordinate chart map (Sec.~\ref{sec:grass_example}). For $t\in\mathbb{R}$, we define $V_t:=\varphi_0^{-1}\left(t\vec{e}_{n-k}\vec{e}_k\right)=\left(\frac{I_k}{t\vec{e}_{n-k}\vec{e}_k^\top}\right)$. The Grassmann distance between $\mathbf{colproj}\left(V_t\right)$ and $\mathbf{colproj}(U)$, where $U:=\varphi_0^{-1}\left(\mathbf{0}_{n-k,k}\right)=\left(\frac{I_k}{\mathbf{0}_{n-k,k}}\right)$, is equal to $\left\lvert\arctan t\right\rvert$.
    \end{claim}

\begin{proof}
        The square of the Grassmann distance $\mathbf{dist}_{\mathbf{Gr}}$ between two projection matrices $P,Q\in\mathbf{Gr}_{n,k}$ is given by the sum of squares of Jordan angles between their subspaces. From Lemma 5 in \citep{neretin}, if matrices $U_P,U_Q\in\mathbb{R}^{n\times k}$ satisfy $P=\mathbf{colproj}U_P$ and $Q=\mathbf{colproj}U_P$, the squares of cosines of the Jordan angles between $P$ and $Q$ are the eigenvalues of the matrix $\left(U_P^\top U_P\right)^{-1}U_P^\top U_Q\left(U_Q^\top U_Q\right)^{-1}U_Q^\top U_P$. Therefore, the following series of deductions holds.
        \begin{align*}\mathbf{dist}_{\mathbf{Gr}}\big(\varphi_0\left((t\vec{e}_k\vec{e}_k^\top)\right),\varphi_0\left(\mathbf{0}_{n-k,k}\right)\big)&=\sqrt{\mathbf{Tr}\left[\arccos\left(\sqrt{\left(U^\top U\right)^{-1}U^\top V_t\left(V_t^\top V_t\right)^{-1}V_t^\top U}\right)^2\right]} \\     &=\sqrt{\mathbf{Tr}\left[\arccos\left(\sqrt{\left(I_k+t^2\vec{e}_k\vec{e}_k^\top\right)^{-1}}\right)^2\right]} \\
            &=\sqrt{\mathbf{Tr}\left[\arccos\left(\sqrt{I_k-\frac{t^2}{1+t^2}\vec{e}_k\vec{e}_k^\top}\right)^2\right]} \\
            &=\sqrt{\mathbf{Tr}\left[\left(\begin{array}{c|c}
                \mathbf{0}_{k-1,k-1} & \mathbf{0}_{k-1,1} \\
                \hline
                \mathbf{0}_{1,k-1} & \arccos\left(\sqrt{\frac{1}{1+t^2}}\right)
            \end{array}\right)^2\right]} \\
            &=\left\lvert\arctan t\right\rvert
        \end{align*}
    \end{proof}
%\end{claim}

\begin{claim}\label{clm:grass_dist_other}
    Let $Q:=\left(\begin{array}{c|c}
        \mathbf{0}_{n-1,1} & I_{n-1} \\
        \hline
        1 & \mathbf{0}_{1,n-1}
    \end{array}\right)$ by the permutation matrix which moves up the indices of row vectors. Further, let $U,V_t$ be as in Claim~\ref{clm:grass_dist_0}, and let $U_\varnothing=QU$. The Grassmann distance between $\mathbf{col}\left(V_t\right)$ and $\mathbf{col}U_\varnothing$ is equal to $\lvert\arccot t\rvert$.
\end{claim}

\begingroup
\allowdisplaybreaks
\begin{proof}
        Following from the discussion of Jordan angles in \citep{neretin}, we get:\\
        
        \begin{align*}\mathbf{dist}_{\mathbf{Gr}}\left(\mathbf{col}V_t,\mathbf{col}U_\varnothing\right)&=\left(\mathbf{Tr}\left[\arccos\left(\left[\left(U_\varnothing^\top U_\varnothing\right)^{-1}U_\varnothing^\top V_t\left(V_t^\top V_t\right)^{-1}V_t^\top U_\varnothing\right]^{1/2}\right)^2\right]\right)^{1/2} \\
            &=\Bigg(\mathbf{Tr}\Bigg[\arccos\bigg(\bigg[\left(I_k\mid\mathbf{0}_{k,n-k}\right)Q^\top\left(\frac{I_k}{t\vec{e}_{n-k}\vec{e}_k^\top}\right)\left(I_k-\frac{t^2}{1+t^2}\vec{e}_k\vec{e}_k^\top\right) \\
            &\hspace{1cm}\cdot\left(I_k\mid t\vec{e}_k\vec{e}_{n-k}^\top\right)Q\left(\frac{I_k}{\mathbf{0}_{n-k,k}}\right)\bigg]^{1/2}\bigg)^2\Bigg]\Bigg)^{1/2} \\
            &=\Bigg(\mathbf{Tr}\Bigg[\arccos\bigg(\bigg[\left(\begin{array}{c|c}
                \mathbf{0}_{1,k-1} & t \\
                \hline
                I_{k-1} & \mathbf{0}_{k-1,1}
            \end{array}\right)\left(I_k-\frac{t^2}{1+t^2}\vec{e}_k\vec{e}_k^\top\right) \\
            &\hspace{1cm}\cdot\left(\begin{array}{c|c}
                \mathbf{0}_{k-1,1} & I_{k-1} \\
                \hline
                t & \mathbf{0}_{1,k-1}
            \end{array}\right)\bigg]^{1/2}\bigg)^2\Bigg]\Bigg)^{1/2} \\
            &=\left(\mathbf{Tr}\left[\arccos\left(\left[\left(\begin{array}{c|c}
                t^2 & \mathbf{0}_{1,k-1} \\
                \hline
                \mathbf{0}_{k-1,1} & I_{k-1}
            \end{array}\right)-\frac{t^4}{1+t^2}\vec{e}_1\vec{e}_1^\top\right]^{1/2}\right)^2\right]\right)^{1/2} \\
            &=\left(\mathbf{Tr}\left[\arccos\left(\left(\begin{array}{c|c}
                \frac{t^2}{1+t^2} & \mathbf{0}_{1,k-1} \\
                \hline
                \mathbf{0}_{k-1,1} & I_{k-1}
            \end{array}\right)^{1/2}\right)^2\right]\right)^{1/2} \\
            &=\left(\mathbf{Tr}\left[\left(\begin{array}{c|c}
                \arccos\left(\sqrt{\frac{t^2}{1+t^2}}\right) & \mathbf{0}_{1,k-1} \\
                \hline
                \mathbf{0}_{k-1,1} & \mathbf{0}_{k-1,k-1}
            \end{array}\right)^2\right]\right)^{1/2} \\
            &=\lvert\arccot t\rvert
        \end{align*}
    \end{proof}
\endgroup

\begin{proof}(of Claim~\ref{clm:Q_A}) 
\label{prf:Q_A}
%that an orthogonal matrix $Q_A$ which satisfies Equation %(\ref{eqn:qa_action}) is given by
%    $$Q_A=\left(\begin{array}{c|c}
 %       \sqrt{\left(I_k+A^\top A\right)^{-1}} & -A^\top\sqrt{\left(I_{n-k}+AA^\top\right)^{-1}} \\
 %       \hline
 %       A\sqrt{\left(I_k+A^\top A\right)^{-1}} & \sqrt{\left(I_{n-k}+AA^\top\right)^{-1}}
  %  \end{array}\right).$$
% \%begin{proof} 
        We first show that $Q_A$ is orthogonal. Observe that $Q_AQ_A^\top$ is equal to
        \begin{align*}
            \left(\begin{array}{c|c}
                \left(I_k+A^\top A\right)^{-1}+A^\top\left(I_{n-k}+AA^\top\right)^{-1}A & \left(I_k+A^\top A\right)^{-1}A^\top-A^\top\left(I_{n-k}+AA^\top\right)^{-1} \\
                \hline
                A\left(I_k+A^\top A\right)^{-1}-\left(I_{n-k}+AA^\top\right)^{-1}A & A\left(I_k+A^\top A\right)^{-1}A^\top+\left(I_{n-k}+AA^\top\right)^{-1}
            \end{array}\right).
        \end{align*}
        To show $Q_A$ is orthogonal, then, it suffices to show that:
        \begin{enumerate}[label=\arabic*)]
            \setlength\itemsep{0em}
            \item $\left(I_k+A^\top A\right)^{-1}A^\top-A^\top\left(I_{n-k}+AA^\top\right)^{-1}=\mathbf{0}_{k,n-k}$;
            \item $\left(I_k+A^\top A\right)^{-1}+A^\top\left(I_{n-k}+AA^\top\right)^{-1}A=I_k$; and
            \item $A\left(I_k+A^\top A\right)^{-1}A^\top+\left(I_{n-k}+AA^\top\right)^{-1}=I_{n-k}$.
        \end{enumerate}
        These are shown in Lemmas, \ref{lem:Q_A_lem_2}, \ref{lem:Q_A_lem_1}, and \ref{lem:Q_A_lem_3}, respectively. \\
        
        It remains to show that $Q_A\left(\begin{array}{c|c}
            I_k & \mathbf{0}_{k,n-k} \\
            \hline
            \mathbf{0}_{n-k,k} & \mathbf{0}_{n-k,n-k}
        \end{array}\right)Q_A^\top=\tilde{\varphi}_0^{-1}(A)$. By definition of $\varphi_0^{-1}$, 
        \begin{equation*}
            \varphi_0^{-1}(A)=\left(\begin{array}{c|c}
                \left(I+A^\top A\right)^{-1} & \left(I+A^\top A\right)^{-1}A^\top \\
                \hline
                A\left(I+A^\top A\right)^{-1} & A\left(I+A^\top A\right)^{-1}A^\top
            \end{array}\right),
        \end{equation*}
        and so completing the proof is a straightforward computation.
    %end{proof}
\end{proof}

\begin{lemma}\label{lem:Q_A_lem_2}
    \begin{equation*}
        \left(I_k+A^\top A\right)^{-1}A^\top-A^\top\left(I_{n-k}+AA^\top\right)^{-1}=\mathbf{0}_{k,n-k}
    \end{equation*}
    \end{lemma}
    \begin{proof}
        Using a Neumann series representation of the matrix inverse \citep[e.g.,][]{stewart}, the lemma is proved by the following series of deductions.
        \begin{align*}
            A^\top\left(I_{n-k}+AA^\top\right)^{-1}&=A^\top\sum_{j=0}^\infty(-1)^j\left(AA^\top\right)^j \\
            &=\left(\sum_{j=0}^\infty(-1)^j\left(A^\top A\right)^j\right)A^\top \\
            &=\left(I_k+A^\top A\right)^{-1}A^\top
        \end{align*}
    \end{proof}

\begin{lemma}\label{lem:Q_A_lem_1}
    \begin{equation*}
        \left(I_k+A^\top A\right)^{-1}+A^\top\left(I_{n-k}+AA^\top\right)^{-1}A=I_k
    \end{equation*}
    \end{lemma}
    \begin{proof}
        By Lemma \ref{lem:Q_A_lem_2}, the following series of deductions holds.
        \begin{align*}
            \left(I_k+A^\top A\right)^{-1}+A^\top\left(I_{n-k}+AA^\top\right)^{-1}A&=\left(I_k+A^\top A\right)^{-1}+\left(I_k+A^\top A\right)^{-1}A^\top A \\
            &=\left(I_k+A^\top A\right)\left(I_k+A^\top A\right)^{-1} \\
            &=I_k
        \end{align*}
    \end{proof}

\begin{lemma}\label{lem:Q_A_lem_3}
    \begin{equation*}
        A\left(I_k+A^\top A\right)^{-1}A^\top+\left(I_{n-k}+AA^\top\right)^{-1}=I_{n-k}
    \end{equation*}
    \begin{proof}
        A proof for this Lemma is easily recreated from the method used to prove Lemma \ref{lem:Q_A_lem_1}.
    \end{proof}
\end{lemma}

\begin{lemma}\label{len:A_commutes_with_root}
    % $A\sqrt{I_k+A^\top A}=\sqrt{I_{n-k}+AA^\top}A,\hspace{1cm}A^\top\sqrt{I_{n-k}+AA^\top}=\sqrt{I_k+A^\top A}A^\top$
    $$\begin{array}{c}
        A\sqrt{I_k+A^\top A}=\sqrt{I_{n-k}+AA^\top}A, \\
        A^\top\sqrt{I_{n-k}+AA^\top}=\sqrt{I_k+A^\top A}A^\top
    \end{array}$$
    \begin{proof}
        This proof relies on a Neumann series representation of the square root of a matrix \citep[e.g.,][]{stewart}.
        \begin{align*}
            A\sqrt{I_k+A^\top A}&=A\left(I_k-\sum_{j=1}^\infty\left\lvert\binom{1/2}{j}\right\rvert\left(I_k-I_k-A^\top A\right)^j\right) \\
            &=A\left(I_k-\sum_{j=1}^\infty\left\lvert\binom{1/2}{j}\right\rvert\left(-A^\top A\right)^j\right) \\
            &=A-\sum_{j=1}^\infty\left\lvert\binom{1/2}{j}\right\rvert A\left(-A^\top A\right)^j \\
            &=A-\sum_{j=1}^\infty\left\lvert\binom{1/2}{j}\right\rvert\left(-AA^\top\right)^jA \\
            &=\left(I_{n-k}-\sum_{j=1}^\infty\left\lvert\binom{1/2}{j}\right\rvert\left(-AA^\top\right)^j\right)A \\
            &=\sqrt{I_{n-k}+AA^\top}A
        \end{align*}
        The remainder of the claim is proved by replacing $I_k$ with $I_n$ and $A$ with $A^\top$.
    \end{proof}
\end{lemma}

\section{Miscellanea on learning atlas graphs from point clouds}
\subsection{Quadratic approximations of point clouds}
~\label{sec:quad_approx_pt_cloud}
For local quadratic approximation of point clouds, there exists a procedure \citep[Appendix~C.1 of][]{sritharan}, which we include here for completeness, with minor modifications in notation. 

Let $X\in\mathbb{R}^{N\times D}$ be a matrix whose rows are observations in $\mathbb{R}^D$. Further, for $\vec{c}\in\mathbb{R}^D$ and $r>0$, let $X_{\vec{c},r}\in\mathbb{R}^{N_{\vec{c},r}\times D}$ be the matrix of rows in $X$ that are within Euclidean distance $r$ of $\vec{c}$, where $N_{\vec{c},r}$ is the number of all such points. Lastly, let $\tilde{X}_{\vec{c},r}\in\mathbb{R}^{N_{\vec{c},r}\times D}$ be the matrix such that, if $\vec{x}_i$ is the $i$th row of $X_{\vec{c},r}$, then the $i$th row of $\tilde{X}_{\vec{c},r}$ is equal\footnote{$\vec{c}$ should be chosen to be the mean of the rows in $X_{\vec{c},r}$.} to $\vec{x}_i-\vec{c}$. From $\tilde{X}_{\vec{c},r}$, we can construct the local covariance matrix
\begin{equation}\label{eqn:local_cov}
    \Sigma_{\vec{c},r}=\frac{1}{N_{\vec{c},r}}\tilde{X}_{\vec{c},r}^\top\tilde{X}_{\vec{c},r}.
\end{equation}
We can get an eigendecomposition $V\Lambda V^\top=\Sigma_{\vec{c},r}$ such that the diagonal entries of $\Lambda$ are decreasing. Fixing $d<D$, we can define $L_{\vec{c},r}\in\mathbb{R}^{D\times d},M_{\vec{c},r}\in\mathbb{R}^{D\times (D-d)}$ satisfying $V=\left(L_{\vec{c},r}\mid M_{\vec{c},r}\right)$.

\textbf{The decomposition $V=\left(L_{\vec{c},r}\mid M_{\vec{c},r}\right)$ allows us to decompose $\mathbb{R}^D$ into \textit{tangential coordinates} with basis given by the columns of $L_{\vec{c},r}$ and \textit{normal coordinates} with basis given by the columns of $M_{\vec{c},r}$.} The data $\tilde{X}_{\vec{c},r}$ similarly decompose into a \textit{tangential component} $\mathbf{\tau}=\tilde{X}_{\vec{c},r}L_{\vec{c},r}$ and a \textit{normal component} $\mathbf{n}=\tilde{X}_{\vec{c},r}M_{\vec{c},r}$. If the data $X$ are sampled from a $d$-dimensional manifold, $N_{\vec{c},r}$ is sufficiently large, and $r$ is within an appropriate range\footnote{Ranges of $r$ that interest us are described at length in \citep{little}.}, then the relationship between $\mathbf{n}$ and $\mathbf{\tau}$ should be that of a quadratic polynomial in terms of tangential coordinates. This relationship takes the form $\mathbf{n}\approx\mathbf{th}$, where $\mathbf{h}$ is a matrix of quadratic coefficients\footnote{By our definition of $L_{\vec{c},r}$, we assume that linear dependence of $\mathbf{n}$ on $\mathbf{t}$ is negligible.} and $\mathbf{t}$ is a matrix of ones and quadratic monomials given by
\begin{equation}\label{eqn:quad_terms}
    \mathbf{t}=\left(\begin{array}{ccccccc}
        1 & \tau_1^{(1)}\tau_1^{(1)} & \ldots & \tau_1^{(1)}\tau_d^{(1)} & \tau_2^{(1)}\tau_1^{(1)} & \ldots & \tau_d^{(1)}\tau_d^{(1)} \\
        \vdots & \vdots & \ddots & \vdots & \vdots & \ddots & \vdots \\
        1 & \tau_1^{(N_{\vec{c},r})}\tau_1^{(N_{\vec{c},r})} & \ldots & \tau_1^{(N_{\vec{c},r})}\tau_d^{(N_{\vec{c},r})} & \tau_2^{(N_{\vec{c},r})}\tau_1^{(N_{\vec{c},r})} & \ldots & \tau_d^{(N_{\vec{c},r})}\tau_d^{(N_{\vec{c},r})}
    \end{array}\right),
\end{equation}
as in \citep{sritharan}. Regression by least squares solves for $\mathbf{h}$ by $\hat{\mathbf{h}}=\left(\mathbf{t}^\top\mathbf{t}\right)^{-1}\mathbf{t}^\top\mathbf{n}$. Instead of storing quadratic coefficients and constant terms in the matrix $\hat{\mathbf{h}}$, we can represent the relationship between tangential coordinate vector $\tv$ and normal coordinate vector $\nv$ as
\begin{equation}\label{eqn:quad_form_standard}
    \nv\approx\frac{1}{2}K_{\vec{c},r}\left(\tv\otimes\tv\right)+\vec{a}_{\vec{c},r},
\end{equation}
where $K_{\vec{c},r}$ is a matrix of quadratic coefficients, $\vec{a}_{\vec{c},r}$ is a vector of constants, and ``$\otimes$'' is the Kronecker product. Accordingly, points $\vec{x}$ in the ambient space obey the relationship
\begin{equation}\label{eqn:quad_form_general}
    \vec{x}\approx\frac{1}{2}M_{\vec{c},r}K_{\vec{c},r}\left(\tv\otimes\tv\right)+L_{\vec{c},r}\tv+\vec{x}_{\vec{c},r},
\end{equation}
where $\vec{x}_{\vec{c},r}=\vec{c}+\vec{a}_{\vec{c},r}$.

\subsection{Injectivity of na\"ive transition maps}\label{app:injectivity}
For a function $f:X\to Y$, let $\overline{f}:X\to X\times Y$ denote the function $\overline{f}:x\mapsto(x,f(x))$ (conventionally, $\overline{f}$ is called the \textit{graph} of $f$).
\begin{claim}\label{clm:lipschitz}
    Let $\mathcal{V},\mathcal{V}^\prime\subset\mathbb{R}^N$ be $d$-dimensional affine spaces with angle $\theta$ between them, and $\psi,\psi^\prime$ be the affine projections onto $\mathcal{V}$ and $\mathcal{V}^\prime$, respectively. If $F:\mathcal{V}\to\mathcal{V}^\perp$ is $K$-Lipschitz for $K<\tan\left(\frac{\pi}{2}-\theta\right)$, then $\psi^\prime\circ\overline{F}$ is injective.
    \begin{proof}
        Assume to the contrary that there exist distinct $a,b\in\mathcal{V}$ such that $(\psi^\prime\circ\overline{F})(a)=(\psi^\prime\circ\overline{F})(b)$. The map $\overline{F}$ is injective, so $\alpha:=\overline{F}(a),\beta:=\overline{F}(b)$ are distinct, and $\psi^\prime$ maps the line $l_{\alpha\beta}$ passing through $\alpha,\beta$ to a single point $p$. The line $l_{\alpha\beta}$ is the graph of some affine function $\varphi:\psi(l_{\alpha\beta})\to\mathcal{V}^\perp$ obeying the relation
        \begin{equation*}
            \frac{\left\lVert\varphi(x)-\varphi(y)\right\rVert_{\mathbb{R}^{N-d}}}{\left\lVert x-y\right\rVert_{\mathbb{R}^d}} \geq \tan\left(\frac{\pi}{2}-\theta\right)
        \end{equation*}
        for distinct $x,y\in\psi(l_{\alpha\beta})$. So
        \begin{equation*}
            \frac{\left\lVert F(a)-F(b)\right\rVert_{\mathbb{R}^{N-d}}}{\left\lVert a-b\right\rVert_{\mathbb{R}^d}} \geq\tan\left(\frac{\pi}{2}-\theta\right),
        \end{equation*}
        which contradicts our assumption that $F$ is $K$-Lipschitz.
    \end{proof}
\end{claim}

\subsection{A computation involving the SVM ``kernel trick''}\label{sec:kernel_trick}
Here, we present a computation that is used to formalize transition boundaries in approximate atlas graphs in Sec.~\ref{sec:chart_boundaries}. We want to create a discriminating boundary between two classes which is the locus of a quadratic. Consider the feature map
\begin{align*}
    \varphi:\mathbb{R}^d&\to\mathbb{R}^{d+\frac{d(d+1)}{2}} \\
    \vec{x}&\mapsto\left(\vec{x}\mid x_1^2\mid x_1x_2\mid\ldots\mid x_1x_d\mid x_2^2\mid\ldots\mid x_d^2\right)^\top
\end{align*}
Let $\vec{a}$ be an element of $\mathbb{R}^{d+\frac{d(d+1)}{2}}$, and let $b$ be some real number. Then the locus of the equation $\vec{a}^\top\vec{y}=b$ is a hyperplane $H$ in $\mathbb{R}^{d+\frac{d(d+1)}{2}}$. We can express $H\cap\mathbf{im}\varphi$ as the locus of the equation $\vec{a}^\top\varphi\left(\vec{x}\right)=b$. Defining $\vec{\lambda}\in\mathbb{R}^d,\vec{\mu}\in\mathbb{R}^{\frac{d(d+1)}{2}}$ by the equation $\vec{a}=\left(\frac{\vec{\lambda}}{\vec{\mu}}\right)$, we get the following, simplifying series of deductions.
\begin{align*}
    \vec{a}^\top\varphi\left(\vec{x}\right)&=b \\
    \vec{\lambda}^\top\vec{x}+\sum_{i=1}^d\sum_{j=i}^d\mu_{ij}x_ix_j&=b \\
    \vec{\lambda}^\top\vec{x}+\sum_{i=1}^d\sum_{j=i}^d\mu_{ij}x_ix_j-b&=0
\end{align*}
If we observe that $\sum_{i=1}^d\sum_{j=i}^d\mu_{ij}x_ix_j$ is equal to $\vec{x}^\top M\vec{x}$ for some symmetric matrix $M$ and define $b^\prime:=-b$, we get the equation
\begin{equation*}
    \vec{x}^\top M\vec{x}+\vec{\lambda}^\top\vec{x}+b^\prime=0,
\end{equation*}
and we are done. \\

Recall that solutions to hard-margin SVMs are invariant under linear isomorphism of the feature space. For this reason, we have as much expressive power with the feature map $\varphi$ used above as with the feature map
\begin{equation*}
    \varphi^\prime:\vec{x}\mapsto\left(\vec{x}\mid x_1^2\mid 2x_1x_2\mid\ldots\mid 2x_1x_d\mid x_2^2\mid\ldots\mid x_d^2\right)^\top
\end{equation*}
(observe that this is simply the result of multiplying coordinate values corresponding to mixed quadratic terms by 2). The kernel $k$ associated with $\varphi$ has the same expressive power as the kernel $k^\prime$ associated with $\varphi^\prime$; however, $k^\prime$ has a concise form, given by the following series of deductions.
\begin{align*}
    k^\prime\left(\vec{x},\vec{y}\right)&=\varphi^\prime\left(\vec{x}\right)\cdot\varphi^\prime\left(\vec{y}\right) \\
    &=\left(\vec{x}\mid x_1^2\mid 2x_1x_2\mid\ldots\mid x_2^2\mid\ldots\mid x_d^2\right)\cdot\left(\vec{y}\mid y_1^2\mid 2y_1y_2\mid\ldots\mid y_2^2\mid\ldots\mid y_d^2\right) \\
    &=\vec{x}^\top\vec{y}+\mathbf{vec}\left(\vec{x}\vec{x}^\top\right)\cdot\mathbf{vec}\left(\vec{y}\vec{y}^\top\right) \\
    &=\vec{x}^\top\vec{y}+\mathbf{Tr}\left(\vec{x}\vec{x}^\top\vec{y}\vec{y}^\top\right) \\
    &=\vec{x}^\top\vec{y}+\left(\vec{x}^\top\vec{y}\right)^2
\end{align*}

\begin{lemma}\label{lem:quad_linear_vanish}
    For a (multivariate, not necessarily homogeneous) polynomial of degree 2, there exists (under a general position assumption) a translation of coordinates such that the linear term vanishes.
    \begin{proof}
        For arbitrary $\vec{\lambda}\in\mathbb{R}^d$, we can perform the translation of coordinates $\vec{x}\mapsto\vec{y}-\vec{\lambda}$. This allows us to rewrite the polynomial $\vec{x}^\top A\vec{x}+\vec{b}^\top\vec{x}+c$ as $\vec{y}^\top A\vec{y}-2\vec{\lambda}^\top A\vec{y}+\vec{\lambda}^\top A\vec{\lambda}+\vec{b}^\top\vec{y}-\vec{b}^\top\vec{\lambda}+c$ (recall that $A$ is symmetric without loss of generality). The linear term vanishes if and only if $\vec{b}=2A\vec{\lambda}$. Such a $\vec{\lambda}$ is guaranteed to exist when $A$ is invertible. Given $A$ is learned by soft-margin support vector machine with the kernel given in Sec.~\ref{sec:kernel_trick}, and given a general position assumption on the data, this can be taken for granted.
    \end{proof}
\end{lemma}

\subsection{Approximate path-lengths of quasi-Euclidean updates}\label{app:naive_app_dist}
In this section, we compute closed forms for the na\"ive approximate distance defined in Equation~\ref{eqn:naive_approx_dist}.
\begin{claim}\label{clm:naive_approx_dist}
    Let $\tilde{d}_i$ denote the na\"ive approximate distance on chart $i$, let $K_{ij}$ denote the matrix of quadratic coefficients for chart $i$, and let $\tv_0,\tv_1$ be representative tangent vectors on chart $i$. We define
    \begin{equation*}
        \begin{array}{ccc}
            \lambda:=\left(\tv_1-\tv_0\right)^\top\left(\tv_1-\tv_0\right), & \kappa_{ij0}:=\left(\tv_1-\tv_0\right)^\top K_{ij}\tv_0, & \kappa_{ij1}:=\left(\tv_1-\tv_0\right)^\top K_{ij}\tv_1, \\
            \ & \ & \ \\
            \mu_{i00}:=\sum_{j=1}^{n-d}\kappa_{ij0}^2 & \mu_{i01}:=\sum_{j=1}^{n-d}\kappa_{ij0}\kappa_{ij1}, & \mu_{i11}:=\sum_{j=1}^{n-d}\kappa_{ij1}^2, \\
            \ & \ & \ \\
            a_i:=\frac{\mu_{i01}-\mu_{i00}}{\mu_{i00}+\mu_{i11}-2\mu_{i01}} & b_i:=\frac{\lambda+\mu_{i00}}{\mu_{i00}+\mu_{i11}-2\mu_{i01}} & c_i:=\sqrt{\mu_{i00}+\mu_{i11}-2\mu_{i01}}.
        \end{array}
    \end{equation*}
    \begin{enumerate}[label=\alph*)]
        \item If $\mu_{i00}+\mu_{i11}-2\mu_{i01}>0$, then
        \begin{align*}
            \tilde{d}_i\left(\tv_0,\tv_1\right)&=\frac{1}{2}c_i\left(1+a_i\right)\sqrt{1+a_i+a_i^2+b_i}-\frac{1}{2}c_ia_i\sqrt{a_i^2-a_i+b_i} \\
            &\hspace{1cm}+\frac{1}{2}c_i\left(b_i-a_i\right)\ln\left(\frac{1+a_i}{\sqrt{b_i-a_i}}+\sqrt{\frac{\left(1+a_i\right)^2}{b_i-a_i}+1}\right) \\
            &\hspace{1cm}-\frac{1}{2}c_i\left(b_i-a_i\right)\ln\left(\frac{a_i}{\sqrt{b_i-a_i}}+\sqrt{\frac{a_i^2}{b_i-a_i}+1}\right).
        \end{align*}
        \item If $\mu_{i00}+\mu_{i11}-2\mu_{i01}=0$, then
        \begin{equation*}
            \tilde{d}_i\left(\tv_0,\tv_1\right)=\frac{1}{3\left(\mu_{i01}-\mu_{i00}\right)}\left(\left(\lambda+2\mu_{i01}-\mu_{i00}\right)^{3/2}-\left(\lambda+\mu_{i00}\right)^{3/2}\right)
        \end{equation*}
    \end{enumerate}

    \begin{proof}
        We begin with the following series of deductions.
        \begin{align*}
            \tilde{d}_i\left(\tv_0,\tv_1\right)&=\int_0^1\sqrt{\lambda+\sum_{j=1}^{n-d}\left(\left(\tv_1-\tv_0\right)^\top K_{ij}\left[(1-t)\tv_0+t\tv_1\right]\right)^2}dt \\
            &=\int_0^1\sqrt{\lambda+\sum_{j=1}^{n-d}\left((1-t)\kappa_{ij0}+t\kappa_{ij1}\right)^2}dt \\
            &=\int_0^1\sqrt{\lambda+\sum_{j=1}^{n-d}\left((1-t)^2\kappa_{ij0}^2+2t(1-t)\kappa_{ij0}\kappa_{ij1}+t^2\kappa_{ij1}^2\right)}dt \\
            &=\int_0^1\sqrt{\lambda+(1-t)^2\mu_{i00}+2t(1-t)\mu_{i01}+t^2\mu_{i11}}dt \\
            &=\int_0^1\sqrt{\lambda+\mu_{i00}+2\left(\mu_{i01}-\mu_{i00}\right)t+\left(\mu_{i00}+\mu_{i11}-2\mu_{i01}\right)t^2}dt
        \end{align*}
        \begin{enumerate}[label=\alph*)]
            \item Say $\mu_{i00}+\mu_{i11}-2\mu_{i01}>0$.
            \begin{align*}
                \tilde{d}_i\left(\tv_0,\tv_1\right)&=\sqrt{\mu_{i00}+\mu_{i11}-2\mu_{i01}}\int_0^1\sqrt{t^2+\frac{2\left(\mu_{i01}-\mu_{i00}\right)}{\mu_{i00}+\mu_{i11}-2\mu_{i01}}t+\frac{\lambda+\mu_{i00}}{\mu_{i00}+\mu_{i11}-2\mu_{i01}}}dt \\
                &=c_i\int_0^1\sqrt{t^2+2a_it+b_i}dt \\
                &=c_i\int_{a_i}^{1+a_i}\sqrt{t^2+b_i-a_i}dt \\
                &=c_i\left.\frac{1}{2}\left(t\sqrt{t^2+b_i-a_i}+\left(b_i-a_i\right)\sinh^{-1}\left(\frac{t}{\sqrt{b_i-a_i}}\right)\right)\right\rvert_{a_i}^{1+a_i} \\
                &=\frac{1}{2}c_i\left(1+a_i\right)\sqrt{1+a_i+a_i^2+b_i}-\frac{1}{2}c_ia_i\sqrt{a_i^2-a_i+b_i} \\
                &\hspace{1cm}+\frac{1}{2}c_i\left(b_i-a_i\right)\ln\left(\frac{1+a_i}{\sqrt{b_i-a_i}}+\sqrt{\frac{\left(1+a_i\right)^2}{b_i-a_i}+1}\right) \\
                &\hspace{1cm}-\frac{1}{2}c_i\left(b_i-a_i\right)\ln\left(\frac{a_i}{\sqrt{b_i-a_i}}+\sqrt{\frac{a_i^2}{b_i-a_i}+1}\right)
            \end{align*}
            \item Say instead $\mu_{i00}+\mu_{i11}-2\mu_{i01}=0$. We assume that $\mu_{i01}-\mu_{i00}\neq0$; otherwise, $\lambda+\mu_{i00}=0$, and $\tilde{d}_i\left(\tv_0,\tv_1\right)=0$.
            \begin{align*}
                \tilde{d}_i\left(\tv_0,\tv_1\right)&=\int_0^1\sqrt{\lambda+\mu_{i00}+2\left(\mu_{i01}-\mu_{i00}\right)t}dt \\
                % &=\left.\frac{1}{3\left(\mu_{i01}-\mu_{i00}\right)}\left(\lambda+\mu_{i00}+2\left(\mu_{i01}-\mu_{i00}\right)t\right)^{3/2}\right\rvert_0^1 \\
                &=\frac{1}{3\left(\mu_{i01}-\mu_{i00}\right)}\left(\left(\lambda+2\mu_{i01}-\mu_{i00}\right)^{3/2}-\left(\lambda+\mu_{i00}\right)^{3/2}\right)
            \end{align*}
        \end{enumerate}
    \end{proof}
\end{claim}

% \bibliographystyle{siamplain}
% \bibliography{references}

% \end{document}